\documentclass[lettersize,journal]{IEEEtran}
\usepackage{amsmath,amsfonts}
\usepackage{array}
\usepackage{textcomp}
\usepackage{stfloats}
\usepackage{url}
\usepackage{verbatim}
\usepackage{graphicx}
\usepackage{cite}
\usepackage{booktabs}       % professional-quality tables
\usepackage{nicefrac}       % compact symbols for 1/2, etc.
\usepackage{microtype}      % microtypography
\usepackage{xcolor}         % colors

\usepackage{amssymb,amsthm}
\usepackage{graphicx}
\usepackage{textcomp}
\usepackage{xcolor}

\usepackage{algorithm}
\usepackage{algpseudocode}
\usepackage{graphicx}
\usepackage{subfigure}
 \usepackage{subcaption}
\usepackage{color,soul,xspace}
\usepackage{multirow}
\usepackage{url}

\def\tensor{~\otimes~}
\def\DictsOrd{\mathbb{D}_{ord}}
\def\DictsOrdL{\mathbb{D}_{ord,L}}
\def\DictsOrdR{\mathbb{D}_{ord,R}}
\def\Vol{\mathrm{Vol}}
\def\Banach{\mathbb{B}}
\def\E{\mathbb{E}}
\def\F{\mathcal{F}}
\def\Hyp{\mathcal{H}}
\def\R{\mathbb{R}}
\def\sparse{\mathrm{sparse}}
\def\X{\mathcal{X}}
\def\Dist{\mathcal{D}}
\def\CovNum{\mathcal{N}}
\def\PackingNum{\mathcal{M}}

\def\Constraint{\mathrm{CoeffMat}}
\def\CoeffMat{\Constraint}
\def\ddict{d_{dict}}
\def\Dicts{\mathbb{D}}
\def\Normal{\mathcal{N}}
\def\Alg{\mathcal{A}}
\def\ExcessRisk{\Delta R}
\def\range{\mathrm{range}}
\def\FullRange{\mathrm{FullRange}}
\def\Uniform{\mathrm{Uniform}}
\def\Packing{\mathcal{P}}
\def\gwidth{\mathfrak{w}}
\def\noiseterm{\vec{\epsilon}_X}
\def\Overfitting{\Delta \mathcal{E}}
\def\Tr{\mathrm{Tr}}
\def\rank{\mathrm{rank}}

\def\OneMatrix{\mathbf{1}}

\def\DictsOrdUnconstrained{\mathbb{D}_{ord,unconst}}
\def\DictsOrdUnconstrainedL{\mathbb{D}_{ord,unconst,L}}
\def\DictsOrdUnconstrainedR{\mathbb{D}_{ord,unconst,R}}

\newtheorem{theorem}{Theorem}
\newtheorem{definition}{Definition}
\newtheorem{lemma}{Lemma}
\newtheorem{proposition}{Proposition}

\newtheorem{remark}{Remark}

\newcommand{\innerproduct}[2]{\langle {#1}, {#2} \rangle}

\DeclareMathOperator*{\argmin}{arg\,min}

\newcommand{\vsa}{\vspace*{-0.28cm}}
\newcommand{\vsb}{\vspace*{-0.19cm}}
\newcommand{\vsc}{\vspace*{-0.16cm}}

\newcommand{\mourmeth}{\text{AODL}}
\newcommand{\ourmeth}{$\mourmeth$\xspace}

\begin{document}

\title{Multi-Dictionary Learning for Low Rank Sparse Coding}

\author{Boya Ma\thanks{Contact emails: \{bma, amagner, mmcneil2, pbogdanov\}@albany.edu}
\thanks{University at Albany -- SUNY, 
Computer Science, 
Albany, NY, USA}
\qquad Abram Magner\thanks{AM funded by NSF CAREER CCF 2338855 and CCF 2212327.} \qquad Maxwell McNeil \qquad Petko Bogdanov 
% \thanks{Manuscript received April 19, 2021; revised August 16, 2021. \am{Change dates.}}
}

%\author{IEEE Publication Technology,~\IEEEmembership{Staff,~IEEE,}
%        % <-this % stops a space
%\thanks{This paper was produced by the IEEE Publication Technology Group. They are in Piscataway, NJ.\am{FINISH ME}}% <-this % stops a space
%\thanks{Manuscript received April 19, 2021; revised August 16, 2021. \am{FINISH ME}}}

% \markboth{IEEE Transactions on Signal Processing ,~Vol.~XX, No.~XX, August~XXXX}%
% {Ma \MakeLowercase{\textit{et al.}}: Multi-Dictionary Learning for Low Rank Sparse Coding}


\maketitle

\begin{abstract}

%% Original Astract at time of registration, let us keep this
Sparse dictionary coding represents signals as linear combinations of a few dictionary atoms. It has been applied to images, time series, graph signals and multi-way spatio-temporal data by jointly employing temporal and spatial dictionaries. Data-agnostic analytical dictionaries, such as the discrete Fourier transform, wavelets and graph Fourier, have seen wide adoption due to efficient implementations and good practical performance. On the other hand, dictionaries learned from data offer sparser and more accurate solutions but require learning of both the dictionaries and the coding coefficients. This becomes especially challenging for multi-dictionary scenarios since encoding coefficients correspond to all atom combinations from the dictionaries. To address this challenge, we propose a low-rank coding model for 2-dictionary scenarios and study its data complexity. Namely, we establish upper and lower bounds on the number of samples needed to learn dictionaries that generalize to unseen samples from the same distribution. We propose an alternating convex optimization solution, called AODL, which employs alternating optimization between the sparse coding matrices and the learned dictionaries. We demonstrate its quality for data reconstruction and missing value imputation in both synthetic and real-world datasets. For a fixed reconstruction quality, AODL learns up to $90\%$ sparser solutions compared to non-low-rank and analytical (fixed) dictionary baselines. In addition, the learned dictionaries reveal interpretable insights into patterns from training samples.
\end{abstract}

\section{Introduction}
%1st paragraph: General intro to sparse coding
Sparse dictionary-based coding has been employed for signal and image processing~\cite{tan2012convex,adler2015sparse}, machine learning~\cite{multilabel2021aaai,evtimova2021sparse,sulam2020adversarial}, compressed sensing~\cite{ji2008bayesian} and data analytics~\cite{TGSD,MDTD,shuman2013emerging, Sandryhaila2014}. In the sparse-coding framework, observed data is represented as a linear combination of vectors called dictionary atoms. Dictionaries can be either derived analytically or learned from data. Commonly adopted analytical dictionaries include the discrete Fourier transform (DFT), wavelets, and the Ramanujan periodic basis~\cite{tennetiTSP2015}. While they provide structured priors such as signal smoothness over a graph structure via the GFT ~\cite{dong2019learning} or periodicity via the Ramanujan dictionary~\cite{tennetiTSP2015}, they may fall short in enabling sparse and accurate representations for data with patterns that do not align well with the predefined atoms. 
%2nd Paragraph: Dictionary learning - high level techniques and challenges
An alternative approach is to learn the dictionaries directly from data which has been shown to enable higher compression rates and better representation quality~\cite{rubinstein2010dictionaries}. 
The input in the dictionary learning problem is a set of (training) signals, and the goal is to learn both a dictionary and corresponding coding coefficients for the input ~\cite{aharon2006k, hawe2013separable, shahriari2021new,pati1993orthogonal}. %Most dictionary learning methods iterate between two key steps: (i) \emph{sparse coding} which calculates the coding coefficients for fixed (initially random) atoms and ii) \emph{dictionary update} to better fit the data. A seminal approach K-SVD~\cite{aharon2006k} employs singular value decomposition of the sparse data reconstruction in the dictionary update step and an approximate solver, typically orthogonal matching pursuit (OMP)~\cite{pati1993orthogonal}, in the sparse coding step. 

%3rd Paragraph: motivate the problem for 2D data
While many existing techniques focus on one-way (vector) input signals, multi-mode samples require learning dictionaries for each mode. For example, spatiotemporal signals %(such as those acquired by sensor networks) 
could be sparsely coded by employing jointly a spatial dictionary with atoms corresponding to spatial localities and temporal dictionary with atoms corresponding to trends in time~\cite{TGSD}. Multi-dictionary sparse coding with fixed (analytical) dictionaries has been demonstrated beneficial for a range of downstream tasks like compression, missing value imputation, and community detection~\cite{TGSD,ur2023time,MDTD}. Learning dictionaries from data in the multi-way setting 
% 4th Paragraph: Existing work and high-level limitation: too many coefficients 
requires i) learning multiple dictionaries and ii) estimating coding coefficients which correspond to combinations of atoms. An early two-way (2D) dictionary learning method SeDiL~\cite{hawe2013separable} promotes full rank dictionaries and atom coherence via a %. It converts the problem into a Matrix Manifolds ~\cite{absil2008optimization} problem and 
geometric conjugate gradient optimizer for the dictionary learning step. Follow-up works~\cite{zhang2017analytic,zhang2016improved} improve on the original method by adopting 2D-OMP~\cite{fang20122d} and FISTA~\cite{beck2009fast} for the sparse coding step. The latest methods in this category MOD and CMOD~\cite{shahriari2021new} employ gradient projection for dictionary learning. A common limitation of all existing methods is that they do not impose any structure (beyond the usual sparsity) on the coding matrices. In the two-way case the number of coefficients grows quadratically with the sizes (number of atoms) of the left and right dictionaries. This rate of growth of the coding matrices makes the iterative learning of dictionaries and coding challenging. 

\begin{figure}[t]
{
   \centering
   \subfigure[Data model employed by \ourmeth.]{
    \includegraphics[width=0.8\linewidth,trim=0cm 4cm 0cm 0pt, clip]{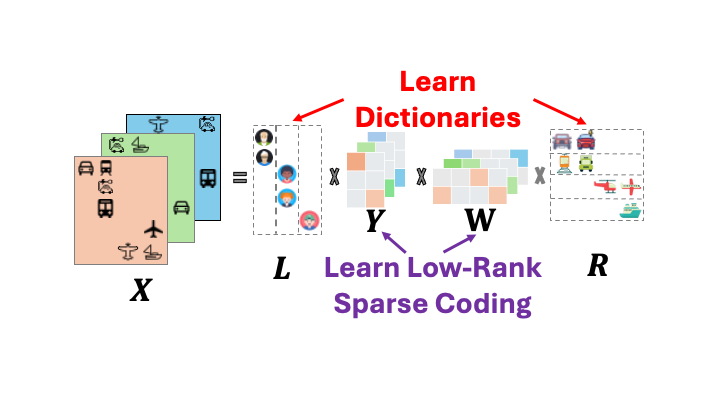}
    \label{fig:aodl_structure}
    }
    \subfigure[Low v.s. unrestricted rank.]{
    \includegraphics[width=0.6\linewidth]{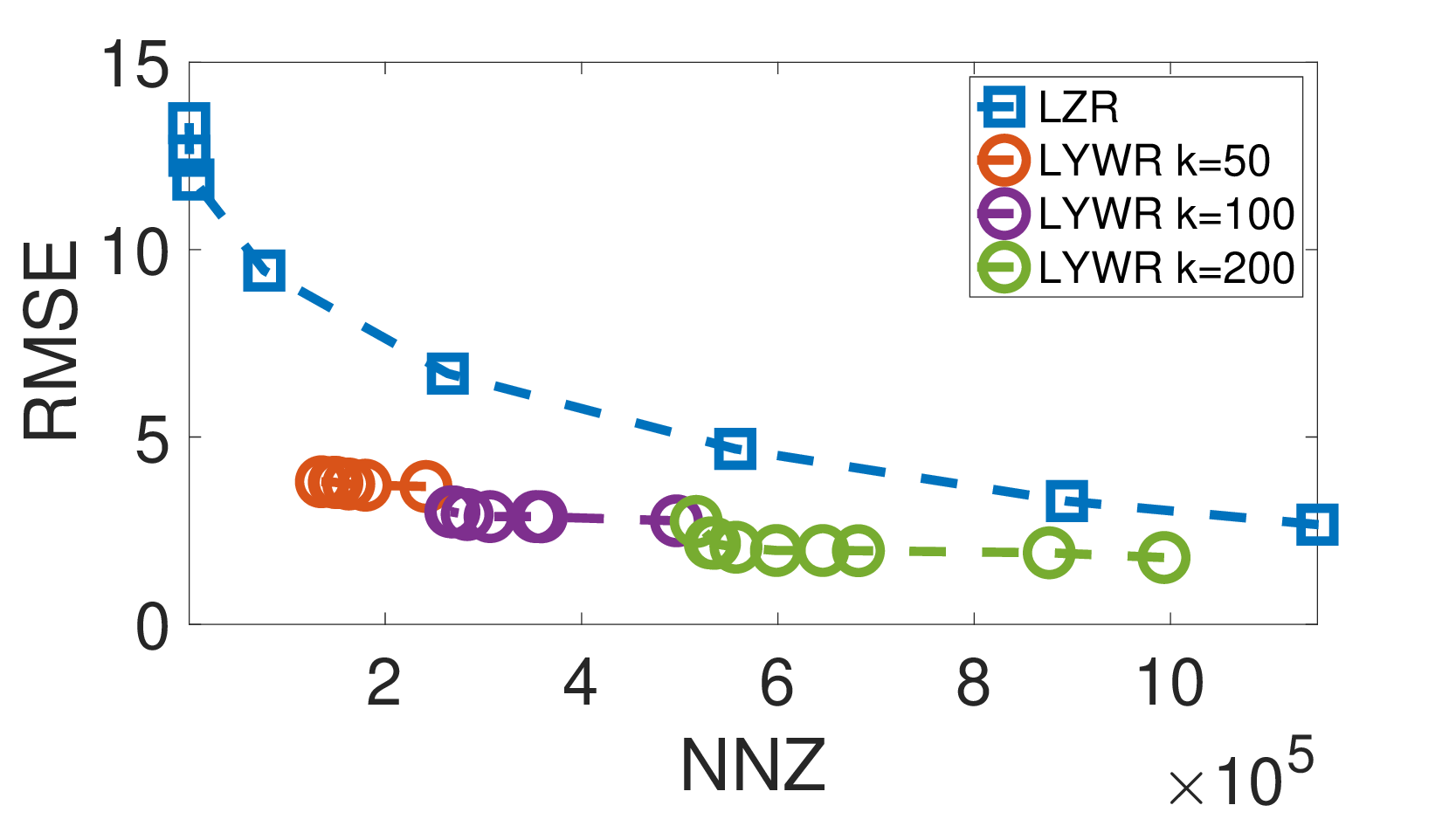}
    \label{fig:full_vs_low_rank} 
    }
    \caption{\subref{fig:aodl_structure} \footnotesize \ourmeth model: $\X$ is a set of data sample matrices,  $L$ and $R$ are shared dictionaries and $Y, W$ are sets of ``slim'' sample-specific coding matrices. \subref{fig:full_vs_low_rank} Comparison of low-rank (LYWR) and unrestricted rank (LZR) coding representations for Road traffic data. Both representations are estimated using an ADMM sparse solver with analytical (GFT and Ramanujan) dictionaries for a single data matrix $X$.}
    % \vsa\vsa\vsa
    \label{fig:aodl_model}
    }
\end{figure}

% 5th paragraph: what do we do
To address the above challenge we propose to learn two-way dictionaries for low-rank coding matrices. The data model is schematically depicted in Fig.~\ref{fig:aodl_structure} where the input signals $X$ correspond to different snapshots of user-transportation preference matrices from some organization. We model the inputs as a product of shared left (user) $L$ and right (transportation type) $R$ dictionaries and a set of low-rank coding matrices represented as the product $YW$. The inner dimension (columns of $Y$ and rows of $W$) determines the rank of the coding matrix. Beyond limiting the number of coefficients generating the data, this representation allows us to capture patterns of shared behavior in the learned dictionary atoms, e.g, persistent groups of users who prefer specific types of transportation for the example from Fig.~\ref{fig:aodl_structure}. In addition, the low-rank coding model enables more succinct representations of spatio-temporal data as demonstrated in Fig.~\ref{fig:full_vs_low_rank} which compares the error and model sizes of the low rank model ($LYWR^T$) and an unconstrained rank model ($LZR^T$) with fixed dictionaries (only sparse coding). 

Beyond introducing the low-rank two-way dictionary learning problem, we derive a data-distribution-independent sample complexity \emph{upper bound} for it and back it up empirically. We generalize the bound for loss functions that enforce spread of $L_1$ mass across the columns/rows of coding matrices to further probe the role of the rank constraint in our model. We also derive a sample complexity \emph{lower bound} establishing that the rank constraint affects the sample complexity by at most a polylogarithmic factor in the number of dictionary atoms. We show that the general sparsity constrained problem can be solved as an $L_1$ regularization objective. Informed by the above theoretical results, we propose an efficient optimization procedure \ourmeth which outperforms baselines in terms of data reconstruction quality and missing value imputation on multiple datasets.

Our contributions in this work are as follows:\\
\noindent{\bf $\bullet$ Novelty:} To the best of our knowledge, we are the first to pose the problem of multi-dictionary learning for the setting where coding matrices are constrained to be \emph{low-rank} and sparse.\\
\noindent{\bf $\bullet$ Sample complexity analysis:} We derive both upper and lower bounds on the data complexity of our low-rank dictionary learning problem, elucidating the role of the rank constraint in the sample complexity of the problem, and back our bounds experimentally.\\
%\todo{Abram, should we say something more specific for both bounds?}
%\noindent{\bf $\bullet$ Theoretical analysis: } We derive a data-distribution-independent sample complexity bound for 2D dictionary learning and show empirical support for this bound.  We then generalize this bound for loss functions that enforce spread of $L_1$ mass across the columns/rows of coding matrices to further probe the role of the rank constraint in our model.  We additionally derive a sample complexity lower bound that establishes that the rank constraint affects the sample complexity by at most a polylogarithmic factor in the number of dictionary atoms. %Our experiments indicate that sample complexity bounds from current techniques may not be tight for natural data-generating distributions, motivating future work on distribution-dependent bounds. \\ 
\noindent{\bf $\bullet$ Practical solution:} We prove that the general sparsity constrained problem can be solved as an $L_1$ regularization objective for which we propose an efficient optimization procedure \ourmeth.\\ %using ur algorithmic solution and show that our sparse coding via regularization approach provably solves sparse coding with $L_1$ constraints. We additionally show that our alternating sparse coding/dictionary optimization procedure converges.
\noindent{\bf $\bullet$ Accuracy:} Our proposed method \ourmeth consistently produces encodings with lower representation error compared to the closest baselines for a fixed number of coding coefficients (model size). It also imputes missing values more accurately than alternative techniques.\\
\noindent{\bf $\bullet$ Compactness:} To achieve the same level of reconstruction quality, our approach saves up to $90 \%$ of the coefficients compared to the closest baselines in large real-world datasets.

% \vsa
\section{Related Work}

{\noindent \bf Sparse coding} is widely employed in signal processing~\cite{zhang2015survey,rubinstein2010dictionaries, GIST}, image analysis~\cite{elad2006image} and computer vision~\cite{wright2008robust}. Existing methods can be grouped into three main categories: convex optimization solutions, non-convex techniques, and greedy algorithms~\cite{marques2018review}. Relaxation techniques impose sparsity on the coding coefficients via L1 regularizers~\cite{TGSD,MDTD}, while 
%Basis Pursuit (BP) is a signal processing technique that decomposes the signal into a superposition of basic elements \cite{marques2018review}, such as a BP-simplex \cite{chen2001atomic}. Greedy BP, which is related to the Matching Pursuit algorithm, builds the representation by iteratively selecting atoms based on their geometry~\cite{huggins2007greedy}. Another set of algorithms are Least Absolute Shrinkage and Selection Operator (LASSO)~\cite{tibshirani1996regression,maleki2013asymptotic}.
greedy algorithms select one atom at a time~\cite{wang2012generalized,de2017sparsity,lee2016sparse}. Most existing methods focus on 1D signals while we focus on 2D signals, in the sense that most existing methods do not enforce a Kronecker structure for learned dictionaries.  %For example, Matching Pursuit (MP)~\cite{mallat1993matching}, Orthogonal Matching Pursuit (OMP) \cite{pati1993orthogonal}, and Matching Pursuit based on Least Squares (MPLS) \cite{de2017sparsity}, select one atom at a time. Stage-wise Orthogonal Matching Pursuit (StOMP)~\cite{donoho2012sparse}, Generalized Orthogonal Matching Pursuit (GOMP) \cite{wang2012generalized}, Multi-path Matching Pursuit (MMP) \cite{kwon2014multipath},  Tree Search Matching Pursuit (TSMP) \cite{lee2016sparse} provide different techniques to chooses multiple atoms per iteration. All these methods are defined for a single dictionary while we focus on 2D data and joint coding via two dictionaries.

\noindent{\bf Two-way and multi-way sparse coding} %PEtko: I have edited it\todo{need to re-write} 
methods generalize the one dimensional setting by employing separate dictionaries for each dimension of the data~\cite{fang20122d,old_2D,zhang2017joint,TGSD,MDTD}. Some methods in this group place no assumptions about the rank of the encoding matrix~\cite{old_2D,zhang2017joint,fang20122d,2D_low_rank_completion}, while others impose a low rank on the learned encoding~\cite{TGSD, MDTD}. Our coding coefficient model is inspired by the low-rank models above; however, while the methods above adopt (fixed) analytical dictionaries we learn the dictionaries from data which as we show through experimental comparisons enables more accurate and compact data encoding. 
%A growing subset of 2D Sparse dictionary coding methods enforce that the learned coefficients form a low rank model \cite{TGSD,MDTD_full,2D_low_rank_completion}. 
%TGSD~\cite{TGSD} as it also enforces that learned coding coefficients are low rank for 2D data. %\cite{MDTD_full} generalizes this to higher order data, and \cite{2D_low_rank_completion} performs this in a soft manner. 
%We elaborate further on these importance of low rank in the following section and demonstrate a superior performance of \ourmeth over other baselines in the experimental section. 

\noindent{\bf Dictionary learning} algorithms aim to find one (for vector data) or multiple (for multi-way data) shared dictionaries directly from training samples. The majority of existing work tackles vector signals~\cite{rubinstein2010dictionaries,engan1999method} and iterates between the sparse coding and dictionary learning stages similar to the seminal K-SVD method~\cite{aharon2006k}. %Common dictionary learning approaches solves this problem in a two stage manner: alternatively update the dictionary and the coefficients. For example,  first update the coefficients with any sparse solver, and then update the dictionary using SVD to update both the dictionary and the coefficients. Some other algorithms use different strategy to update the dictionary. For instance, MOD \cite{engan1999method} (Method of Optimal Directions) leverage gradient projection approach to find the dictionary. 
%\noindent{\bf 2D and multi-way dictionary learning}
Dictionary learning methods for two-way (matrix) data follow a similar alternating process to learn two dictionaries~\cite{hawe2013separable,zhang2016improved,shahriari2021new,jiang2021dictionary}. They solve the same problem as ours except that they do not impose a low-rank structure on the encoding matrices of samples. Some methods impose orthogonality~\cite{quan2015dynamic, gong2020low} or low-rank~\cite{bahri2018robust} constraints on the learned dictionaries motivated by the specific domain applications (e.g., video processing). %Alternative solvers based on Tucker decomposition of multiple samples arranged as a tensor were also proposed~\cite{bahri2018robust}, however, this model imposes rank constraints on the dictionaries. 
Another separable dictionaries model \cite{ghassemi2019learning} learns a mixture of Kronecker-structured dictionaries for tensor data but can be viewed as a general solution for \cite{shahriari2021new}.
We compare experimentally to the state-of-the-art methods from this group that learn general unconstrained dictionaries and demonstrate that our low-rank model enables more accurate and compact sparse coding matrices and scales better than dictionary learning alternatives.

  %are taking videos or sequence of frames as input and the goal is to learn shared dictionaries with an extra orthogonality constraint. \cite{bahri2018robust} is based on Tucker tensor decomposition and samples are represented as two-part sums of CMOD-like coding components and sparse outlier matrices with rank constraints on the dictionaries. These differences from these models do not allow for direct comparison to our method. } 

%SeDiL \cite{hawe2013separable} is among the first papers to solve such problem. It converts the problem to a Matrix Manifolds problem and employs geometric conjugate gradient method to solve for the two dictionaries. \cite{zhang2017analytic} \cite{zhang2016improved} are later works that split SeDiL into a traditional two stage problem. They use 2D-OMP and FISTA to solve for the sparse coding problem, and keep using geometric conjugate gradient as the dictionary solver.The latest work 2D-MOD and 2D-CMOD \cite{shahriari2021new} are also two stage approaches, they use MOD \cite{engan1999method} to update the dictionaries in a closed form, which is also the dictionary solver in our approach.

%\input{tex/3_contributions}
% \vsa
\section{Preliminaries} \label{sec:prelim}
%PEtko: I have edited it\todo{need to re-write}
%Before we formally define our problem of multi-dictionary learning, we introduce necessary preliminaries and notation. 

\noindent{\bf Sparse coding.} The goal of 1D sparse coding is to represent a vector signal $x \in \R^{N}$ as a sparse linear combination $y \in \R^{P}$ of the atom columns of a dictionary $L \in \R^{N \times P}$ by solving the following optimization problem:   
$\min_y f(y)~~\text{s.t.}~~ x=L y,$
where $f(y)$ is a sparsity-promoting function (e.g., the $L_1$ norm). %A popular greedy strategy to solve the problem, particularly when the dictionary forms an over-complete basis, is the orthogonal matching pursuit (OMP)~\cite{pati1993orthogonal}. The OMP algorithm maintains a residual of the signal $r$ that is not yet represented, and proceeds in greedy steps to identify the dictionary atom best aligned to the residual: $\l_t = {\mathrm{argmax}}_{l_i} (r^T l_i)$, where $i \in [1, P]$, and $l_i$ is the $i$-th atom in $L$. The selected atom $l_t$ at step $t$ is appended to the result set, the signal is re-encoded and the residual re-computed. The process continues until a desired number of atoms are employed, while satisfying the sparsity function $f(y)$.
In %this paper we consider 
the 2D setting the input is a real-valued matrix $X \in \R^{N \times M}$ which can be represented as a sparse encoding matrix $Z\in \R^{P \times Q}$ via a left (column atoms) dictionary $L \in \R^{N \times P}$ and a right (row atoms) dictionary $R^T \in \R^{Q \times M}$ by optimizing the following: 
\begin{equation}
    \min_Z  f(Z)~~\text{s.t.}~~ X=L Z R^T, \label{eq:2dform}
\end{equation}
where $f(Z)$ is a sparsity promoting function. %It is important to note that both analytical and data-driven dictionaries can be employed for the problem~\cite{rubinstein2010dictionaries}. %Intuitively, this decomposition facilitates a representation which aligns to dictionaries across both modes (dimensions) instead of just one. 
%An early solution for Eq.~\ref{eq:2dform} was motivated by decomposing a 2D image via copies of the same dictionary, i.e. $L = R$~\cite{fang20122d}. %Explicitly:
%$$  X = \Psi Z \Psi^T ~~\text{s.t.}~~ ||Z||_0=b  $$
%This solution generalizes the greedy 1D orthogonal matching pursuit (OMP)~\cite{pati1993orthogonal} to obtain a 2D-OMP algorithm by forming 2D atoms $B_{i,j}=L_i^T R_j$ as outer products of individual left $l_i$ and right $r_j$ atoms, and by selecting 2D atoms based on their alignment 
%$\frac{\langle R, B_{i,j} \rangle}{|| B_{i,j} ||_F}$ 
%with the residual $R$ at every iteration. % , where $\langle R, B_{i,j} \rangle \triangleq \psi_i^T R \phi_j$ is the alignment, and $|| B_{i,j}||_F = ||\psi_i||_F ||\phi_j||_F$ is the Frobenius norm of the 2D atom. 
%Importantly, while sparse, this solution might in general result in high-rank encoding matrix $Z$ and it does not scale to large spatio-temporal datasets and large dictionaries as demonstrated in~\cite{LRMDS}.
A recent alternative 2D model, called TGSD~\cite{TGSD}, introduced a low-rank structure for the encoding matrix $Z$, i.e. $X \approx L Y W R^T$, where $Y \in \R^{P \times k}$ and $W \in \R^{k \times Q}$ are sparse dictionary-specific encoding matrices and $k$ controls the rank of the encoding. All above methods do not learn dictionaries, but instead estimate the sparse coding coefficients.

\noindent{In \bf dictionary learning} the goal is to jointly learn the dictionaries from multiple data samples and estimate the per-sample sparse coding coefficients. Existing 2D approaches generalize Eq.~(\ref{eq:2dform}):
\begin{equation}
    \begin{aligned}
        \underset{L,R,Z} {\mathrm{argmin}} || X - L Z R^T||_F^2 + \lambda || Z ||_1,
         \label{eq:1D-admm}
    \end{aligned}
\end{equation}
to learn $L,R$ in addition to $Z$. Solvers alternate between i) sparse coding with fixed dictionaries by employing 2D-OMP~\cite{fang20122d}, FISTA~\cite{beck2009fast} and others; and ii) dictionary updates for fixed coding matrices by employing conjugate gradient updates~\cite{hawe2013separable} or direct solutions~\cite{shahriari2021new}.

\section{Problem formulation and sample complexity}
\label{sec:problem-formulation}
\subsection{Problem formulation} The input to our problem is a set of $S$ samples of $2$-way (matrix) data $\mathcal{X} \in \R^{N \times M \times S}$. Each data sample $X_s$ is a matrix in $\R^{N \times M}, s \in [1, \cdots, S]$. The left $L \in \R^{N \times P}$ and right $R \in \R^{M \times Q}$ dictionaries have $P$ and $Q$ atoms respectively and the two encoding matrices for each sample are denoted as $Y_s \in \R^{P \times k}$ and $W_s \in \R^{k \times Q}$, where $k$ is an encoding rank parameter. Our objective is to learn a set of two-way dictionary atoms $(L, R)$ for which there exists low-rank, sparse encodings of the data samples $\X$. %Specifically, given the input dataset $\X$, our task is to learn a pair of shared dictionaries $(L, R)$. 
We measure the quality of the learned dictionaries on a data sample $X_s \in \R^{N\times M}$ via the following
\emph{loss function} $\ell:(\R^{N\times P} \times \R^{Q\times M}) \times \R^{N\times M} \to \R$:
\begin{equation}
\begin{aligned}
    \ell((L, R), X_s) := & \underset{Y_s, W_s} \min \| X_s - LY_sW_sR^T \|_{F}^2, \\
    \text{~~s.t.~~}  &\sparse(Y_s, W_s) \leq \kappa,
    \label{eq:obj}
\end{aligned}
\end{equation}
where $\sparse(\cdot, \cdot)$ is a sparsity-promoting function.  In this work, we choose
$\sparse(Y, W) := \max\{ \|Y\|_1, \|W\|_1 \}$ (noting that a bound on $\sparse(Y, W)$ implies a bound on the sparsity parameter of $Y\cdot W$).  The task, then, is to use the training set $\X$ to learn a dictionary pair that has a small expected loss on a new sample $X$.  
We denote by $\Constraint(k, \kappa)$ the set of matrices $Z \in \R^{P\times Q}$ that can be decomposed as $Z = YW$ satisfying the above dimension and sparsity constraints with constants $k$ and $\kappa$.

In a typical statistical learning scenario, it is assumed that the training samples $\X$ and the new (test) sample $X$ are drawn independently and identically distributed from an unknown data-generating distribution $\Dist$.  In this setting, one measures the performance of
a dictionary $(L, R)$ via its \emph{expected} loss on $X$, known as its \emph{risk}:
\begin{align}
    R_{\Dist}((L, R)):= R((L, R)) := \E_{X\sim \Dist}[ \ell((L, R), X) ].
\end{align}
A learning rule that produces a dictionary $(L, R)$ given $\X$ is called $(\epsilon, \delta)$-probably approximately correct (PAC) if, for every data-generating distribution $\Dist$, there is a sample size $m_0 := m_0(\epsilon, \delta)$ such that with probability $\geq 1-\delta$ over the choice of a size $m \geq m_0$ training set sampled iid from $\Dist$, the learning rule outputs a dictionary $(L, R)$ satisfying:
\begin{align}
    R_{\Dist}((L, R)) \leq \epsilon + \inf_{(L_*, R_*)} R_{\Dist}((L_*, R_*)).
\end{align}
The minimum $m_0$ for which this bound holds is called the \emph{sample complexity} of dictionary learning.
We say that the learning rule is PAC if it is $(\epsilon, \delta)$-PAC for all $(\epsilon, \delta)$ arbitrarily close to $0$.  We denote by $\Hyp$ (for ``hypothesis class'') the set of dictionaries over which the infimum is taken. The most fundamental learning rule is called \emph{empirical risk minimization (ERM)}, which we define in our context next. The \emph{empirical risk} $\hat{R}(h, \X)$ of a dictionary matrix pair $h$ on a dataset $\X := (X_1, ..., X_S)$ is:
\begin{align}
    \hat{R}(h, \X) := S^{-1} \sum_{s \in [1..S]} \ell(h, X_s).
\end{align}
The ERM learning rule chooses a hypothesis that minimizes the empirical risk given an input dataset.  
ERM is foundational to statistical learning because it is a \emph{universal} learning rule -- whenever a hypothesis class is learnable with finitely many samples, it is learnable via ERM. \vspace{0.1in}

\fbox{
\begin{minipage}{.9\linewidth}
    \textbf{Main algorithmic problem:} Given a dataset $\X := (X_1, ..., X_S)$, solve the ERM optimization problem in our sparse, low-rank dictionary learning setting: %to produce a dictionary pair $L, R$.  That is, given a dataset , our goal is to solve:
    \begin{equation}
        \begin{aligned}
            \argmin_{L, R, Y, W} &\frac{1}{S}\sum_{s \in [1..S]} \|X_s - LY_sW_sR^{T}\|_F^2, \\
            \text{~~s.t.~~} &\max\{\|Y_s\|_1, \|W_s\|_1\} \leq \kappa.
        \label{eq:main-problem}
        \end{aligned}
    \end{equation}
    \textbf{Main statistical problem:} Bound the \emph{sample complexity} of ERM for our problem.  
\end{minipage}
}

%Next, we give a sample complexity bound for two-way dictionary learning, which generalizes the sample complexity results for single-dictionary learning given in~\cite{JMLR:v12:vainsencher11a}.  In the follwing Section~\ref{sec:our-algorithm}, we give our algorithm for two-way dictionary learning.

%\begin{equation}
%    \begin{aligned}
%         f(L, R, Y, W) = \underset{L, R, Y, W} {\mathrm{argmin}} \hspace{0.1cm} & \sum_{s=1}^{S} (|| X_s - L Y_s W_s R^T ||_F^2  \\
%         &+ f(Y_s) + f(W_s))
%    \end{aligned} 
%    \label{eq:obj}
%\end{equation}
%where $f$ is any sparse function.
%
%In our solution, we use $L_1$ norm for the sparsity function $f$.

\subsection{Sample complexity upper bounds %for two-way dictionary learning.
}
\label{sec:sample-complexity}
%We next give a sample complexity bound for the two-way dictionary learning problem. 
%To state it, we define the \emph{empirical risk} $\hat{R}(h, \X)$ of a dictionary matrix pair $h$ on a dataset $\X := (X_1, ..., X_S)$ as:
%\begin{align}
%    \hat{R}(h, \X) := \frac{1}{S} \sum_{j=1}^S \ell(h, X_j).
%\end{align}
We next show two upper bounds for the sample complexity of dictionary learning. Each takes a form of a \emph{generalization} or \emph{uniform convergence bound}~\cite{UnderstandingMachineLearning}: a high-probability upper bound on
$\sup_{h} |R(h) - \hat{R}(h, \X)|$ that holds for all data-generating distributions of interest, where $R(h)$ is the risk of hypothesis $h$. Our results provide an accuracy guarantee on the ERM learning rule in terms of the number of samples $S$.
\begin{theorem}[{ Generalization bound for two-way dictionary learning}]
    \label{thm:dictionary-learning-generalization-bound}
    \label{THM:DICTIONARY-LEARNING-GENERALIZATION-BOUND}

    Let $\Dist$ be a distribution on $\Omega := \R^{N\times M}$ such that almost surely (i.e., with prob. $1$), $\|X\|_F \leq C$. Let $\Hyp$ denote the hypothesis class given by pairs of matrices $(L, R) \in \R^{N\times P} \times \R^{M\times Q}$
    satisfying the normalization condition for any pair $(i,j)$ of left $L_i$ and right $R_j$ column atoms:  %and   recalling that dictionary atoms in $L$ and $R$ are their columns, we require that for every column index $i$ of $L$ and row index $j$ of $R^{T}$, 
    $\|L_i\cdot (R^T)_j\|_{F} \leq 1$, where $\|\cdot \|_F$ denotes the Frobenius norm of a matrix.
    Then, for all $x > 0$, with probability at least $1 - e^{-x}$ over $S$ samples $\X := (X_1, ..., X_S)$
    iid from $\Dist$, we have, for all $h \in \Hyp$,
    \begin{equation}
        \footnotesize
        \begin{aligned}
        &|R(h) - \hat{R}(h, \X)| \\
        &~~\leq~~ \frac{2}{\sqrt{S}} + C^2 \left( \sqrt{\frac{x}{2S}} + \sqrt{ \frac{ (N P + MQ) \log(8\sqrt{S}\kappa^2) }{2S} } \right), \vsa
        \end{aligned}
    \end{equation}
    where we recall that $\kappa$ is the sparsity constraint on the coding matrices. 
\end{theorem}
%We provide the proof in Section~\ref{sec:sample-complexity-proof}.
% \input{tex/7_proof-sample-complexity-upper-bound}

For large enough sample sizes $S$, the bound in Theorem~\ref{thm:dictionary-learning-generalization-bound} decreases 
monotonically with $S$ (when all other parameters are fixed).  Thus, if one wishes to ensure a risk bound of $\epsilon$ 
with probability at least $1-\delta$ for every possible data-generating distribution, one can set $e^{-x} = \delta$ in the above expression and the entire expression equal 
to $\epsilon$, and solve for $S$ in terms of $\epsilon$ and $\delta$ (possibly using a numerical method).  Thus, as is 
well known in statistical learning theory, a generalization bound translates to a sample complexity bound.

We contrast with similar results in~\cite{TOITSampleComplexity}: our bound is in terms of hard sparsity and rank constraints on the coding matrix, whereas the bound in~\cite{TOITSampleComplexity} is for the $L_1$ regularized version of the problem, with no rank constraint.  Thus, it is difficult to interpret their bound in terms of exact constraints.

%\am{TODO: Explain the difference between our result and the one in Gribonval et al.  Differences are as follows: we state our bound in terms of hard sparsity and rank constraints on the coefficients.  Their bound is for the regularized version of the problem (so no good way to interpret it in terms of exact constraints), does not incorporate a rank constraint, and involves data distribution-dependent constants that are slightly worse than ours.}

%\am{TODO: See if we can get a fast rate result.}

% \todo{perhaps omit or move to discussion?} We stress that such bounds are \emph{worst-case}, in the sense that they must hold for all distributions on the data.  It may be that the bound is loose for some natural data-generating distributions, but tight for pathological distributions.  To our knowledge, existing distribution-independent sample complexity bounds have not been empirically evaluated on natural distributions.

%%%%%%%%%%%
We note that in Theorem~\ref{thm:dictionary-learning-generalization-bound}, 
the sparsity constraint $\kappa$ plays a logarithmic role, but the rank constraint does not appear to.  To probe the rank constraint further, we note that in the proof, the possibility of involvement of $k$ disappears when we appeal to the submultiplicativity of the $L_1$ norm to conclude that for a pair of coding matrices
$Y, W$, $\|YW\|_1 \leq \|Y\|_1 \|W\|_1$.  This inequality is only tight when the $L_1$ mass of $Y$ and $W$ is concentrated in a single column and row, respectively.  There exist data-generating distributions for which this occurs, and so we cannot tighten this inequality in the above bound.  However, data-generating distributions for which the submultiplicativity bound is \emph{not} tight arise in reality.
To understand this phenomenon further and determine to what extent tightening this inequality may involve the rank constraint, we next formulate a modification of the loss function to include an additional entropy constraint on the coding matrices, along with a corresponding generalization bound.

For matrices $Y \in \R^{P\times k}, W \in \R^{k \times Q}$, we denote by
$c(Y), r(W) \in \R^{k}$ the following vectors:
$ %\begin{align}
    c(Y)_\ell := \sum_{i=1}^P |Y_{i,\ell}|, 
    r(W)_\ell := \sum_{j=1}^Q |W_{\ell,j}|.
$ % \end{align}
We furthermore define $\hat{c}(Y) := \frac{c(Y)}{\|Y\|_1}, \hat{r}(W) := \frac{r(W)}{\|W\|_1}$.  We note that these are stochastic vectors.

For a stochastic vector $p \in \R^{k}$, we denote by $H_2(p)$ the R\'enyi entropy of order $2$ of the probability distribution corresponding to $p$.

\begin{definition}[Loss function with entropy constraint]
    \label{def:loss-function-entropy}
    Let $0 \leq \rho \leq 2\log k$.
    We define $\ell_\rho((L, R), X)$ as follows:
    \begin{align}
        \ell_\rho((L, R) X) := \min_{Y, W} \|X - LYWR^T\|_F^2, \\
            s.t. ~~ \sparse(Y, W) \leq \kappa, \\
                 \rho \leq H_2(\hat{c}(Y)) + H_2(\hat{r}(W)).
    \end{align}
\end{definition}
With this definition, our original loss function satisfies $\ell := \ell_0$. We have the following generalization of Theorem~\ref{thm:dictionary-learning-generalization-bound}:
\begin{theorem}[Generalization bound with R\'enyi entropy-constrained loss function]
    \label{thm:generalization-bound-entropy}
    Let $\rho \in [0, 2\log k]$.
    In the same setting as Theorem~\ref{thm:dictionary-learning-generalization-bound} except with the true and empirical risks defined with respect to the loss function $\ell_{\rho}$, we have, with probability at least $1-e^{-x}$,
     \begin{equation}
        \footnotesize
        \begin{aligned}
        &|R(h) - \hat{R}(h, \X)| \\
        &~~\leq~~ \frac{2}{\sqrt{S}} + C^2 \left( \sqrt{\frac{x}{2S}} + \sqrt{ \frac{ (N P + MQ) \log(8\sqrt{S}\kappa^2 \exp(-\frac{\rho}{2})) }{2S} } \right).
        \end{aligned}
    \end{equation}   
\end{theorem}
The proof is in Section~\ref{sec:renyi-entropy-bound-proof}.
\begin{remark}
    Theorem~\ref{thm:generalization-bound-entropy} has an implication for the sample complexity of the original learning problem (i.e., with the original loss function $\ell$) under a data distribution constraint.
    In particular, if we constrain the class of data-generating distributions to be such that risk minimizers with respect to the original loss function produce coding matrices satisfying the R\'enyi entropy constraint with probability $1$, then risk minimizers for both the constrained and unconstrained loss functions are the same.  This implies that if the number of samples is such that the ERM rule for the constrained loss function is $(\epsilon, \delta)$-PAC for this class of distributions and for $\ell_{\rho}$, it is also $(\epsilon, \delta)$-PAC for $\ell$.  In other words, Theorem~\ref{thm:generalization-bound-entropy} implies a sample complexity upper bound for the original learning problem under the aforementioned distributional constraint.
\end{remark}

In Theorem~\ref{thm:generalization-bound-entropy}, when $\rho \geq c\log k$ for some $c < 2$ (a mild constraint that enforces a very approximately uniform spread of $L_1$ mass across the columns of $Y$ and the rows of $W$), we have $\exp(-\rho/2) \leq k^{-c/2}$.  When, furthermore, $\kappa \leq \hat{\kappa} k$ for some fixed $\hat{\kappa} > 0$, we have
$ %\begin{align}
    \kappa^2 \exp(-\rho/2) \leq \hat{\kappa}^2 k^{2-c/2},
$ %\end{align}
so that we see that the sample complexity increases with $\hat{\kappa}$ (which one may think of as a \emph{per-column/row} sparsity constraint) and
$k$.  Empirically, we find that for appropriately chosen $k$, the coding matrices derived via our algorithm applied to real data have nearly maximal $\rho$, and $\hat{\kappa} := 10/3$ yields competitive performance with the rank-unconstrained model.  Thus, it is empirically reasonable to assume the parameter regime previously discussed, in which decreasing $k$ decreases the sample complexity.  
%We note, however, that the tightening of the sample complexity in Theorem~\ref{thm:generalization-bound-entropy} over that of Theorem~\ref{thm:dictionary-learning-generalization-bound} is only logarithmic.

%\todo{we need a bit more elaboration on what is the implication for the rank k}

%%%%%%%%%%
\subsection{Sample complexity lower bound. }
We note that the rank constraint does not appear in Theorem~\ref{thm:dictionary-learning-generalization-bound} and only appears in the logarithm of Theorem~\ref{thm:generalization-bound-entropy}.  To determine whether or not this is inherent or, alternatively, the result of a loose bound, we next present a sample complexity
\emph{lower} bound for the rank and sparsity-constrained dictionary learning problem.
\begin{theorem}[Sample complexity lower bound for rank-constrained dictionary learning]
    \label{thm:sample-complexity-lower-bound}
    Let $k=1$ and $\kappa > 0$ be arbitrary.  Let $P, Q$ be large enough
    but $O( \min\{N, M\})$. 
    %\am{Is this right?  $P$ and $Q$ too large may mess with the coherence bound.}
    Let $\Alg$ be an
    $(\epsilon, \delta)$-PAC learning rule for the dictionary learning problem with rank and sparsity constraints $k, \kappa$, respectively, with sample complexity $S$.
    Then we have that
    \begin{align}
        %\delta \geq 1- \frac{S\cdot \log^2(PQ)}{\Omega((NP+MQ) \log(1/\epsilon))},
        S \geq \frac{(1-\delta)\log(1/\epsilon) \Omega(NP+MQ)}{\log^2(PQ)},
    \end{align}
    where the $\Omega(\cdot)$ is as $N, M \to \infty$. 
\end{theorem}
The proof is in Section~\ref{sec:sample-complexity-lower-bound-proof}.

An implication of Theorem~\ref{thm:sample-complexity-lower-bound} is that 
the rank constraint on the coding matrix cannot improve the sample complexity uniformly over all data-generating distributions except by at most a polylogarithmic factor in $P, Q$.  In fact, the ``hard'' data-generating distribution that we use in the proof is quite natural: a single ``ground-truth'' two-way dictionary is fixed, and each data sample is a randomly chosen atom, perturbed by Gaussian noise.

The empirical test error improvements that we see when imposing the rank constraint are thus likely not due to sample complexity improvements, but instead due to improvements in the algorithmic tractability of the optimization problem that we solve, compared to those of competing methods.  We also emphasize that in addition to test error improvement, the rank constraint has the advantage of reducing the number of coding coefficients required to describe a dictionary representation of a data matrix.

\section{\ourmeth: Dictionary learning for low-rank sparse coding}
\label{sec:our-algorithm}

%\todo{Many changes comparing with the KDD version, please review the whole section}

%Contrained is hard to optimize directly
%So we transform into regularized
%this transformation preserves the optimum (Theorem 2)
%Then we do alternating optimization -> overall algorithm
%It converges (Theorem 3)
%Here is a version with missing values
%Limitations
In this section, we describe our algorithm for two-way dictionary learning, establish the convergence of the objective function values of its iterates and discuss potential limitations.  
Since the direct empirical risk minimization to solve the constrained learning problem is difficult due to the $L_1$ constraints, we first reformulate the objective as an $L_1$ regularization 
% terms to solve the sparse coding problem:
and show that its solution is also a solution to the original problem. The regularized objective is as follows: \vsc
\begin{equation}
  \begin{aligned}
       % f(\X) = 
       \underset{L, R, Y, W} {\mathrm{argmin}} \sum_{s \in [1..S]} ( &|| X_s - L Y_s W_s R^T ||_F^2 
       + \lambda_1 ||Y_s||_1 + \lambda_2 ||W_s||_1).
  \end{aligned}
  \label{eq:obj_aodl}
\end{equation}

Our next theorem shows that exactly solving this regularized version of the problem provides an exact solution to the original constrained problem (proof available in Appendix~\ref{sec:equivalence}).  Similar, but not identical, statements are well known
to hold for a single optimization variable~\cite{SparseModelingBookFNT}.  We include the statement for our case and its proof for completeness.
\begin{theorem}[{\bf Constrained optimization via regularization}]
    \label{thm:regularization-solves-constrained}
    \label{THM:REGULARIZATION-SOLVES-CONSTRAINED}
    For each $\kappa > 0$, there exists a pair $(\lambda_1, \lambda_2)$ such that the sparse coding subproblem from Eq.~\ref{eq:obj_l1} below is an exact solution to the $\kappa$-constrained problem from Eq.~\ref{eq:obj}. As a result, a solution of the overall regularized dictionary learning from Eq.~\ref{eq:obj_aodl} is a solution to the constrained version from Eq.~\ref{eq:main-problem}.
\end{theorem}
%
% \am{The objective above is not correct.  The stuff below is.}
% \begin{align}
%     (Y_*,s, W_*,s) &:= \argmin_{Y_s, W_s} \|X_s - LY_sW_sR^{T}\|_F^2 + \lambda_1 \|Y_s\|_1 + \lambda_2 \|W_s\|_1, ~~ s \in \{1, ..., S\}\\
%     (L_*, R_*) &:=  \argmin_{L, R} \sum_{s=1}^S \|X_s - LY_sW_sR^{T}\|_F^2.
%     \label{eq:obj_aodl}
% \end{align}
Although the objective in (\ref{eq:obj_aodl}) is not jointly convex in $(L, R, Y, W)$, each sub-problem is convex, leading to a natural alternating optimization solver we call Alternating Optimization (low rank) Dictionary learning (\ourmeth). Our algorithm alternates between sparse coding and dictionary learning: % (given $L, R$, to solve for $Y, W$) and optimization of dictionaries (given $Y, W$ to solve for $L, R$), until convergence. 

\begin{algorithm} [t]
\footnotesize
    \caption{\ourmeth}
        \begin{algorithmic}[1]
        \State {\bf Input:} Samples $X_s, s \in [1 \cdots S] $, dictionary sizes $P$ and $Q$, encoding rank $k$ and sparsity params. $\lambda_1, \lambda_2$
        \State {\bf Output:} Dictionaries $L \in\R^{N\times P}$ and $R\in\R^{M\times Q}$ and encodings $(Y_s,W_s),\forall s \leq S$
        \State Initialize $L,  R$ with unit-norm atoms
        \Repeat
            %\State // \emph{Stage I: Sparse coding}          
            \For{$s = [1 \cdots S$]}
                \State $[Y_s, W_s] = \text{LRSC}(X_s, L, R, \lambda_1, \lambda_2,k)$ ~~~~~ //in Appendix~\ref{appendix:alg}
            \EndFor
            %\State // \emph{Stage II: Dictionary updates}
            \State {$L = \text{normalize}((\sum X_s R W_s^T Y_s^T)(\sum Y_s W_s R^T R W_s^T Y_s^T)^{-1})$}
            \State {$R = \text{normalize}((\sum X_s^T L Y_s W_s)(\sum W_s^T Y_s^T L^T L Y_s W_s)^{-1})$}
        \Until{Convergence or fixed max iterations}
        \end{algorithmic}
    \label{alg:aodl}
\end{algorithm}

\noindent{\bf Stage I: Sparse coding.} For fixed $L, R$ and for each $s \in \{1, ..., S\}$, estimate $Y_s, W_s$:
\begin{equation}
    \begin{aligned}
         % (Y_{*,s}, W_{*,s}) := 
         \underset{Y_s, W_s} {\mathrm{argmin}} ( &|| X_s - L Y_s W_s R^T ||_F^2 
         + \lambda_1 ||Y_s||_1 + \lambda_2 ||W_s||_1).
    \end{aligned}
    \label{eq:obj_l1}
\end{equation}
%\am{TODO: We need to revise the above, because it's not what the algorithm is actually doing.  The algorithm solves the sparse coding problem, then solves the unregularized version of the problem to get the dictionaries.  This happens to be good for the theory.}
%\am{BEGIN EDIT}

%\am{EDIT HERE}
%In Appendix~\ref{sec:equivalence}, we show that exactly solving this regularized version of the problem provides an exact solution to the original constrained problem.
\noindent{\bf Stage II: Dictionary learning.} For fixed sparse coding matrices $Y_{s}, W_s$, in each iteration of our algorithm, we then
solve the following optimization problem:
\begin{align}
    \underset{L, R} {\mathrm{argmin}} \sum_{s \in [1..S]}\| X_s - L Y_s W_s R^T \|_F^2.
    \label{eq:obj_dicts}
\end{align}
% We give the precise details next.

%\am{END EDIT}

\noindent{\bf Optimization algorithm, complexity and convergence.} The overall algorithm is provided in Alg.~\ref{alg:aodl}. After initialization (Step 3), we first perform low-rank sparse coding by solving Eq. \ref{eq:obj_l1} (Steps 5 - 7) for individual samples $X_s$ via ADMM (detailed steps of LRSC in Appendix~\ref{appendix:alg}). In the second stage (Steps 8-9), we update the two dictionaries $L, R$ using the gradient projection method to solve Eq.~\ref{eq:obj_dicts}. The sparse coding stage is dominated by an eigendecomposition (see Appendix~\ref{appendix:alg}) with a complexity of $O(P^3 + Q^3 + k^3)$ per sample sparse coding update in the worst case. The dictionary learning stage is dominated by the matrix inversions with complexity $O(T + P^3 + Q^3)$, where $T$ is the product of the maximum $3$ values among $\{N, M, P, Q, k\}$. Details of all derivations and dictionary initialization strategies are provided in Appendix~\ref{appendix:alg}.

We can show that the objective function values of iterates of AODL converge (proof available in Appendix~\ref{sec:convergence-of-alternating-minimization}):

%%%%%%%%%%%%
%\noindent{\bf Convergence of the alternating minimization.}

%We have the following theorem guaranteeing convergence of our algorithm.
\begin{theorem}[{\bf Convergence of AODL}]
    \label{thm:convergence-aodl}
    \label{THM:CONVERGENCE-AODL}
    %\am{TODO: Replace this with the convergence of iterates theorem.}
    Let $L^{(k)}, R^{(k)}, Y^{(k)}, W^{(k)}$ denote the dictionaries and sparse coding matrices after
    $k$ iterations of AODL.  Let $F^{(k)} := \sum_{s=1}^S \|X_s - L^{(k)}Y^{(k)}_s W^{(k)}_s R^{(k)T}\|_F^2$,
    and let $G^{(k)}$ denote the regularized version of $F^{(k)}$.  Then
    as $k\to\infty$, both $F^{(k)}$ and $G^{(k)}$ converge.
\end{theorem}
%The proof of Theorem~\ref{thm:convergence-aodl} is in Appendix~\ref{sec:convergence-of-alternating-minimization}.
We note that it is generally difficult to obtain much stronger convergence results, such as results about convergence
of iterates themselves, for alternating minimization procedures.  In particular, the lack of strict convexity of each coordinate block of the objective function in our setting makes it impossible to apply existing alternating minimization iterate convergence results.   Despite this, such procedures are commonly used heuristics with good empirical performance.

%%%%%%%%%%%%%%%%%%%%
\noindent{\bf \ourmeth in the presence of missing values.} 
To handle samples with missing values and perform imputation, we also introduce a version of our problem with a sample-specific $0-1$ mask $\Omega_s$, where the data fit term from Eq.~\ref{eq:obj_aodl} is replaced by $|| \Omega_s \odot (X_s - L Y_s W_s R^T) ||_F^2$.
%\begin{equation}
%    \begin{aligned}
%         % f(\X) = 
%         \underset{L, R, Y, W} {\mathrm{argmin}} \sum_{s=1}^{S} (&|| \Omega_s \odot (X_s - L Y_s W_s R^T) ||_F^2 + \lambda_1 ||Y_s||_1 + \lambda_2 ||W_s||_1),
%    \end{aligned} 
%    \label{eq:obj_missing}
%\end{equation}
% \am{TODO: The above equation needs to be revised in the same way as (\ref{eq:obj_aodl}).}
We derive an ADMM solution for the missing value objective and detail it in Appendix~\ref{appendix:alg-missing}.

\noindent{\bf Limitations} 
\ourmeth learns good dictionaries and succinct codes when samples allow encoding with low rank $k$. If the data input is not low-rank, \ourmeth will require $k \approx min(P, Q)$ in the worst case, resulting in coding matrices $Y$ and $W$ that exceed the size of the single matrix $Z$ in the CMOD model. In such sub-optimal settings our scalability and model size advantages would diminish compared to the CMOD baseline. Additionally, while we establish sample complexity and objective function convergence for our problem and method, 
establishing iterate convergence in this setting, even to a local optimum, remains a challenging open problem. %we plan to investigate the rate of convergence theoretically in future work. % since the $nnz(Z)$ is much less than ($nnz(Y) + nnz(W)$). Only when Y and W are low rank and $k << min(P, Q)$, can we achieve the goal that the coefficients in Y and W is less than $nnz(Z)$.

\vsb
\section{Experimental evaluation}
\label{sec:experiments}

We characterize \ourmeth's strengths and weaknesses in comparison to state-of-the-art baselines on a range of datasets. We quantify reconstruction quality (as RMSE) and compactness (as number of non-zero coefficients NNZ) of competing models as well as their ability to impute missing values. We also investigate the patterns in the learned dictionaries and empirically test our theoretical results. All tests are conducted on an Intel(R) Xeon(R) Gold 6138 CPU @ 2.00GHz and 251 GB memory server using MATLAB's R2023a 64-bit version. %The representation quality is quantified as the root mean squared error (RMSE) between the data and the learned representation. %We also provide a case study using LA traffic speed data (Road) by visualize the top four learned atoms by \ourmeth, which clearly matches with real daily traffic patterns. Also, an ablation test is provided to show the limiations of our method. 
An implementation of our method is available at: \url{https://tinyurl.com/AODL-demo}.

%focuses on the ground truth dictionary recovering and the representation quality of competing methods on synthetic and real-world datasets listed in Table \ref{table:datasets}. 

\begin{table*}[t]
\setlength\tabcolsep{1 pt}
\footnotesize
\centering
 \begin{tabular}{|c| c| c| c| c |c| c | c | c| c | c| c | c| c|c| c|c| c|c|c|c|} 
 \hline
 {\multirow{2}{*}{\bf{Dataset}}} & {\multirow{2}{*}{\bf{\#Nodes}}}  & \bf{\#Time}  & {\multirow{2}{*}{\bf{Res.}}} & {\multirow{2}{*}{\bf{\#Samp.}}} & {\multirow{2}{*}{\bf{Split by}}} & {\bf Associated} & {{\bf{Max.}}} & \multicolumn{2}{|c|}{\bf{TGSD}} & \multicolumn{2}{|c|}{\bf{SeDiL}} & \multicolumn{2}{|c|}{\bf{OSubDiL}} & \multicolumn{2}{|c|}{\bf{CMOD-OMP}} & \multicolumn{2}{|c|}{\bf{CMOD-ADMM}} & \multicolumn{2}{|c|}{\bf{\ourmeth (ours)}}\\
 \cline{9-20}
  &   &  \bf{steps} &  & & & {\bf graph} & {\bf NNZ} & RMSE& Time & RMSE& Time & RMSE& Time & RMSE & Time & RMSE & Time & RMSE& Time \\
  \hline
  Synthetic  & 20 & 3000 & - & 100 & ``Time'' & - & 35 & \bf{0.6}  & {\bf 0.5} & 8.8 & 19 & 8.4 & 3.6k & 6.2 & \underline{16} & 8.0 & 36 & \underline{0.9}  & 32 \\ 
 \hline
 Road~\cite{LA_traffic}  & 2780   & 8640 & 5m & 30 & Time & Road net. & 3k & \underline{7.78}   & {\bf 38}  & 8.2 & 4.9k & - & - & 15.3  & 170k& 26.7   & 448  & \bf{3.18} & \underline{255} \\ 
 \hline
   Twitch\cite{twitch}  & 9000   & 512 & 1h & 30 & Nodes & Co-views &  3k& 1.29   & {\bf 98}  & 1.29  & 315 & - & - & \underline{1.23}  & 3.7k & 1.26   &  \underline{203}  & \bf{1.11}  & 267   \\ 
 % \hline
 %  Twitch (small)~\cite{twitch}  & 9567   & 512 & 1h & Shared audience & 1.34 & 1388 & 1.32 & 709 & 1.33 & 142 & \bf{1.25} & 137 & 1.29 & \bf{41}\\ 
 \hline
  Wiki~\cite{Wiki} & 11400   & 792 & 1h  & 38 & Nodes & Co-clicks &1.5k& \underline{11.5}  & {\bf 126}  & 12.7 & 6.6k & - & - & 19.6  & 2.1k & 21.3   & 329  & \bf{4.5}  & \underline{283}  \\
  \hline
  MIT~\cite{eagle2006reality} & 94   & 8352 & 5m  & 29 & Time & Messages &1.5k&  5.4  & {\bf 11}   & 5.4  & 784 & 5.2 & 403k & \underline{4.2}  & 3.4k &  4.9   & \underline{101} & \bf{2.6}  & 122  \\
  \hline
    Air~\cite{strohmeier2021crowdsourced} & 5500   & 124 & 6h  & 25 & Nodes & Flight net. &1k& 2.3   & {\bf 26} & 1.5 & 165 & 1.7 & 259k & 1.3  &1k &  \underline{1.3}   &  \underline{103} & \bf{1.0} & 119 \\
  \hline
  % Covid~\cite{covid} & 3047 & 678 & 1d  & Spatial k-NN & 31969 & 551 & 23908 & 2668 & 21320 & 267 & \bf{204} & 145 & 228 & \bf{88}\\
  % \hline
\end{tabular}
 \caption{\footnotesize Statistics of the evaluation datasets (left) and quality and running times for competing techniques at fixed maximum NNZ (right). Datasets have a temporal (\emph{\# Time steps}) and graph mode (\emph{\# Nodes}) and temp. resolution in col. \emph{Res.} %with corresponding analytical (used by TGSD) and learned dictionaries. The temporal resolution of each dataset is specified in column (\emph{Res.}). 
 We split the data into samples along the larger (Time/Nodes) mode (\emph{Split by} col.) and list the type of \emph{Associated graph}. The right sub-table shows the lowest reconstruction (\emph{RMSE}) and timing in seconds (\emph{Time}) results for competing techniques at a fixed maximal allowed model size (\emph{Max. NNZ}). The most accurate (lowest RMSE) and fastest results are bolded, while the second best are underlined.
 % \vsa\vsa
 }
\label{table:datasets}
\end{table*}

\subsection{Datasets and Experiment Setup}
\noindent{\bf Datasets.}
We employ synthetic and real-world datasets summarized in Tbl.~\ref{table:datasets}.
% and described in details in Appendix~\ref{appendix:data}. 
The real-world datasets span multiple domains: content exchange (\emph{Twitch}~\cite{twitch}), web traffic (\emph{Wiki}~\cite{Wiki}), sensor network readings (\emph{Road}~\cite{LA_traffic}, \emph{Air}~\cite{strohmeier2021crowdsourced}) and social interactions (\emph{MIT}~\cite{eagle2006reality}). 
% Detailed in the Appendix section \ref{appendix:data}.

\emph{Synthetic data.} We generate synthetic data according to the model $X_s = L Y_s W_s R^T + \epsilon$, where $\epsilon$ is Gaussian noise. Dictionaries $L \in \R^{20 \times 20}$ and $R \in \R^{30 \times 30}$ contain random unit-norm atoms and encodings are of rank $k = 3$, i.e., $Y \in \R^{20 \times 3}, W \in \R^{3 \times 30}$. For each sample $Y_s$ and $W_s$ each contains $15$ randomly selected coefficients with normally distributed in $\mathcal{N}(0,1)$. There are a total of $100$ training samples which are used to learn the dictionaries. 

\emph{Real-world datasets.}  We employ $5$ real-world datasets with temporal and spatial dimensions to be able to evaluate against competing techniques like TGSD~\cite{TGSD} employing analytical temporal and graph dictionaries. To prepare the datasets we follow the same protocols as prior work. The datasets span multiple domains: content exchange (\emph{Twitch}~\cite{twitch}), web traffic (\emph{Wiki}~\cite{Wiki}), sensor network readings (\emph{Road}~\cite{LA_traffic} and \emph{Air}~\cite{strohmeier2021crowdsourced}) and a social interaction (\emph{MIT}~\cite{eagle2006reality}). In order to create multiple samples for dictionary learning, we slice the data on the larger of its two dimensions (time or spatial/graph extent). We also consider alternative slicing (see Fig.~\ref{fig:twitch_slice_short}) that confirms our comparative analysis findings.% Discussion about slicing the data in the smaller dimension is provided in the ablation test sub-section.

Twitch~\cite{twitch} contains viewer-streamer temporal interactions. We create a graph among viewers and add an edge between a pair of viewers if they viewed the same stream at least 3 times. We use the largest connected component of the co-viewing graph. Values in data samples $X\in\R^{9000 \times 512}$ represent the number of minutes in any given hour that a viewer spent viewing any streams (i.e., their level of activity). We slice the data randomly into $30$ samples along the graph dimension. 

The Wiki dataset~\cite{Wiki} records hourly number of views of Wikipedia articles over $792$ hours. A co-click graph among articles is constructed by placing edges between articles with at least $10$ pairwise click events (clicked by the same IPs) within a day. A breadth-first-search (snowball) subgraph of $11400$ around the China article is selected and then sliced into $38$ samples along the graph dimension. 

The Road~\cite{LA_traffic} dataset consists of $2780$ highway speed sensors in the LA area. We use the average speed for $30$ days at $5$-minute interval ($8640$ timesteps) as our signal matrix. The graph is based on connected road segments. We slice the data into 30 samples on the time dimension. 

MIT~\cite{eagle2006reality} is a communication dataset of timestamped messages between users and a weighted social graph. We split the data along its time dimension. 

The Air~\cite{strohmeier2021crowdsourced} dataset contains the the number of flights between connected airports (nodes), while the edges connect flight origin and destination. Temporal snapshots represent the number of incoming flights over a 6 hour window. We remove small airports with less than $8$ flights a day and split the data along its graph dimension obtaining 25 samples.

A note on the selection of real-world datasets: We employ the above spatio(graph)-temporal (ST) data as they have been shown to align to low-rank encoding models due to clustered (shared) behavior present in the temporal and spatial mode observed in the TGSD~\cite{TGSD} baseline. Different from the baseline, however, our method AODL does not use graph or temporal information associated with the data but learns temporal and spatial dictionaries from scratch. The learned dictionaries perform better than analytical ones employed by TGSD~\cite{TGSD} as demonstrated in in our comparative analysis in the main paper. The ST datasets also allow us to perform qualitative analysis (case studies) and also create controlled synthetic data from GT dictionaries akin to the setup in TGSD.

\noindent{\bf Baselines.} We compare against one analytical (fixed) dictionary baseline \emph{TGSD}~\cite{TGSD} and three multi-dictionary learning approaches \emph{CMOD}~\cite{shahriari2021new}, \emph{SeDiL}~\cite{hawe2013separable}, and \emph{OSubDiL}~\cite{ghassemi2019learning}  . 
% We experiment with two versions of CMOD: \emph{CMOD-OMP} employing 2D-OMP as a sparse coding solver and an ADMM sparse-coding alternative. While they exhibit similar quality (RMSE in Tbl.~\ref{table:datasets}), the latter scales orders of magnitude better and we adopt it exclusively for experiments beyond Tbl.~\ref{table:datasets}. We similarly report the relatively older baseline SeDiL only in Tbl.~\ref{table:datasets} since it is prohibitively slow, sensitive to its hyperparameters and does not produce better quality solutions than newer baselines. 

% A detailed description of the baselines and justification for their selection/use is available in Appendix~\ref{appendix:baselines}.

%\noindent{\bf Metrics:} We measure the reconstruction quality as the element-wise root mean squared error between a sample $X_s$ and its reconstruction $X_s'$: $\text{RMSE} = \frac{1}{S}\sum_{s}^{S} \sqrt{\frac{\sum_{(i,j)} ({X_s}_{(i,j)} - {X_s'}_{(i,j)})^2}{|X_s|}},$ where $|X_s|$ denotes the number of elements in $X_s$. 
%We employ the average number of non-zero coefficients (NNZ) across samples to quantify the size of the encodings produced by competing techniques.

TGSD employs analytical dictionaries but has a low-rank model for the encoding similar to \ourmeth. Within this baseline we employ the authors' implementation and the GFT dictionary based on data graphs for graph dimensions and the Ramanujan periodic dictionary for the time dimensions. 

Shahriari-Mehr et Al.~\cite{shahriari2021new} proposed two methods 2D-CMOD and 2D-MOD for 2D dictionary learning among which 2D-CMOD converged faster to a better solution according to the authors' experiments. Hence we adopt it as a baseline, and we call it \emph{CMOD} for brevity in all experiments. We experiment with two versions of CMOD: \emph{CMOD-OMP} which is the originally proposed method that uses 2D-OMP as a sparse coding solver; and a variant \emph{CMOD-ADMM} employing an ADMM solver for the sparse coding step. While they produce similar quality results (see columns 11 and 12 in Tbl.~\ref{table:datasets}), the OMP version is about 3 orders of magnitude slower and required over $47$ hours for a single run on some datasets when the target number of coding coefficients is large. As a result, in all experiments (apart from Tbl.~\ref{table:datasets}) we employ the CMOD-ADMM version. 

\emph{SeDiL}~\cite{hawe2013separable} is an older baseline which learns dictionaries employing a conjugate gradient approach. In our experiments, we found that it is sensitive to its hyperparameters and even when tuned extensively, it produces similar or worse results than the newer baseline CMOD~\cite{shahriari2021new} while requiring orders of magnitude more time to complete on some datasets (see Tbl.~\ref{table:datasets}). Furthermore, our observations of SeDiL's performance are consistent with those reported by the authors of CMOD~\cite{shahriari2021new}. As a result, we report \emph{SeDiL} results only in Tbl.~\ref{table:datasets} and omit it from the comparisons in the rest of the experiments. 

{\emph{OSubDiL}~\cite{ghassemi2019learning} learns a Kronecker product of the two sub-dictionaries online, and it requires less memory since it processes one data sample per iteration. However, because it maintains a Kronecker product of sub-dictionaries instead of the sub-dictionaries themselves, it still requires a large amount of memory. The other two models in the same family (STARK and TeFDiL) are omitted since they use much more resources than \emph{OSubDiL}. In contrast, CMOD and \ourmeth also process each sample separately and store the sub-dictionaries directly. As a result, we only report \emph{OSubDiL} results for the small datasets MIT and AIR in Tbl.~\ref{table:datasets}. } 

\noindent{\bf Tuning.} We tune the hyperparameters of all competing techniques by an extensive grid search. Details of the grid search and best parameter values for all baselines are presented in Appendix~\ref{appendix:hyper}.

\noindent{\bf Metrics.} We measure the reconstruction quality as the element-wise root mean squared error between a sample $X_s$ and its reconstruction $X_s'$: $\text{RMSE} = \frac{1}{S}\sum_{s}^{S} \sqrt{\frac{\sum_{(i,j)} ({X_s}_{(i,j)} - {X_s'}_{(i,j)})^2}{|X_s|}},$ where $|X_s|$ denotes the number of elements in $X_s$. 
We employ the average number of non-zero coefficients (NNZ) across samples to quantify the size of the encodings produced by competing techniques. We also measure actual running time for competing techniques to compare their scalability.

\noindent{\bf Dictionary sizes:} While the data dimensions $N$ and $M$ of the learned dictionaries $L \in \R^{N \times P}$ and $R \in \R^{M \times Q}$ are predetermined by the size of the input signals $X_s$, one has a choice when it comes to the number of atoms in each dictionary ($P$ and $Q$). We employ square dictionaries in all experiments (i.e., $P=N$ and $Q=M$) as this is the minimum number of atoms to form a basis for each of the data dimensions. We also keep the sizes of the analytical dictionaries employed by TGSD the same as those learned by the rest of the competing techniques (GFT is square and for Ramanujan we employ the first $Q$ atoms when ordered from low to high periods). %It is important to note, however, that the sparse coding framework is often employed with overcomplete dictionaries and larger dictionaries can also be considered, though increasing the dictionary sizes did not result in significant changes of the performance of competing techniques on our datasets. 

%To measure the similarity of the learned dictionary to a ground truth (GT) dictionary we employ in synthetic tests, we use the F1 score ($\text{F1} = \frac{2}{\frac{1}{\text{precision}} + \frac{1}{\text{recall}}}$, where precision is the average of the best alignments of the learned dictionary atoms against the GT dictionary atoms, and recall similarly compares GT atoms against the learned dictionary atoms. 

\begin{figure*}
    \centering
    \subfigure [Road]
    {
        \includegraphics[width=0.19\linewidth]{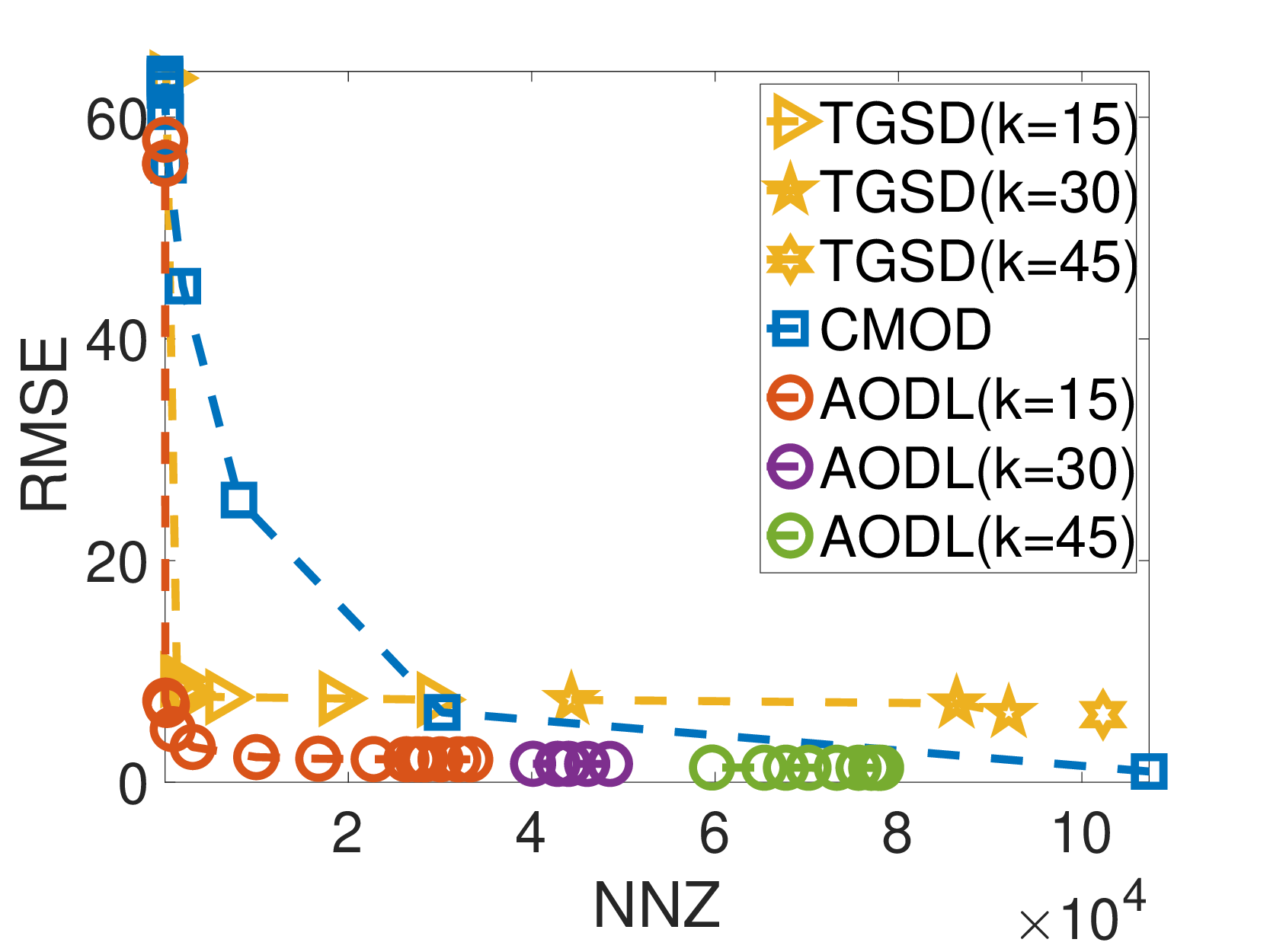}
        \label{fig:real_road}
    }\hspace{-0.16in}
    \subfigure [Twitch]
    {
        \includegraphics[width=0.19\linewidth]{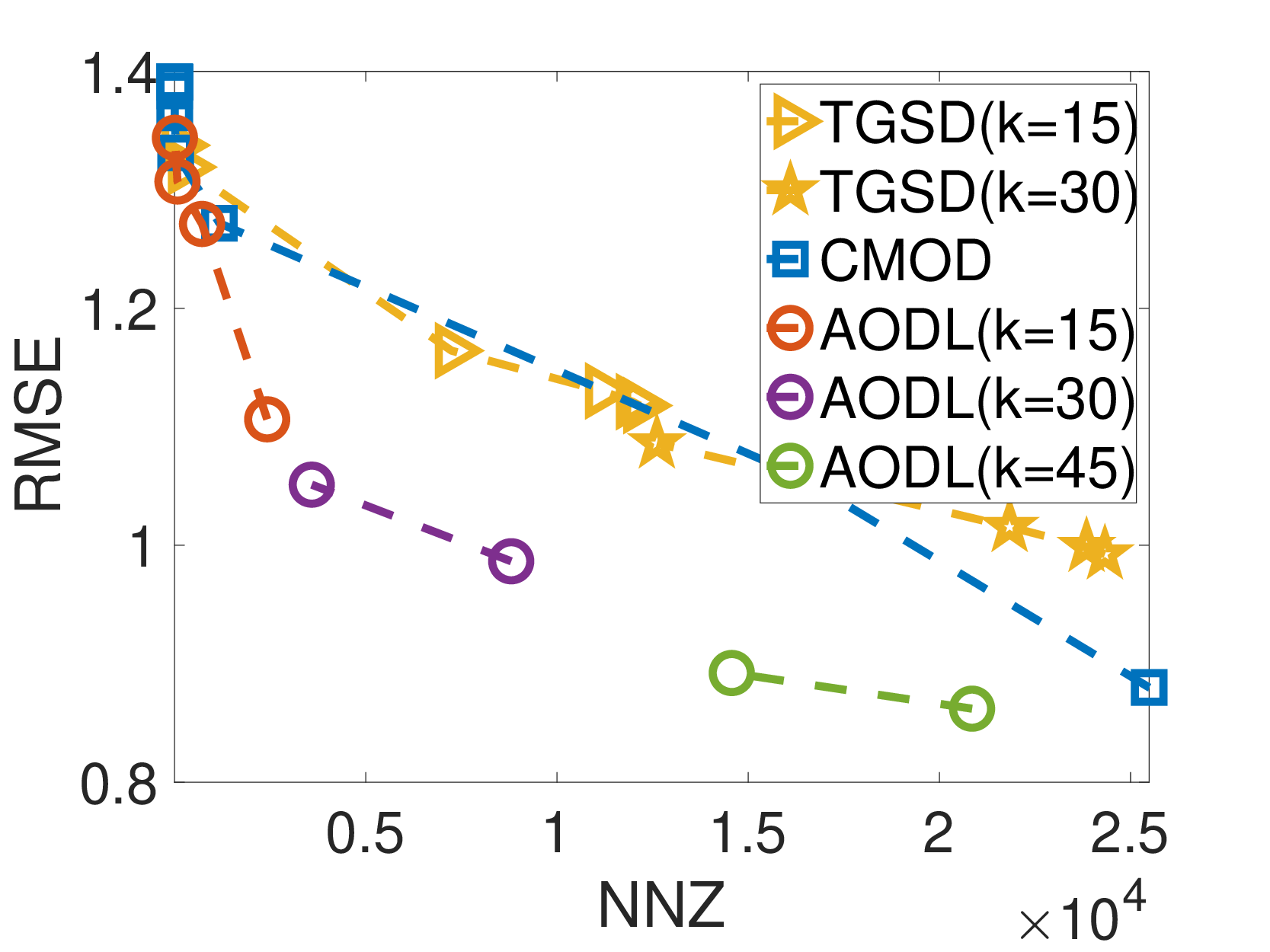}
        \label{fig:real_twitch}
    }\hspace{-0.16in}
    \subfigure [Wiki]
    {
        \includegraphics[width=0.19\linewidth]{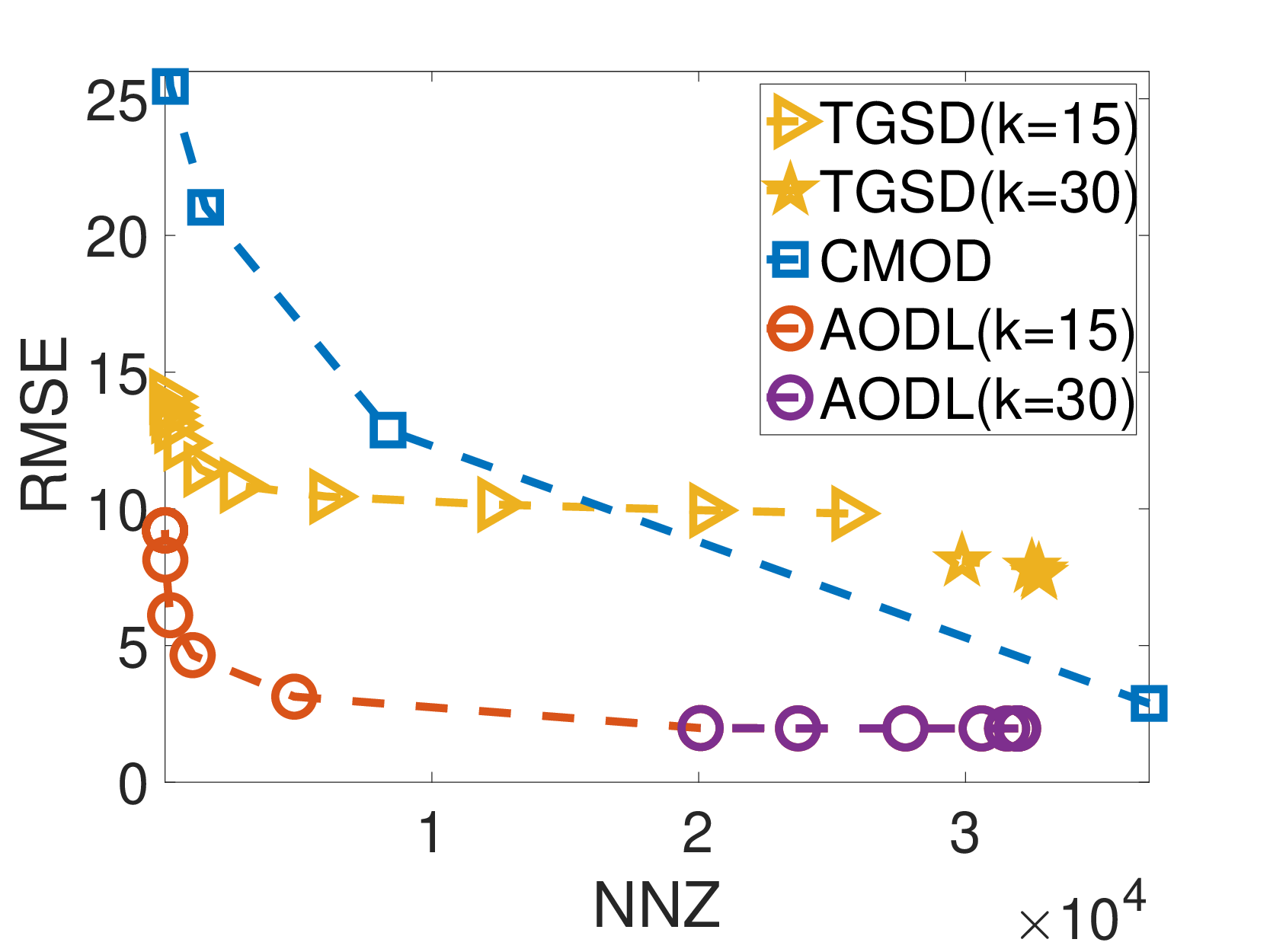}
        \label{fig:real_wiki}
    }\hspace{-0.16in}
    \subfigure [MIT]
    {
        \includegraphics[width=0.19\linewidth]{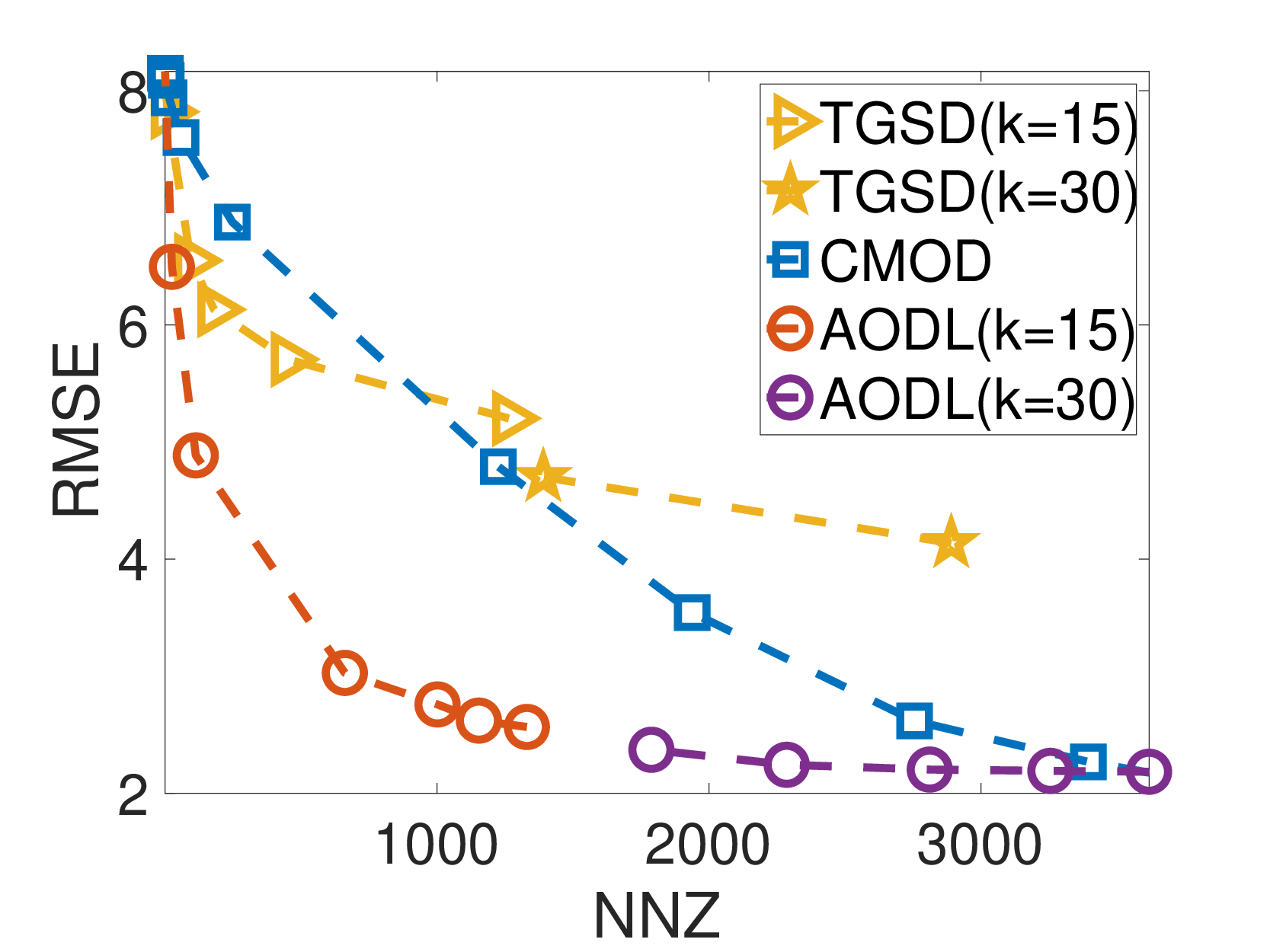}
        \label{fig:real_mit}
    }\hspace{-0.13in}
    \subfigure [Air]
    {
        \includegraphics[width=0.19\linewidth]{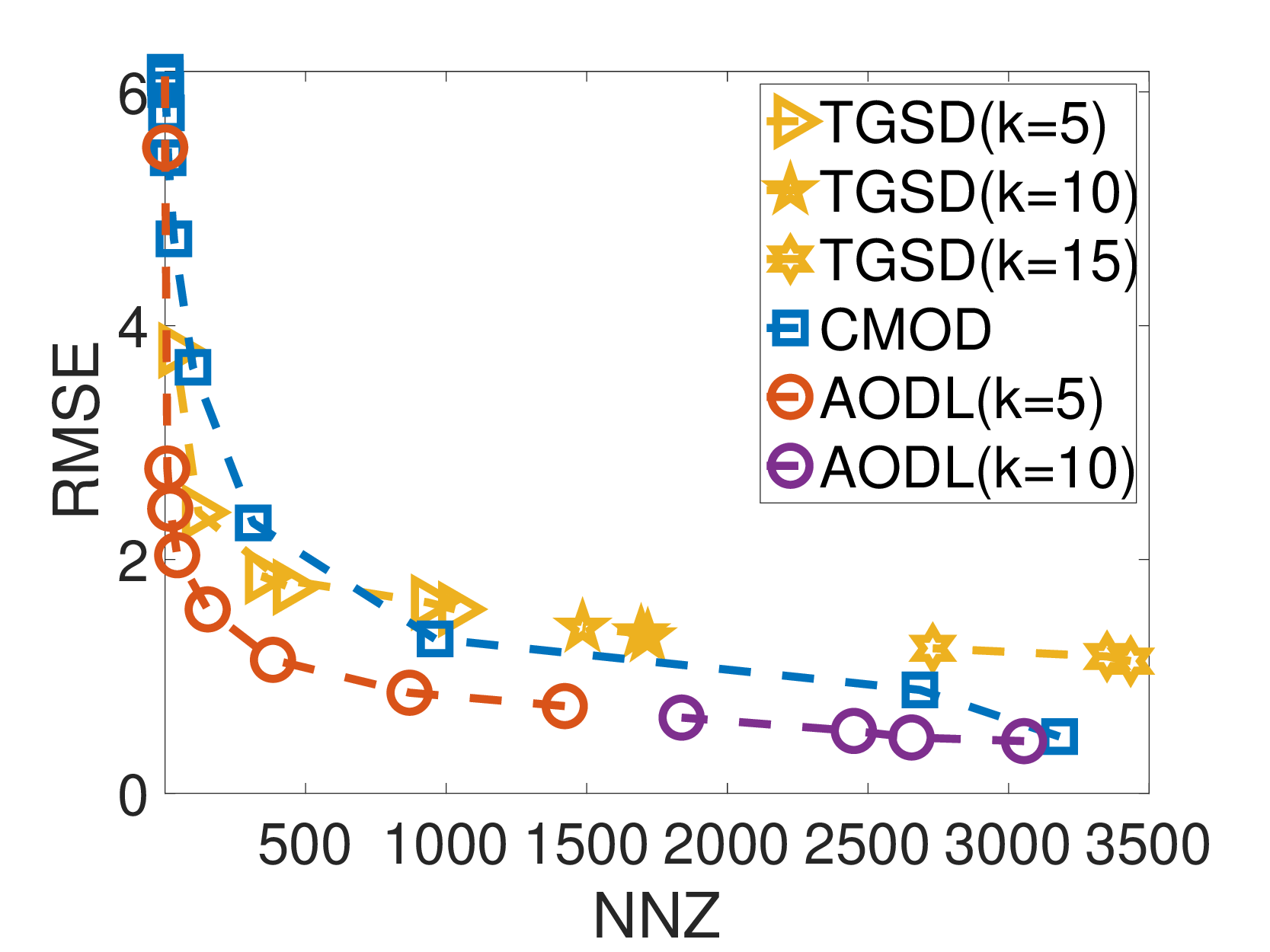}
        \label{fig:real_air}    
    }\hspace{-0.16in} \\
    \vsa
     \subfigure [Impute: Road]
    {
        \includegraphics[width=0.19\linewidth]{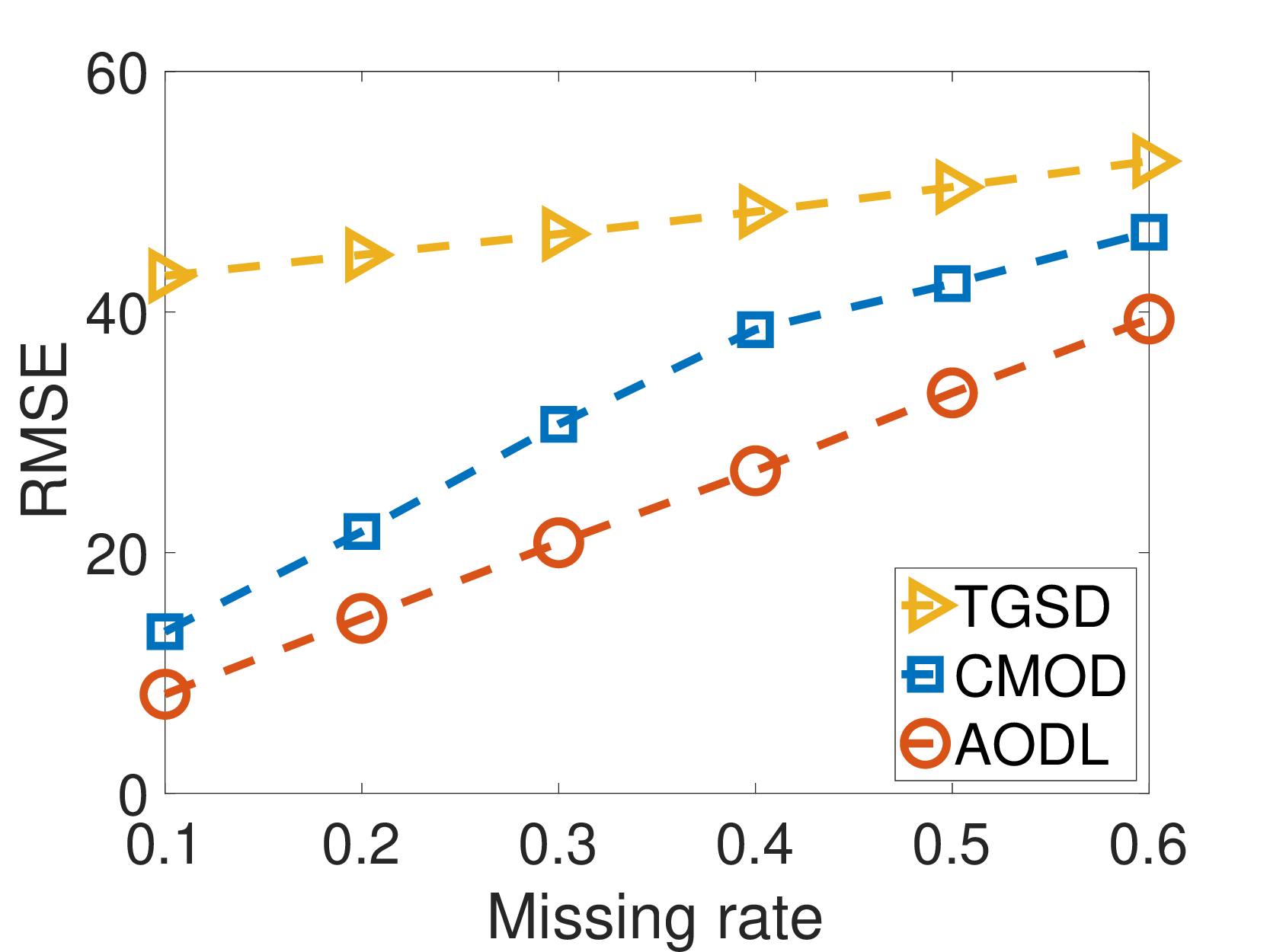}
        \label{fig:road_imp}
    }\hspace{-0.16in}
    \subfigure [Impute: Twitch]
    {
        \includegraphics[width=0.19\linewidth]{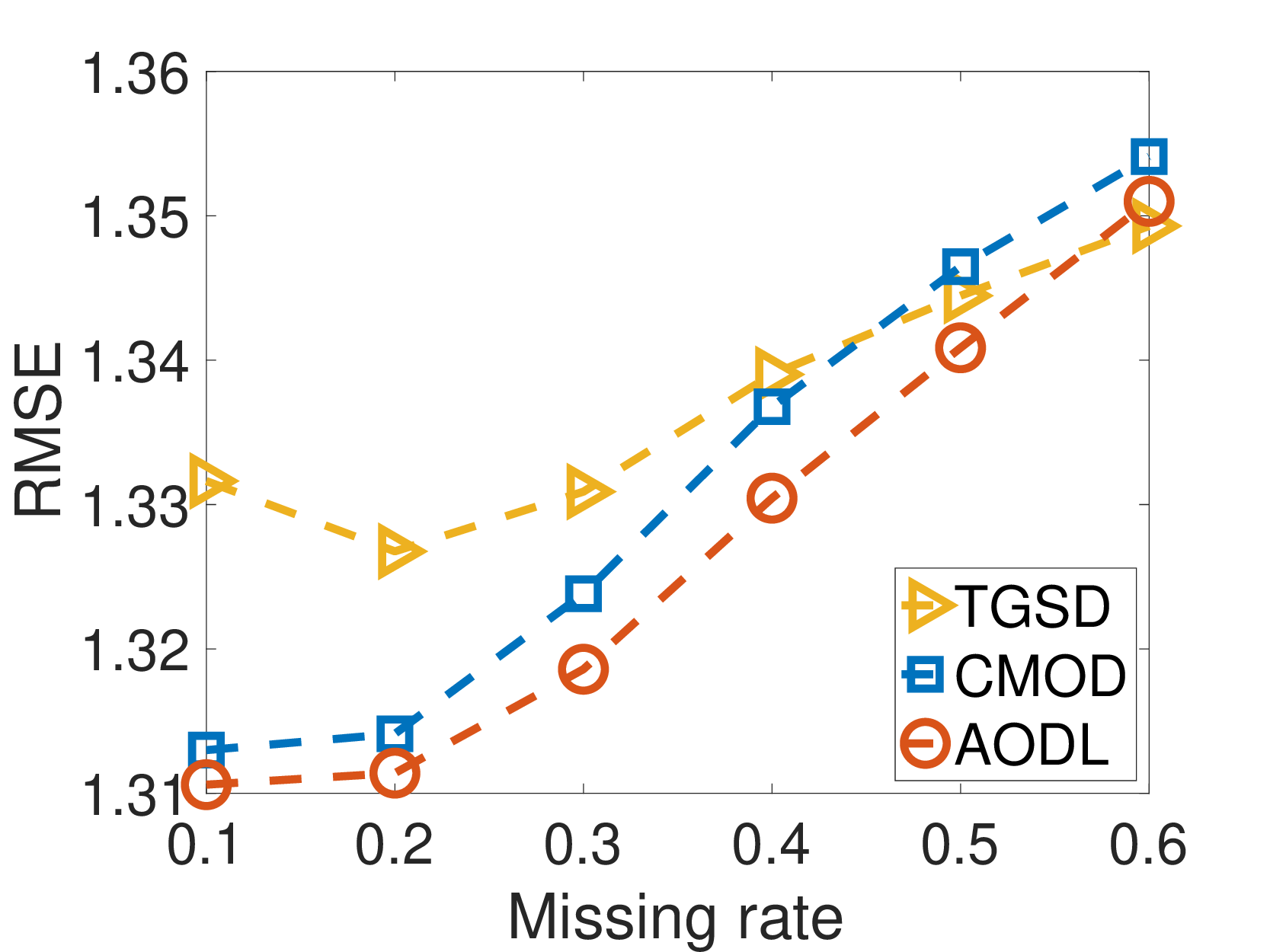}
        \label{fig:twitch_imp}
    }\hspace{-0.16in}
    \subfigure [Impute: Wiki]
    {
        \includegraphics[width=0.19\linewidth]{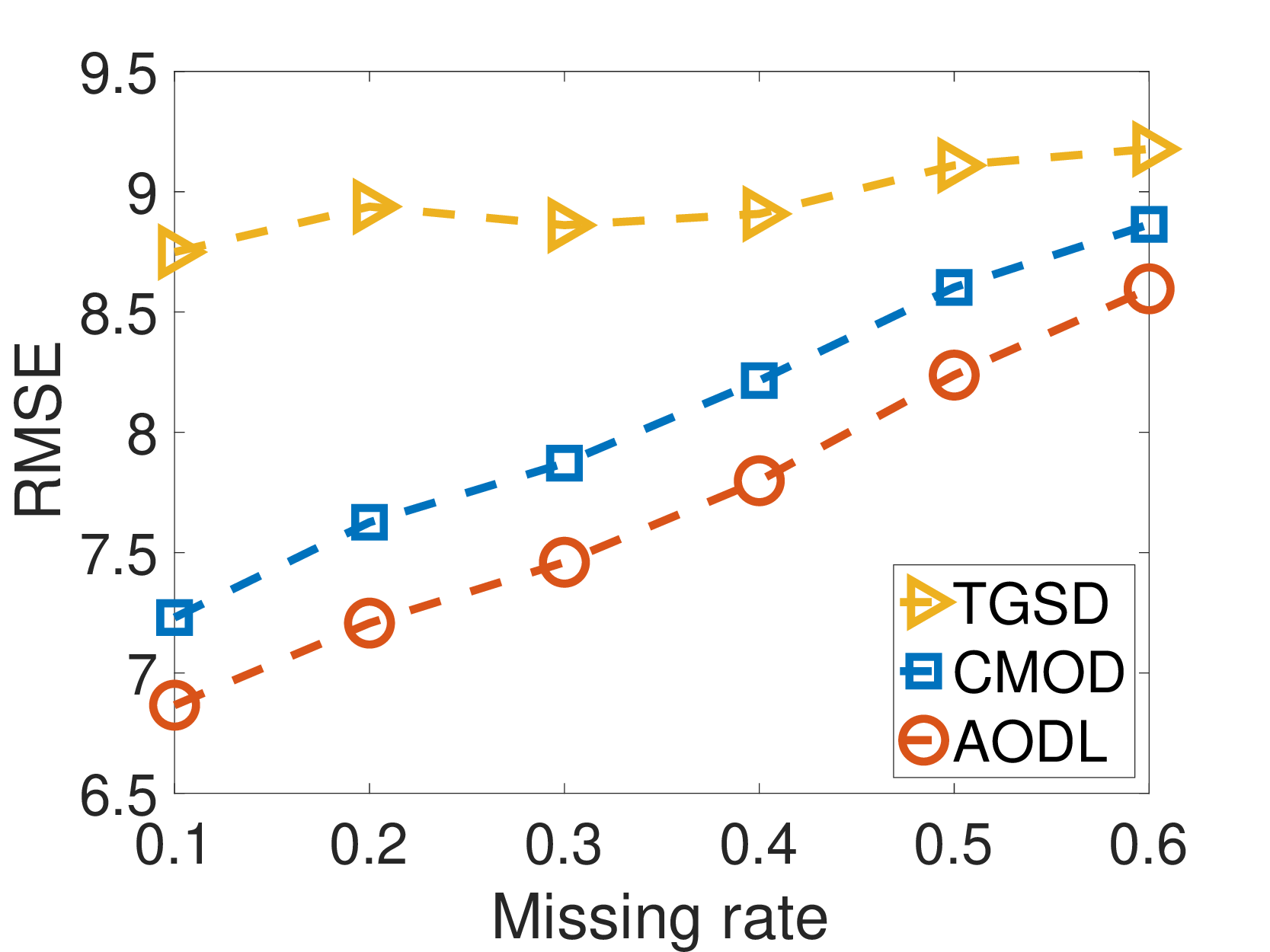}
        \label{fig:wiki_imp}
    }\hspace{-0.16in}
    \subfigure [Impute: MIT]
    {
        \includegraphics[width=0.19\linewidth]{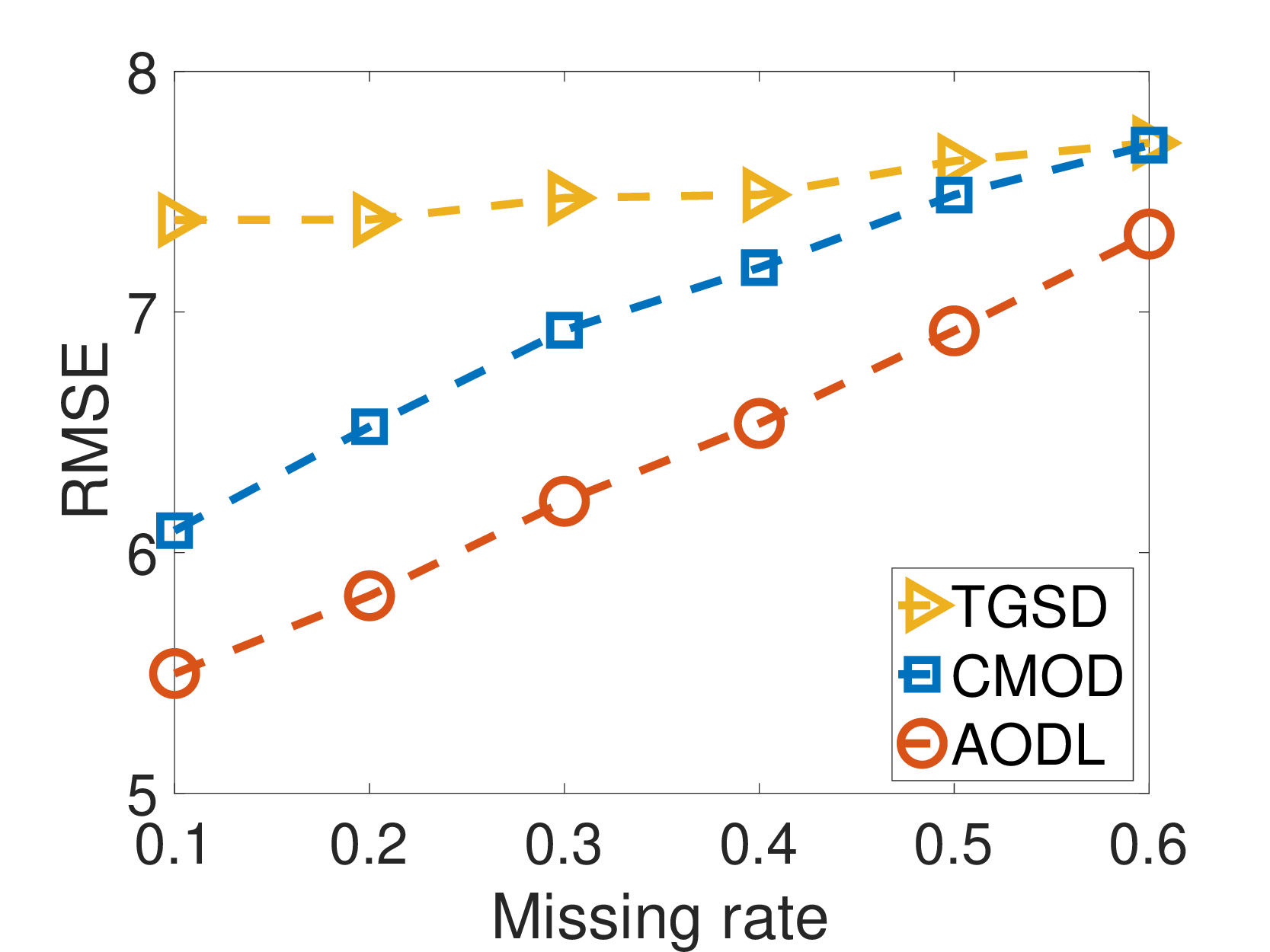}
        \label{fig:mit_imp}
    }\hspace{-0.16in}
    \subfigure [Impute: Air]
    {
        \includegraphics[width=0.19\linewidth]{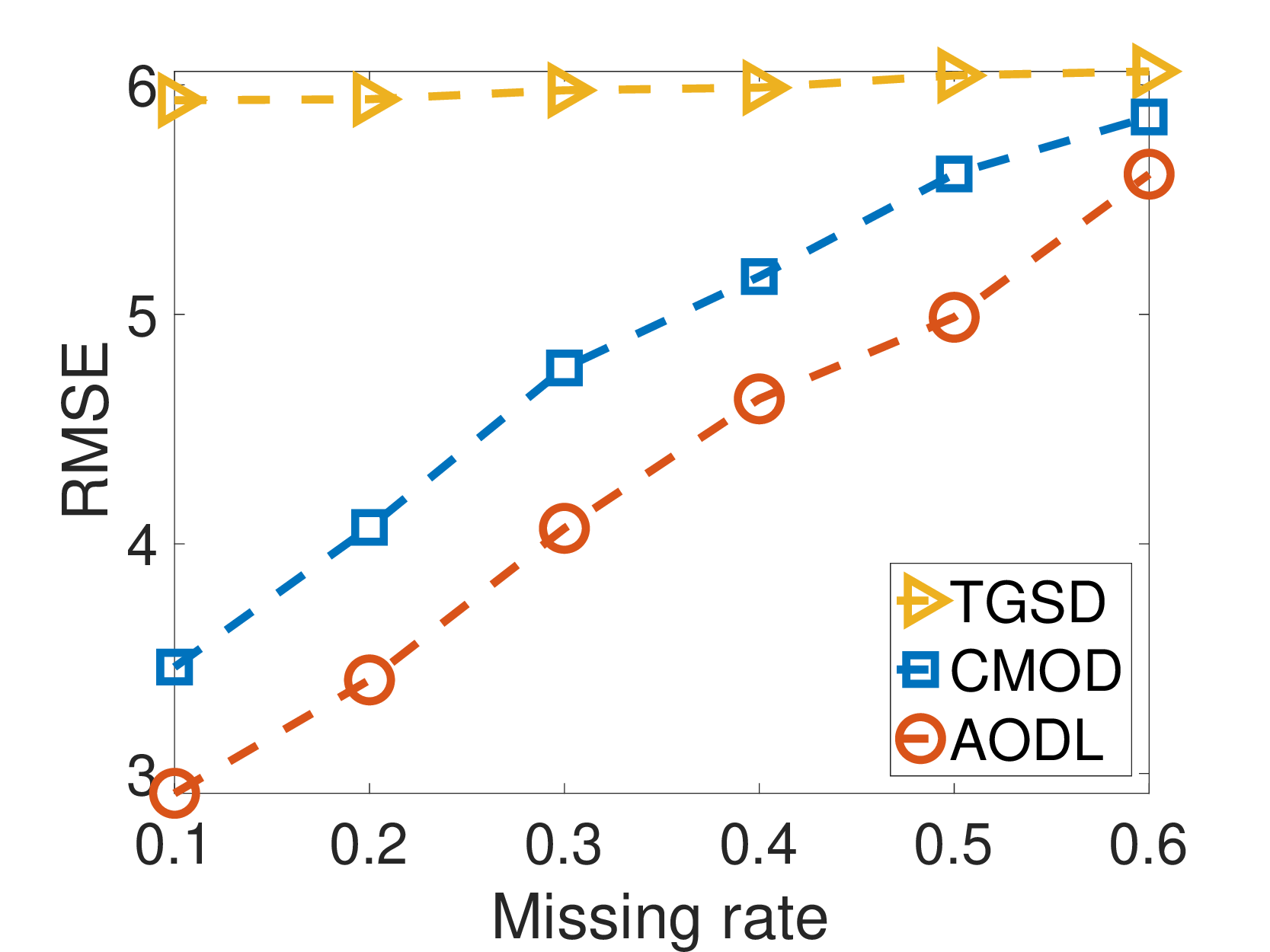}
        \label{fig:air_imp}
    }%\hspace{-0.16in}
    % \vsa\vsa
    \caption{ \footnotesize
     Comparison of competing techniques for data reconstruction \subref{fig:real_road}-\subref{fig:real_air} and missing value imputation \subref{fig:road_imp}-\subref{fig:air_imp} on all real-world datasets. 
     % \vsa\vsa\vsb
    }
    \label{fig:real_test}
\end{figure*}

\subsection{Experiments}
%%%%%%%%%%%%%%%%%%%%%%%%%%%%%%%%%%%%%%%%%%%%%%%%%%%%%%%%%%%%%%%%%%%%%%%%%%%%%%%
\noindent{\bf Quality and running time for fixed model size.} Sparse coding can be viewed as a compressive lossy reconstruction of the input data. We first compare the quality of reconstruction and running time of competing techniques on all datasets for a fixed model size. We report this comparison in Tbl.~\ref{table:datasets} where the maximal model size (Max. NNZ) is listed in the $8$-th column. %Trade-offs between model size and quality is investigated later in this section.
TGSD is the fastest among competing techniques as it only performs sparse coding and no dictionary learning. However, its quality is dominated by \ourmeth in all but the Synthetic dataset since the learned dictionaries allow for a more accurate representation of the data. TGSD has better RMSE than \ourmeth on the synthetic dataset since we equip it with the ground truth dictionaries used for generating the data, while these dictionaries are not provided to any of the dictionary learning techniques. \ourmeth achieves the smallest error at fixed model size on all real-world datasets and its running time is similar to the ADMM version of CMOD. While CMOD-OMP is slightly better than CMOD-ADMM regarding RMSE, it requires orders of magnitude more time (10s of hours on some datasets) to converge and this gap grows with the NNZ. SeDiL's running time is inconsistent since the required number of iterations to converge to a low-RMSE solution varies widely across datasets while its best RMSE is similar to that of CMOD and dominated by \ourmeth.

%%%%%%%%%%%%%%%%%%%%%%%%%%%%%%%%%%%%%%%%%%%%%%%%%%%%%%%%%%%%%%%%%%%%%%%%%%%%%%%
\noindent{\bf Reconstruction quality vs size on real-world data.}
We next evaluate the trade-off between reconstruction error (RMSE) and coding coefficient size (NNZ) for all techniques on the real-world datasets in Figs.~\ref{fig:real_road}-\ref{fig:real_air}. Recall that all competing techniques employ (dense and square) dictionaries of the same size, and hence, the dictionary size is not reported as part of the NNZs. To obtain different points in the RMSE-NNZ space we vary the sparsity regularization hyper-parameters for all competing techniques and we also consider variants of \ourmeth with different rank $k$. %dictionaries for all competing techniques %Since there is no ground truth dictionaries available for them, we decide to use analytic dictionaries as a baseline. The dataset are spatial-temporal data, as a result, the left dictionary we used is GFT of the graph, and the right dictionary is Ramanujan periodic dictionary.
\ourmeth outperforms baselines on all datasets by consistently producing more accurate models (lower RMSE) at the same level of NNZ. Note that a larger $k$ and hence larger coding coefficient matrices enables lower RMSE, but the RMSE reduction diminishes with $k$. In the Road (Fig.~\ref{fig:real_road}) and Wiki's (Fig.~\subref{fig:real_wiki}) datasets \ourmeth has the largest relative advantage. For example, to match the best achieved RMSE of CMOD (at $105k$ NNZ), \ourmeth requires an order of magnitude fewer (around 10k) coefficients, while in Wiki \ourmeth can save close to $80\%$ of the coefficients to match CMOD's quality ($7k$ vs $37k$ NNZ). %\ourmeth has also significant advantages on MIT (Fig.~\ref{fig:real_mit})
The advantages of \ourmeth stem from i) the low-rank model which aligns well with conserved temporal behaviors for subgraphs (spatial sub-regions) and ii) its learned dictionaries specifically tailored to low-rank encoding matrices. Note that TGSD also employs a low-rank encoding model and it tends to work on par or even better than CMOD in low-NNZ regimes. However, data-driven dictionaries (even for a non-low-rank CMOD model) offer an advantage for higher NNZ. 
%However, Twitch and Air data are sliced on the graph dimension, our advantage is smaller. The reason would be that it is not a typical low rank data because the graph structure is different between samples, and it is relatively harder to find a common graph structure.
% is also low rank model, which works better than CMOD at the beginning of the figure, but then is overtaken by CMOD when NNZ is getting larger. This because the low rank feature helps A-Dict win at first, while CMOD wins later because it learns the dictionaries that fit the data better. 

%%%%%%%%%%%%%%%%%%%%%%%%%%%%%%%%%%%%%%%%%%%%%%%%%%%%%%%%%%%%%%%%%%%%%%%%%%%%%%%
\noindent{\bf Missing value imputation.} We also evaluate the quality of learned dictionaries for missing value imputation. We employ \ourmeth handling missing values and develop a similar missing-value-aware version of CMOD-ADMM (details in Appendix~\ref{appendix:alg-missing}). TGSD supports missing value imputation by design. We manually remove a fraction of values in random locations from each data sample and report the RMSE between imputed and actual data for increasing fraction of missing values (ranging from $10\%$ to $60\%$). \ourmeth achieves the best RMSE among all competing techniques and across all missing value levels. CMOD is the second best technique demonstrating the benefits of dictionary learning compared to analytical dictionaries employed by TGSD. An exception to these trends is  Twitch (Fig.~\ref{fig:twitch_imp}) at high (0.5-0.6) rates of missing values, where TGSD works on par and even better than the CMOD and \ourmeth. This maybe due to Twitch data being sparser than other datasets limiting the benefits of data-driven dictionaries (CMOD and \ourmeth) due to insufficient observed data.

\begin{figure*}[t]
    \centering
    \subfigure [Error vs \# Samples]
    {
        \includegraphics[width=0.23\linewidth]{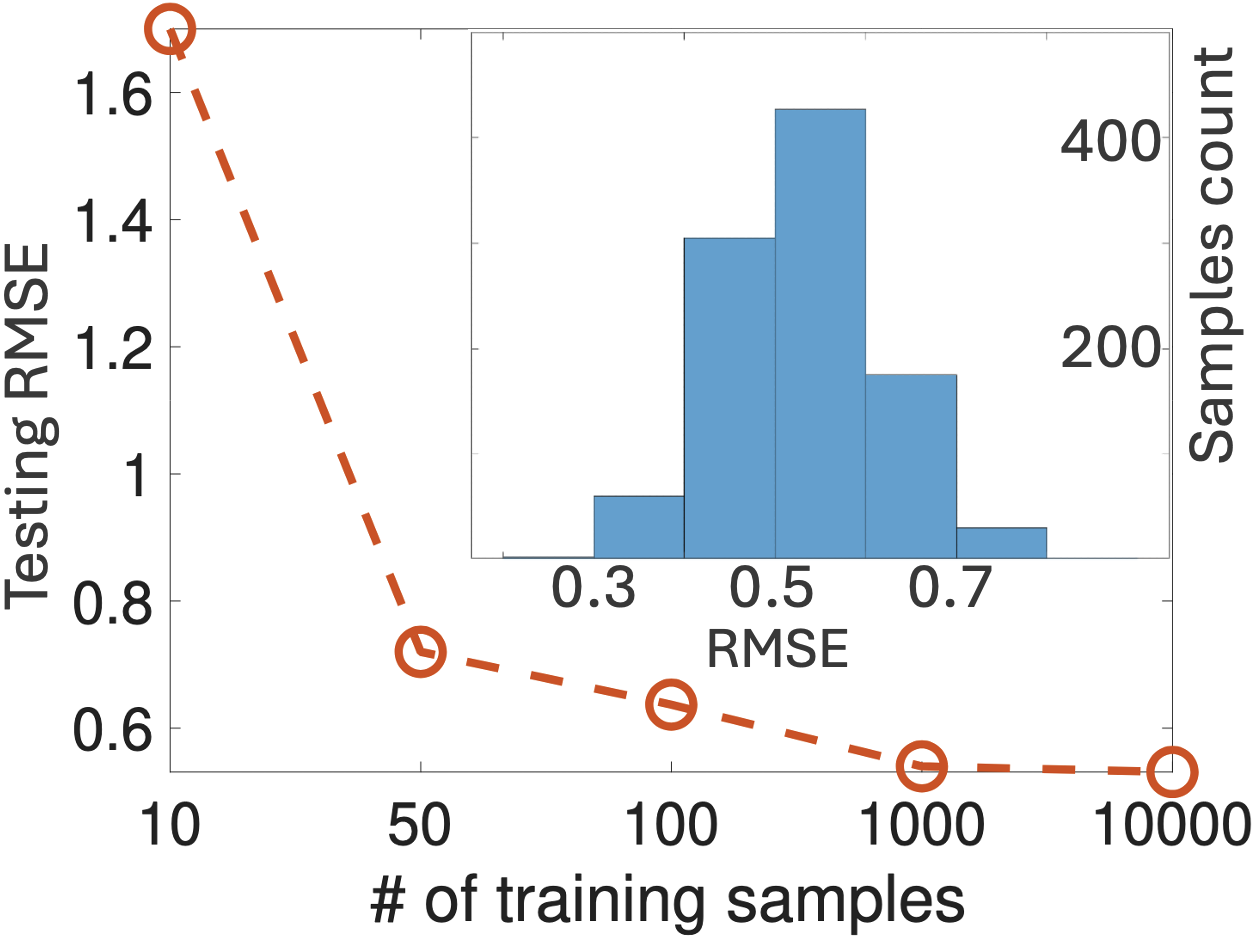}
        \label{fig:theory_error_vs_samples}
    }
    % \subfigure [Distribution of test errors ]
    % {
    %     \includegraphics[width=0.45\linewidth]{fig/res/theory1_error_dist10000.eps}
    %     \label{fig:theory_error_dist}
    % }
    \subfigure [Convergence]
    {
        \includegraphics[width=0.23\linewidth]{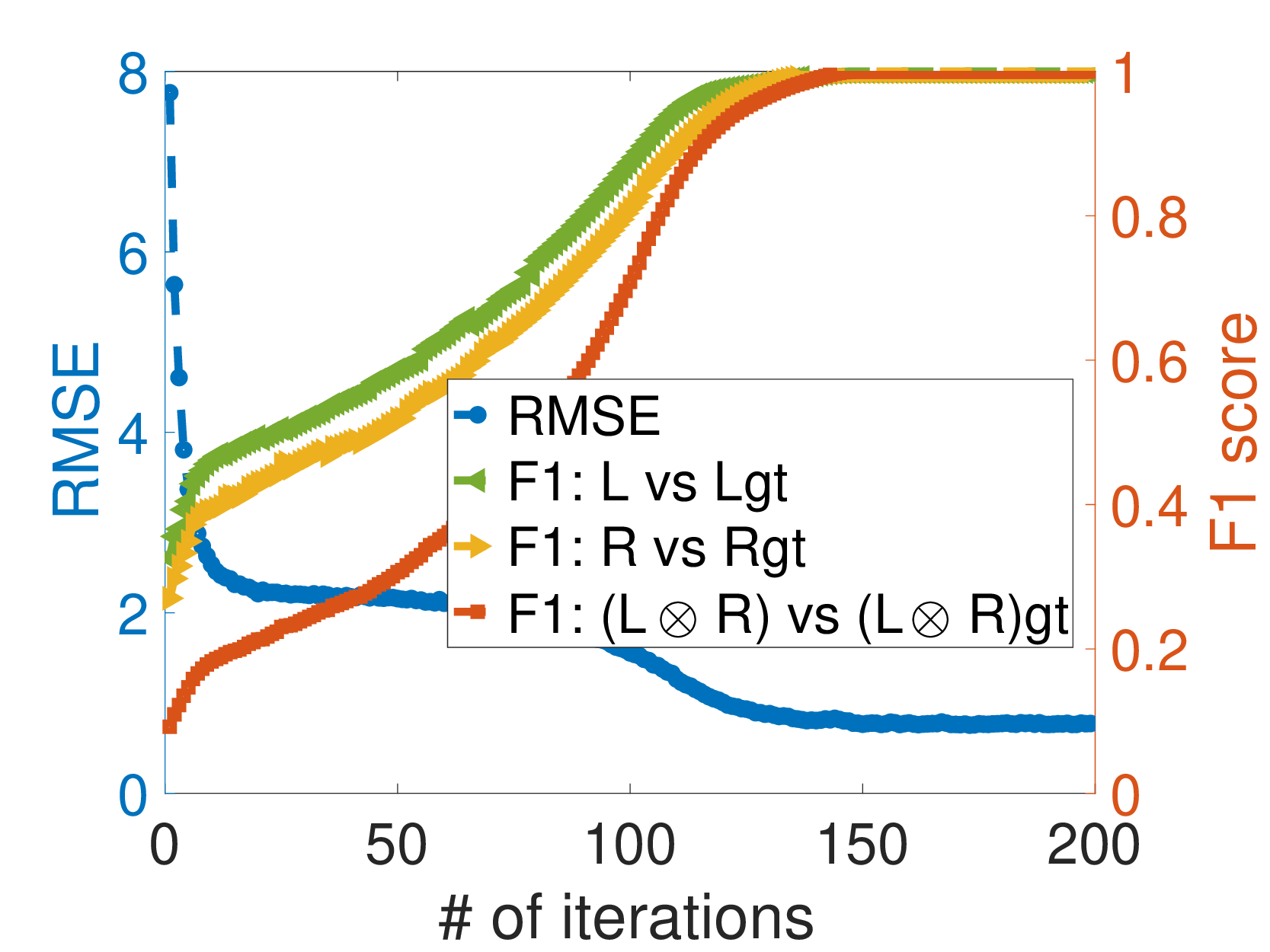}
        \label{fig:convergence_rmse_f1_vs_iteration}
    }
    % \subfigure [Convergence: RMSE vs Iteration ]
    % {
    %     \includegraphics[width=0.23\linewidth]{fig/res/convergence_rmse.eps}
    %     \label{fig:convergence_rmse_vs_iteration}
    % }
    \subfigure [Syn: $\kappa$ vs quality]
    {
        \includegraphics[width=0.23\linewidth]{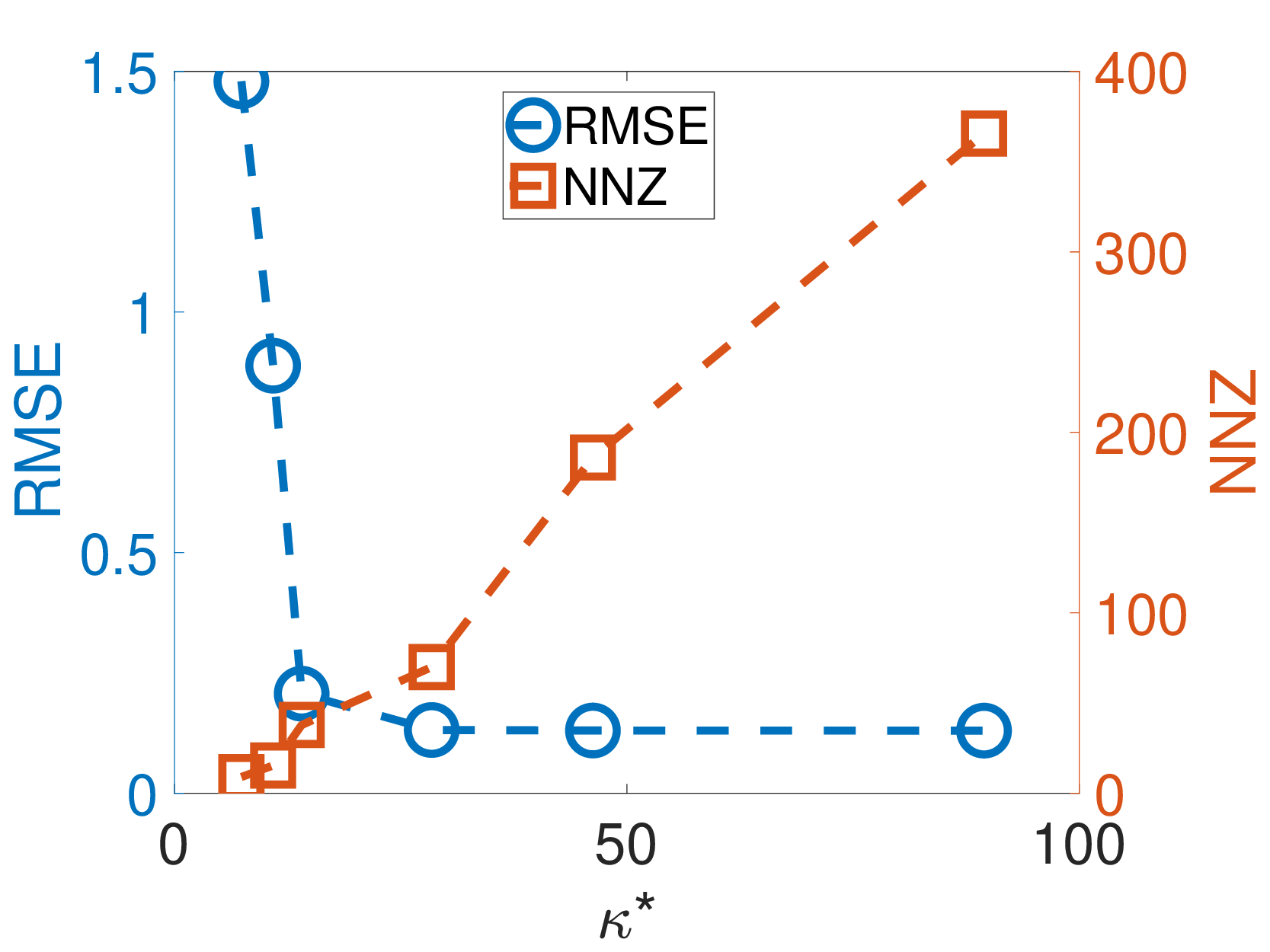}
        \label{fig:syn_kappa_test}
    }
    % \subfigure [Syn: $\kappa$ vs NNZ ]
    % {
    %     \includegraphics[width=0.23\linewidth]{fig/res/syn_kappa_nnz2.eps}
    %     \label{fig:syn_kappa_nnz}
    % }
    \subfigure [MIT: $\kappa$ vs quality]
    {
        \includegraphics[width=0.23\linewidth]{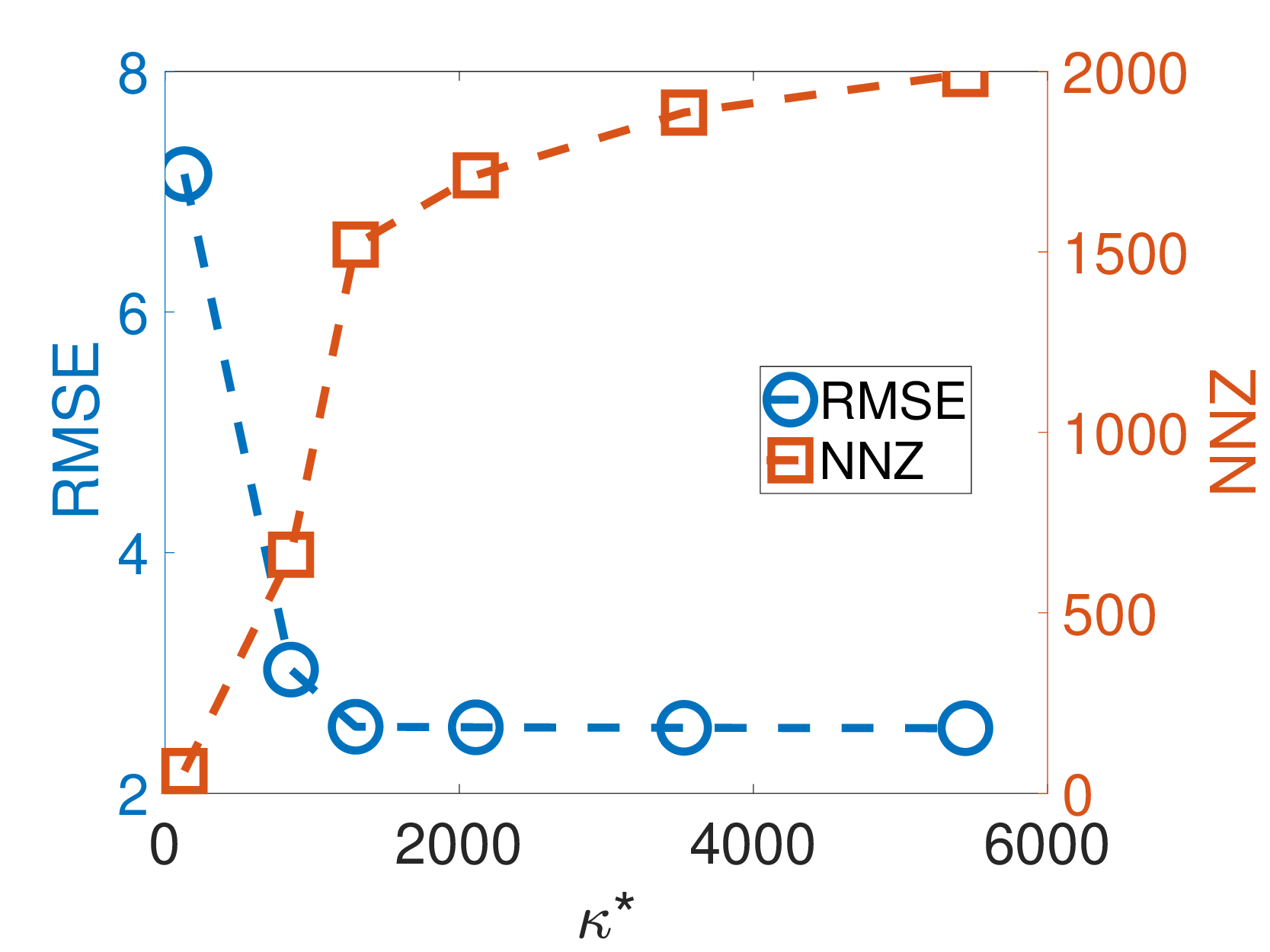}
        \label{fig:real_kappa_test}
    }
    % \subfigure [Real: $\kappa$ vs NNZ ]
    % {
    %     \includegraphics[width=0.23\linewidth]{fig/res/real_kappa_nnz.eps}
    %     \label{fig:real_kappa_nnz}
    % }
    \caption{\footnotesize
     %Theoretical bound evaluation \subref{fig:theory_error_vs_samples}; Convergence test \subref{fig:convergence_rmse_f1_vs_iteration}; Constrained problem via regularization test \subref{fig:syn_kappa_test}\subref{fig:real_kappa_test} . 
     \subref{fig:theory_error_vs_samples}: Average reconstruction error of unobserved (test) samples as a function of training samples and encoding error of the testing samples. The distribution of testing errors for $1k$ testing samples when $S = 1k$ is in the inset figure. The empirical error is {slightly smaller than the one predicted by the bound (0.64).} \subref{fig:convergence_rmse_f1_vs_iteration} Convergence as RMSE and F1 score as a function of \#iterations. F1 score is computed between (1) learned dictionaries (L, R, and their Kronecker product) and (2) GT dictionaries based on atom inner product alignment.\subref{fig:syn_kappa_test},\subref{fig:real_kappa_test}: %Validation of Thm.~\ref{thm:regularization-solves-constrained} 
     Control of sparsity parameter $\kappa_*$ of learned coding matrices on synthetic and MIT data via grid search in the space of regularization parameters $\lambda_1, \lambda_2$.
     %$\kappa_*$ is calculated using the coefficients learned via regularizers, where the 
     %Regularization parameters are estimated via grid search to achieve a given target $\kappa_*$.  
     %Our ability to control $\kappa_*$ by appropriately setting $\lambda_1, \lambda_2$ corroborates Thm.~\ref{thm:regularization-solves-constrained}. 
     Larger $\kappa_*$ relaxes the sparsity constraint, which allows for larger NNZ and lower RMSE. 
     }
    \label{fig:theory_test}
\end{figure*}

%%%%%%%%%%%%%%%%%%%%%%%%%%%%%%%%%%%%%%%%%%%%%%%%%%%%%%%%%%%%%%%%%%%%%%%%%%%%%%%
\noindent{\bf Theoretical bound evaluation (Thm.~\ref{thm:dictionary-learning-generalization-bound}).} We also study the performance of \ourmeth with different number of samples as an empirical validation of Thm.~\ref{thm:dictionary-learning-generalization-bound}.  Specifically, we treat \ourmeth as a proxy for the empirical risk minimization rule that leads to the sample complexity bound. Given a number of samples from a natural distribution for which the minimum possible risk is $0$, we demonstrate that \ourmeth achieves an error that is less than or equal to the bound. The experimental setup is as follows: (i) sample data generation: $N$, $M$, $P$, $Q$ are set to $20$, $30$, $20$, $30$, respectively. Both $L$ and $R$ are almost orthogonal with length-1 atoms. %, atoms are normalized to length $1$. 
For each sample, the coefficient matrices $Y$ and $W$ contain 20 and 30 non-zero values (rank $k=3$), and we restrict $||Y||_1=||W||_1=10$. (ii) Theory parameters:  max Frobenius norm of the samples $C=10$, $1 - e^{-2} = 0.8647$ and $\kappa = 10$. %which is the max $L_1$ norm of $Y, W$. The number of training samples are varying in the range from $[10, 50, 100, 1000, 10000]$.
The test is conducted in a training-testing manner where we (i) generate the training and testing samples using the same ground truth dictionaries and constraints on coefficients; (ii) next we learn the dictionaries from increasing number of training samples; and (iii) use the learned dictionaries to sparse code a fixed testing dataset to calculate the RMSE error. 

% \am{TODO: This text needs updating, because it incorrectly neglects to normalize the theoretical error.}
We report the average test reconstruction error as a function of the number of training samples as well as the distribution of sample errors (inset) in Fig.~\ref{fig:theory_error_vs_samples}. As expected, the test error decreases with the number of training samples and the error remains relatively stable beyond $1000$ samples. To support the theoretical bound, we need to show that the test error is less than or equal to the error predicted by the theory. {Substituting $S=1000$ and the remaining theory parameters mentioned above, we obtain an error (RMSE) bound of roughly $0.64$, which is slightly larger than the error in the figure.  Small perturbations of the confidence probability (governed by the parameter $x$) changes this number very little.} 
% \todo{A question here: with the new theory function, we got a bounded error of 0.64 in this simple settings (Using same GT dictionaires for the input data). We also have test result with random input, the test error is about 2.8. However, the theory values remains the same, the theory bound is still 0.64.}
% The empirical error is much lower than the bound since the sample distribution corresponds to an ``easy'' setting, namely all samples come from the same ground truth dictionaries and restricted coefficients while the theory provides a general bound assuming iid data samples. 
%In such a case, no model could find shared dictionaries that can produce almost perfect reconstruction. 
We simplified the data distribution for the purposes of this test in order to be able to reach a test error close to zero. The bar chart in Fig.~\ref{fig:theory_error_vs_samples} shows the test error distribution when $S=1000$ and empirical error of less than $1$ for all samples, which indicates that the learned dictionaries fit the data well. The gap between the test error achieved by our algorithm and the bound motivates further theoretical investigation discussed at the end of Appendix~\ref{sec:sample-complexity-proof}.

\noindent{\bf Convergence.} We next conduct empirical convergence analysis of \ourmeth on synthetic data. The synthetic parameters are set based on the default values outlined in Appendix~\ref{appendix:hyper}. The ground truth (GT) dictionaries are almost orthogonal. We are interested in quantifying the rate at which the learned dictionaries converge to the GT and quantify this in terms of F1 score. Specifically, we compute inner products between all pairs of learned and GT atoms in a dictionary and pick, without replacement, the top pairs in turn keeping the inner product value as a fractional success (TP) count. The F1 score is then computed based on precision and recall. We measure the F1 score as a function of the number of iterations for $L$, $R$ and the 2D atoms computed as the Kronecker products of the two dictionaries (Fig. \ref{fig:convergence_rmse_f1_vs_iteration} (green, yellow and red curves). \ourmeth recovers the GT dictionaries in about 140 iterations for this setting ($F1=1$). We also plot the corresponding RMSE as a function of the iterations in Fig.~\ref{fig:convergence_rmse_f1_vs_iteration} (blue curve). The RMSE drops significantly in the first few iterations, then becomes relatively stable and converges at around iteration $150$ when the GT dictionaries are recovered (note that the final RMSE>0 since the data contains noise). %This test shows that \ourmeth can successfully recover the GT dictionaries within a number of iterations.

\noindent{\bf Constrained optimization via regularization.} % (Thm.~\ref{thm:regularization-solves-constrained})}
Next, we show empirically that we can effectively control the sparsity (i.e., $\kappa_*$) of learned coding matrices via regularization.
We first predefine some target $\kappa_*$ values for a given dataset. When learning the dictionaries, we (i) grid search the regularizers ($\lambda_1, \lambda_2$) in \ourmeth; (ii) pick the regularizers that produce a $\kappa_* = max\{\|Y_s\|_1, \|W_s\|_1\}$ that is close to the target $\kappa_*$ values; and (iii) report the RMSE and NNZ v.s. $\kappa_*$ in Figs.~\ref{fig:syn_kappa_test},\subref{fig:real_kappa_test}. The RMSE decreases for increasing $\kappa_*$ (blue curves) while the NNZ increases (red curves). The model is less sparsity-constrained at larger $\kappa_*$ allowing denser coding matrices.  The fact that we succeed in finding $(\lambda_1, \lambda_2)$ that yield each target $\kappa_*$ (even those below the maximum $L_1$ norm of the coding matrices that generated the data) is an empirical observation that aligns with Thm.~\ref{thm:regularization-solves-constrained} (although Thm.~\ref{thm:regularization-solves-constrained} cannot be completely confirmed empirically, since we cannot efficiently solve the constrained problem). In a practical scenario with a desired budget ($\kappa$) of coefficients, one can perform bisection search on the regularizers to satisfy the desired sparsity level. 
%
%\todo{Abram, Boya, I do not see how the last sentence is justified. Is there equivalence?} These empirical observations align with Thm.~\ref{thm:regularization-solves-constrained}.

% All previous tests are conduct with GT NNZ = 30, the corresponding coefficient sparsity level is $\frac{15+15}{p*k + k \times q} = \frac{15+15}{20*3 + 3 \times 30} = 20\%$, which is a relatively dense. If we keep increase this value to NNZ = 50 (sparsity level $33 \%$), both \ourmeth and CMOD failed recovering the GT dictionaries. This is because denser coefficients make the atoms less unique to the data. From Fig \ref{fig:snr2_rmse_denser}, we can see \ourmeth cannot match exactly with GT, however, it is much closer to it than CMOD due to the low rank model.

\begin{figure*}[t]
\vsa
    \centering 
     \subfigure [\#1 time atom]
    {
        \includegraphics[width=0.145\linewidth]{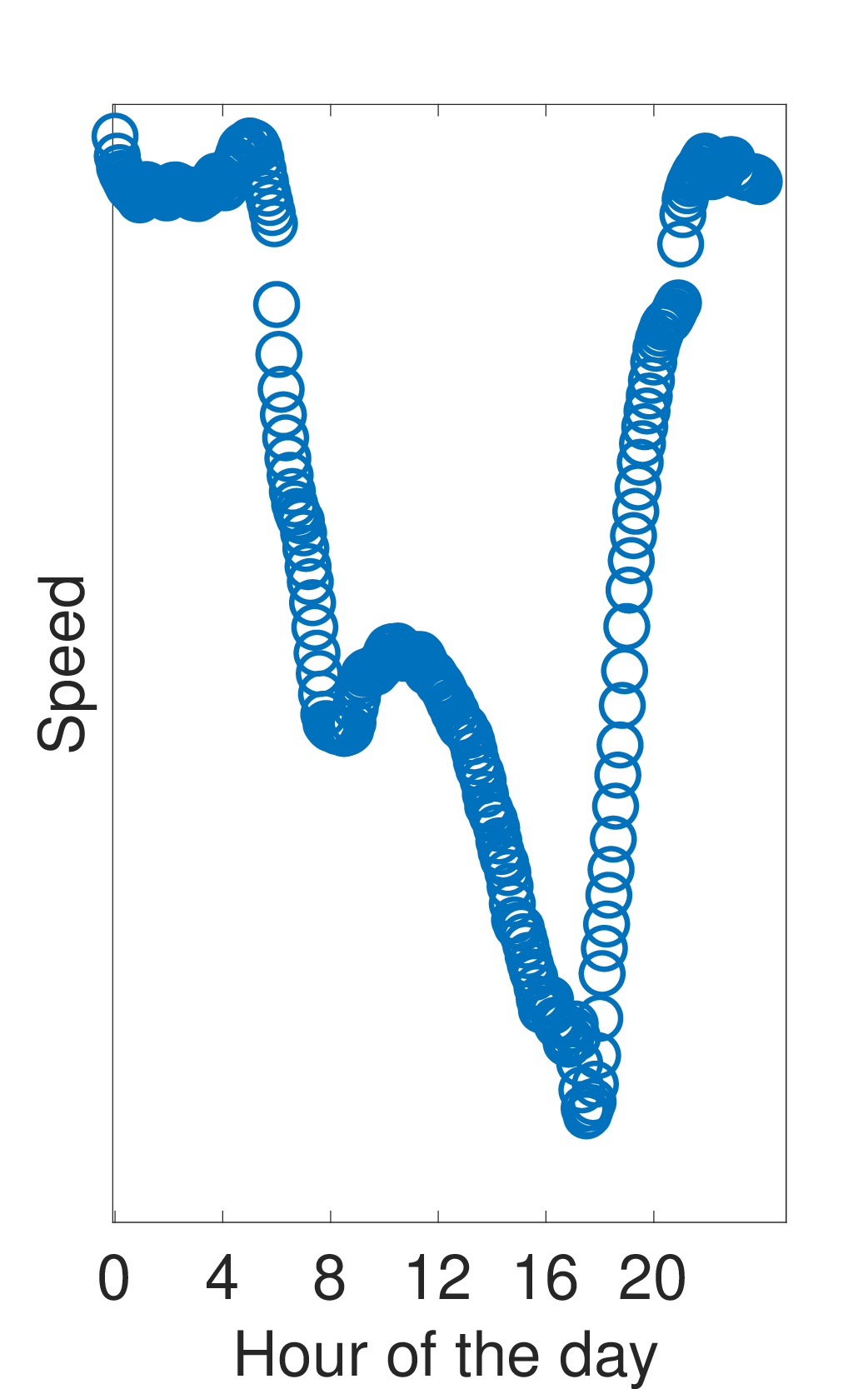}
        \label{fig:road_atom_tempora}
    }
    \subfigure [Use of \#1 time atom]
    {
        \includegraphics[width=0.31\linewidth]{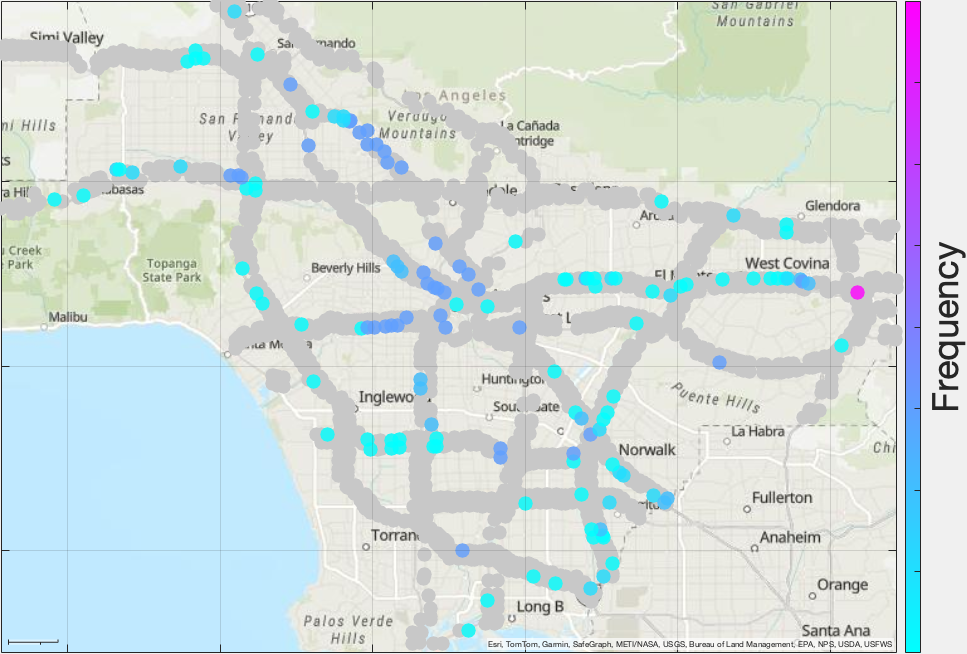}
        \label{fig:road_atom_tempora_freq}
    }
     \subfigure [\#1 spatial atom (\#1sa) ]
    {
        \includegraphics[width=0.31\linewidth]{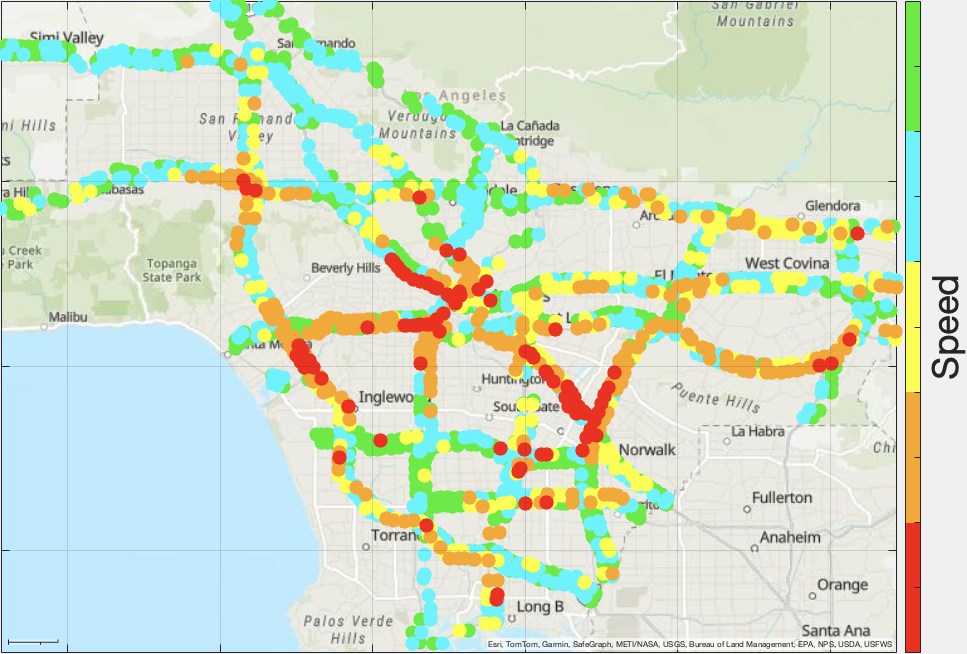}
        \label{fig:road_atom_spatial}
    }
    \subfigure [\#1sa in time]
    {
        \includegraphics[width=0.145\linewidth]{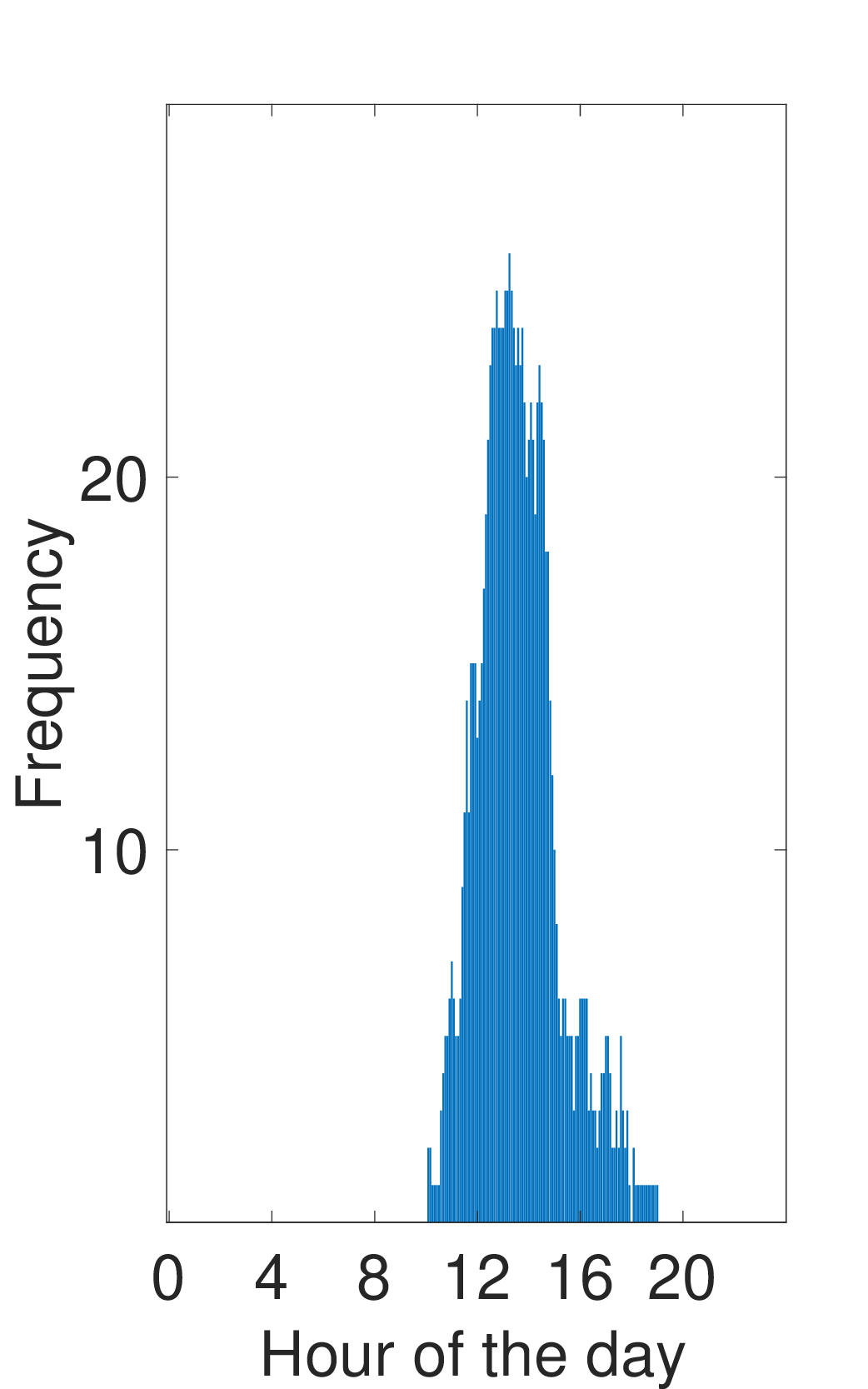}
        \label{fig:road_atom_spatial_freq}
    }\vsa
    \caption{\footnotesize 
     Most used temporal \subref{fig:road_atom_tempora} and spatial \subref{fig:road_atom_spatial} atoms learned on Road traffic dataset and their usage distribution in space \subref{fig:road_atom_tempora_freq} and time \subref{fig:road_atom_spatial_freq} respectively. \vsa
     % \textcolor{blue}{\subref{fig:road_atom_tempora} is the most commonly used temporal atom, \subref{fig:road_atom_tempora_freq} is its usage frequency in space}
    }
    \label{fig:road_case}
\end{figure*}

%%%%%%%%%%%%%%%%%%%%%%%%%%%%%%

% \begin{figure*}
%     \centering
%     \subfigure [Road]
%     {
%         \includegraphics[width=0.18\linewidth]{fig/res/road_imp.eps}
%         \label{fig:road_imp}
%     }
%     \subfigure [Twitch]
%     {
%         \includegraphics[width=0.18\linewidth]{fig/res/twitch_imp.eps}
%         \label{fig:twitch_imp}
%     }
%     \subfigure [Wiki]
%     {
%         \includegraphics[width=0.18\linewidth]{fig/res/wiki_imp.eps}
%         \label{fig:wiki_imp}
%     }
%     \subfigure [MIT]
%     {
%         \includegraphics[width=0.18\linewidth]{fig/res/mit_imp.eps}
%         \label{fig:mit_imp}
%     }
%     \subfigure [Air]
%     {
%         \includegraphics[width=0.18\linewidth]{fig/res/air-imp.eps}
%         \label{fig:air_imp}
%     }
%     \caption{ \footnotesize
%      Missing value imputation on real-world data.
%     }
%     \label{fig:real_imp}
% \end{figure*}

%\subsection{Evaluation on real-world data}

%In this test, TGSD is performed using the analytic dictionaries, the left dictionary is the graph GFT, and the right dictionary is Ramanujan periodic dictionary. It runs the fattest since it has no steps/iterations to learn any dictionaries.

%\noindent{\bf Missing value imputation}

\noindent{\bf Case study: Atoms learned from Road traffic data.} Our expectation is that the learned dictionaries by \ourmeth capture patterns that are representative of typical sample behavior. To qualitatively investigate this hypothesis, we visualize the learned atoms by \ourmeth that are most used for sparse coding. We learn the dictionaries $L$ (atoms are spatial patterns) and $R$ (atoms are temporal daily patterns) from the 30 daily samples of the \emph{Road dataset}. In addition to the atoms, \ourmeth has also estimated the encoding matrices $Y_s \in \R^{P \times k}, W_s \in \R^{k \times Q}$. For a given sample, if a row in $Y_s$ has at least one non-zero coefficient, we count the corresponding atom as being used for this sample and find the top atom used by the most number of samples. We similarly find the most frequently used right atom based on the columns in $W_s$ for each sample. We also investigate where and when the top temporal and spatial atoms are used respectively. Assuming that the top graph atom is $l$, we calculate its alignment $a(l)$ with each data timestep in $X_s \in \R^{N \times M}$ as follows: $a(l) = l^T X_s$. We retain the top 24 best aligned timesteps (spanning 2 hours) for each sample, and sum the total occurrence of these timesteps in all samples. This cumulative score is proportional to the frequency with which the top atom is used at different times of the day. We similarly quantify the locations whose daily time series best align with the top right (temporal) atom $r$. We then retain the top $20$ locations by alignment from each sample and quantify their frequency in space.

%For the Road traffic dataset, which is the speed sensor data captured in the city of Los Angeles. The whole data contains 2780 speed sensors, and 30 days speed data with a time interval of 5 minutes. In our test, we slice the data into 30 samples, each sample contains the speed data of a single day. The output from \ourmeth would be two dictionaries: the left one captures the speed changes at different sensor locations (spatial patterns), and the right one captures how the speed changes during the day (temporal patterns). 

We visualize the top atoms based on the above frequency definitions in Fig.~\ref{fig:road_case}. The top temporal atom is plotted in Fig.~\ref{fig:road_atom_tempora} where the vertical axis is a proxy for (or the magnitude of) the speed as a function of time of the day (horizontal axis). This atom captures an expected daily commute pattern. High (average) speeds (or low congestion) occur at night (between 8pm and 6am) while the lowest speed coincide with morning (7am-8am) and afternoon (5pm-6pm) rush hours. The atom also has relatively low value for mid-day hours. We also visualize the typical locations where this atom is used in Fig.~\ref{fig:road_atom_tempora_freq} in which circles designate the locations of road sensors and the circle color designates how often the top temporal atom is used at a given location (more purple colors designates locations that employ the atom more frequently). Among the top locations are intersection of highways 10 and 57 and several in downtown LA (dark blues). 

We similarly visualize the top spatial atom in Fig.~\ref{fig:road_atom_spatial} where the atom values are color-coded. Redder colors correspond to low atom values (i.e., lower speeds) and higher values are color-coded by less red colors like blue and green. Heavy traffic ares areas such as downtown LA and segments of Highway 405 are as expected redder unlike areas far from downtown. We also plot the typical times of the day when the top spatial atom is being used to encode samples in Fig.~\ref{fig:road_atom_spatial_freq}. The top spatial atom is predominantly used to encode daily traffic with a peak in the afternoon hours. 
%For each sample $X \in R^{N \times M}$, the left spatial atom is in 

\vsa
\section{Proof of Theorem~\ref{thm:dictionary-learning-generalization-bound}}
%\begin{proof}
\label{sec:sample-complexity-proof}

In this section, we prove Theorem~\ref{thm:dictionary-learning-generalization-bound}, which gives a generalization bound for two-way dictionary learning.  Omitted proofs of auxiliary lemmas and propositions are in Section~\ref{sec:proofs-omitted-upper-bound}.  

Our sample complexity analysis generalizes the one in~\cite{JMLR:v12:vainsencher11a} (see also~\cite{TOITSampleComplexity}), which is for the case of a single dictionary, with real-valued vector signals.  Our generalization proceeds by viewing our signal matrices as vectors in an appropriate space and dictionary pairs as linear operators on that space.

%%%%%%%%%%%%%%%%%%%%%%
%\subsection{Detailed sample complexity analysis}
%\label{sec:analysis-sketch}

To derive a sample complexity bound, we will derive what is known as a \emph{uniform convergence bound}.  This is a high-probability bound on 
the \emph{generalization gap} function, defined as follows:
$ % \begin{align}
    \Psi(\X)
    := \sup_{h \in \Hyp} |R(h) - \hat{R}(h, \X)|.
$ % \end{align}
The bound must hold regardless of the data-generating distribution $\Dist$.  Such a bound constitutes a \emph{generalization bound} for empirical risk minimization, which yields a sample complexity bound.

We first collect standard notions required for our proof.

For a metric space $M$, we denote its $\gamma$-covering number by $\CovNum(M, \gamma)$.
%\begin{definition}[Covering numbers of a metric space]
%    \label{def:covering-numbers}
%    Consider a metric space $M$ with metric $d$ and a real number $\gamma > 0$.  We say that a collection $S \subseteq M$ is a $\gamma$-covering of $M$ if, for all $x \in M$, there exists $z \in S$ such that $d(x, z) \leq \gamma$.
%
%    The $\gamma$-covering number of $M$, denoted by $\CovNum(M, \gamma)$, is given by the cardinality of the minimum-size
%    $\gamma$-covering of $M$.
%\end{definition}
%
For the purpose of deriving a generalization bound, we must bound the covering numbers of the following \emph{loss class}, which is
a class of functions $\F := \{f_h:\Omega \to \R ~|~ h \in \Hyp\}$ induced by composing hypotheses in $\Hyp$ with the loss function.
\begin{definition}[Loss class associated with a hypothesis class]
    \label{def:loss-class}
    For a fixed loss function $\ell:\Hyp\times \Omega \to \R$ and a hypothesis class $\Hyp$, the 
    \emph{loss class} $\F$ induced by $\Hyp$ is given by
    $ %\begin{align}
        \F := \{ f_h(x) := \ell(h, x) ~|~ h \in \Hyp\}.
    $ %\end{align}
\end{definition}
The relevant metric on $\F$ is the one induced by the $L_{\infty}$ norm: for $f \in \F$,
$ %\begin{align}
    \|f\|_{\infty}
    := \sup_{x \in \Omega} |f(x)|.
$ %\end{align}
We note that every element of the loss class satisfies
$% \begin{align}
    f_h(x) \in [0, B],
$ %\end{align}
where $B := C^2$, the constraint on $\|x\|_F^2$ imposed by Theorem~\ref{thm:dictionary-learning-generalization-bound}.
This is because we can upper bound the minimization in the loss function by
setting $Y = 0, W=0$ in their respective spaces.

We next give a standard generalization bound~\cite{JMLR:v12:vainsencher11a} based on the $L_{\infty}$ covering numbers of a function class $\F$.  
%This can be found in~\cite{JMLR:v12:vainsencher11a}.
\begin{lemma}[Generalization bound based on covering numbers of the loss class~\cite{JMLR:v12:vainsencher11a}]
    \label{lemma:covering-to-generalization}
    Let $\F$ denote a class of functions satisfying, for all $f$, $f(x) \in [0, B]$.  Then for all $x > 0$,
    we have that with probability at least $1- e^{-x}$ over $m$ samples $X_1, ..., X_m$ sampled iid from an arbitrary
    distribution $\Dist$, we have, for all $f \in \F$,
    \begin{align}
        &|\E_{X\sim \Dist}[f(X)] - \frac{1}{m}\sum_{j=1}^m f(X_j)| \\
        &\leq 2 \epsilon + B \cdot \left(  \sqrt{ \frac{\log \CovNum(\F, L_{\infty}, \epsilon)}{2m}} + \sqrt{\frac{x}{2m}}    \right).
    \end{align}
    %\am{Idea is to discretize based on a covering, then use Hoeffding and a union bound.}
    %\am{TODO: Note that a good choice of $\epsilon$ is $\frac{1}{\sqrt{m}}$.}
    %\am{TODO: Note that if the covering number is upper bounded by $\epsilon$ to some fixed power, this is good enough.}
\end{lemma}

The $L_{\infty}$ covering number of $\F$ can be upper bounded using the covering numbers of $\Hyp$ along with a bound on the Lipschitz
constant of $f_h$ as a function of $h$.  In order to make this precise, we must specify a norm on $\Hyp$.  Recall the normalization 
conditions on the outer products of columns of $L$ with rows of $R^T$.  These translate to an operator norm condition on the mappings induced by pairs $(L, R)$.  We have the following
standard definition.
\begin{definition}[Operator norms of a linear mapping]
    \label{def:operator-norms}
    Let $V, W$ be normed spaces with norms $\|\cdot \|_V, \|\cdot\|_W$, and let $F:V\to W$ be a linear operator.
    The operator norm $\|F\|_{op,V,W}$ is defined to be
    $ %\begin{align}
        \sup_{x \in V} \frac{ \|Fx\|_W }{ \|x\|_W }.
    $ %\end{align}
\end{definition}
It is well known that the $L_{1}\to L_{2}$ operator norm of a matrix is equal to the maximum $L_2$ norm of any of its columns. 
We thus will use the following norm for $\Hyp$: for any $(L, R) \in \Hyp$,
\begin{align}
    \label{expr:hyp-norm}
    \|(L, R)\|_{op,1,2}
    := %\frac{\|L\|_{op,1,2} + \|R\|_{op,1,2}}{2}. \am{TODO: Modify this to agree with stuff below.}
    \max_{Z} \frac{\|LZR^T\|_F}{\|Z\|_{L_1}}.
\end{align}
The advantage of this norm is that it turns $\Hyp$ into a unit ball in a finite-dimensional Banach space, for which covering number bounds are known.  Specifically, we have the following lemma. %\am{Cite Cucker and Smale, 2002.}
\begin{lemma}[Covering numbers of Banach space balls (Proposition 5 of~\cite{CuckerSmale2002})]
    \label{lemma:basic-covering-number}
    Let $\Banach$ denote a Banach space with norm $\|\cdot\|$, with dimension $\dim(\Banach)$.
    Then the $\epsilon$-covering number of a ball $B(0, R)$ with radius $R$ is upper bounded as follows:
    \begin{align}
        \log \CovNum(B(0, R), \epsilon)
        \leq \dim(\Banach) \log\left( \frac{4R}{\epsilon}\right).
    \end{align}

    Since the dimension of $\Hyp$ is given by $NP + MQ$, this results in the bound
    $ %\begin{align}
        \log \CovNum(\Hyp, \epsilon)
        \leq (NP + MQ) \log\left( \frac{4}{\epsilon}\right).
    $ %\end{align}
\end{lemma}

We next turn to translating the covering number bounds for $\Hyp$ to bounds for $\F$.  The following well known lemma is the vehicle for this translation.
\begin{lemma}[Lipschitz bound for covering numbers]
    \label{lemma:lipschitz-covering}
    Let $M_1, M_2$ be two metric spaces.  Suppose that $f:M_1 \to M_2$
    is a $\lambda$-Lipschitz function.  Then 
    $ %\begin{align}
        \CovNum(f(M_1), \epsilon)
        \leq \CovNum(M_1, \epsilon/\lambda).
    $ %\end{align}
\end{lemma}

In the next lemma, we bound the Lipschitz constant for the map $G:\Hyp \to \F$ defined by
$G(h) := f_h$.  We recall that the norm on $\F$ is the $L_{\infty}$ norm.
\begin{lemma}[Lipschitz constant bound for the hypothesis to loss function mapping]
    \label{lemma:lipschitz-bound}
    The mapping $G$ is such that, for all $h, h' \in \Hyp$,
    $ %\begin{align}
        \|G(h) - G(h')\|_{\infty}
        \leq 2\cdot \kappa \|h - h'\|_{op,1,2}.
    $ %\end{align}
\end{lemma}
%\begin{proof}
%    We expand out the definition of $G$ and the $L_{\infty}$ norm:
%    \begin{align}
%        &\|G(L, R) - G(L', R')\|_{\infty} \\
%        &= \max_{X} | \min_{Y,W} \|X - LYWR^{T}\|_F^2 - \min_{Y',W'} \|X - L'Y'W'R'^T\|_F^2 |
%        \label{expr:starting-lipschitz}
%    \end{align}
%    Now, note that $\min_{Y,W} \|X - L'Y'W'R'^T\|_F^2 \leq \|X - L'YWR'^T\|_{F}^2$, where $(Y, W)$ minimizes the
%    $(L, R)$ term in the above expression.  Thus, if we can upper bound
%    \begin{align}
%        |\|X - LYWR^{T}\|_F^2 - \|X - L'YWR'^T\|_F^2|
%    \end{align}
%    for arbitrary $Y, W$, this gives us an upper bound on (\ref{expr:starting-lipschitz}).  We next have
%    \begin{align}
%        &|\|X - LYWR^{T}\|_F^2 - \|X - L'YWR'^T\|_F^2| \\
%        &\leq 2 |\|X - LYWR^{T}\|_F - \|X - L'YWR'^T\|_F| \\
%        &\leq 2\| LYWR^T - L'YWR'^T\|_{F} \\
%        &\leq 2 \|(L, R) - (L', R')\|_{op,1,2} \cdot  \|YW\|_{1} \label{expr:using-submultiplicativity} \\
%        &\leq 2 \cdot \|(L, R) - (L', R')\|_{op,1,2}\cdot  \|Y\|_{1} \|W\|_1 \\
%        &\leq 2 \cdot \kappa^2  \|(L, R) - (L', R')\|_{op,1,2}.
%    \end{align}
%    This implies the claimed Lipschitz bound.
%    %\am{The submultiplicativity that we use here is loose!}
%\end{proof}

Applying Lemma~\ref{lemma:basic-covering-number}, Lemma~\ref{lemma:lipschitz-covering}, and Lemma~\ref{lemma:lipschitz-bound} we get the following bound on the covering numbers of $\F$.
\begin{proposition}[Upper bound on the covering numbers of $\F$]
    \label{prop:loss-class-covering-numbers}
    We have the following bound on the covering numbers of $\F$:
    \begin{align}
        \log \CovNum(\F, \gamma)
        \leq  \log \CovNum(\Hyp, \gamma/(2\kappa^2))
        = (NP + MQ) \log\left( \frac{8\kappa^2}{\gamma}\right).
    \end{align}
\end{proposition}

%%%%%%%%%%%%%%%%%%%%%%%
Applying the bound in Proposition~\ref{prop:loss-class-covering-numbers} to Lemma~\ref{lemma:covering-to-generalization},
we finally get the generalization bound given in Theorem~\ref{thm:dictionary-learning-generalization-bound}.
%\end{proof}
\section{Proof of Theorem~\ref{thm:generalization-bound-entropy}}
\label{sec:renyi-entropy-bound-proof}

The proof of Theorem~\ref{thm:generalization-bound-entropy} is the same as that of Theorem~\ref{thm:dictionary-learning-generalization-bound} except for two parts: upper bounding $\|YW\|_1$ after (\ref{expr:using-submultiplicativity}) and computing a probability-$1$ upper bound $B$ on the loss function.

In particular, steps after (\ref{expr:using-submultiplicativity}) require modification because
we require a tighter bound on the Lipschitz constant of the loss function, and we derive this from a tighter bound than submultiplicativity of the $L_1$ norm for products of matrices affords us.

To that end, we have
\begin{align}
    \|YW\|_1
    &= \sum_{i=1}^P\sum_{j=1}^Q \left| \sum_{\ell=1}^k Y_{i,\ell}W_{\ell,j} \right| \\
    % Triangle inequality
    &\leq \sum_{i=1}^P\sum_{j=1}^Q \sum_{\ell=1}^k \left|  Y_{i,\ell}W_{\ell,j} \right| \\
    &= \sum_{\ell=1}^k \sum_{i=1}^P |Y_{i,\ell}| \sum_{j=1}^Q |W_{\ell,j}| \\
    &= \innerproduct{c(Y)}{r(W)} \\
    % Cauchy-Schwarz
    &\leq \|c(Y)\|_2 \|r(W)\|_2,
\end{align}
where the first inequality is the triangle inequality, and the second equality is by interchanging summations.  The second inequality is by Cauchy-Schwarz.

Next, by multiplying and dividing by $\|c(Y)\|_1 \|r(Y)\|_1 = \|Y\|_1 \|W\|_1$, we get
\begin{align}
    \|c(Y)\|_2 \|r(W)\|_2
    &= \|Y\|_1 \|W\|_1  \|\hat{c}(Y)\|_2 \|\hat{r}(W)\|_2 \\
    &\leq \kappa^2 \|\hat{c}(Y)\|_2 \|\hat{r}(W)\|_2 \\ 
    &= \kappa^2 \exp(\log( \|\hat{c}(Y)\|_2 \|\hat{r}(W)\|_2)) \\
    &= \kappa^2 \exp( -\frac{ H_2(\hat{c}(Y)) + H_2(\hat{r}(W))}{2}).
\end{align}
Finally, applying the constraint $\rho \leq H_2(\hat{c}(Y)) + H_2(\hat{r}(W))$
completes the modification of the upper bound on the Lipschitz constant of the loss function.

Turning to calculation of the value $B$, we cannot upper bound the loss function
by plugging in $Y = 0, W=0$, because the stochastic vectors $\hat{c}(Y)$ and $\hat{r}(W)$ are not well-defined.  We instead plug in 
$Y = \alpha \OneMatrix_{P\times k}, W = \alpha \OneMatrix_{k \times Q}$, where $\OneMatrix_{n\times m}$ denotes the matrix of all $1$s in $\R^{n\times m}$ and
$0 < \alpha < \frac{\kappa}{\max\{P, Q\}k}$ so that the sparsity constraint is satisfied.  Note that the R\'enyi entropies of $\hat{c}(Y)$ and $\hat{r}(W)$ are maximized, so that the R\'enyi entropy constraint is also satisfied.
Then $YW = \alpha^2 k^2 \OneMatrix_{P\times Q}$.
This yields an upper bound on the loss function of
\begin{align}
    \|X - LYWR^{T}\|_{F}^2
    &\leq (\|X\|_F + \|LYWR^T\|_F)^2 \\
    &\leq (C + \alpha^2 k^2 PQ)^2 \\
    &= C^2 + \alpha^2\cdot(2C k^2 PQ + \alpha^2 k^4 (PQ)^2).
\end{align}
The first inequality is the triangle inequality.  The second is using the upper bound on the Frobenius norm of $X$ and the operator norm of the dictionary.  Then,
choosing $\alpha$ sufficiently small implies that we can set
$B = C^2 + \hat{\alpha}$ for any $\hat{\alpha} > 0$, so that we can set
$B = C^2$.

This completes the proof of Theorem~\ref{thm:generalization-bound-entropy}.

\section{Proof of Theorem~\ref{thm:sample-complexity-lower-bound}} 
\label{sec:sample-complexity-lower-bound-proof}
%complexity lower bound}

%Here we prove the following sample complexity lower bound, the content of which is that we cannot improve the sample complexity upper bound using
%the rank constraint.
    Here we prove the sample complexity lower bound, Theorem~\ref{thm:sample-complexity-lower-bound}.  Omitted proofs of lemmas and propositions are in Section~\ref{sec:proofs-omitted-lower-bound}.

%\begin{proof}[Proof of Theorem~\ref{thm:sample-complexity-lower-bound}]
    In order to derive a sample complexity lower bound, we must exhibit
    a ``hard'' data-generating distribution.  We will choose one such
    that, given $s$ samples, a successful dictionary learner can be used
    to implement a successful \emph{test} for a multiple hypothesis testing problem in which hypotheses are a collection of dictionary atoms, and observations are their noise-perturbed versions.  We prove an error probability lower bound for this testing problem using Fano's inequality.

    We define the following set of dictionaries.
    \begin{align}
        \Dicts  := \{ &(L, R) ~|~ \forall i,j, \|L_i\|_F = \|R_j\|_F = 1, \\
        &\mu(L\tensor R) \leq \mu = \Theta(1) \}. 
    \end{align}
    Here, $\mu(M)$ is the \emph{pairwise coherence} of a matrix $M$; that is, it is the maximum absolute value of the dot product between any two columns of $M$.

    %%% Explain ddict.
    We next define the sense in which we approximate dictionaries.  We consider two dictionaries to be equivalent if one can be obtained from the other by permutation of atoms.  This motivates the following dictionary equivalence class metric.
    \begin{definition}[Metric on the space of dictionary equivalence classes]
        \label{def:metric-on-dictionaries}
        Let $D_1, D_2$ be the Kronecker product forms of two dictionaries in $\Dicts$.
        We denote by 
        $\ddict(D_1, D_2)$ the following distance:
        $ %\begin{align}
            \ddict(D_1, D_2)
            := \min_{\Pi} \|D_1 - D_2 \cdot \Pi\|_F,
        $ %\end{align}
        where the minimization is over all possible column permutation matrices.
    \end{definition}

    To define the hard data-generating distribution, we construct a packing $\Packing$
    of the space of dictionaries $\Dicts$.
    We guarantee the existence of a large packing in the next lemma.
    \begin{lemma}[Existence of a large packing of $\Dicts$]
        \label{lemma:packing-existence}
        There exists a $\ddict$-packing $\Packing$ with distance $\sqrt{PQ}\cdot \gamma$ of $\Dicts$ with cardinality
        \begin{align}
            |\Packing| 
            \geq \left(\frac{C}{\gamma} \right)^{\Omega(NP + MQ)}/(PQ)!.
        \end{align}
    \end{lemma}
    \paragraph{The hard data-generating distribution}
    Throughout, we fix the rank constraint $k = 1$ and the sparsity
    constraint $\kappa = 1$.

    Consider the following data-generating distribution for a sample
    matrix $X \in \R^{N\times M}$: a ground-truth dictionary $(L_*, R_*)$
    is fixed from $\Packing$.
    We generate a coefficient matrix $Z_X \in \R^{P\times Q}$
    by selecting an element $(i, j) \in [P]\times [Q]$ uniformly at random
    and placing a $1$ there.  The remaining entries are set to $0$.
    We then denote by $\noiseterm \in \R^{N\times M}$ a matrix
    whose entries are iid $\Normal(0, \sigma^2)$ random variables, for some 
    $\sigma > 0$.  We will choose $\sigma = \frac{1}{\log(PQ)}$ with foresight.
    %\am{TODO: It may be natural to make this tend to 0, because the Frobenius norm condition on the dictionary atoms implies that individual entries are small.  But it can't be too small, because it is in the denominator of the sample complexity lower bound.  Provided that it is a small enough constant, it's okay.}
    Then $X = LZ_XR^{T} + \noiseterm$.  We also denote by $X^* := X - \noiseterm$ the noiseless version of $X$.
    In other words, $X$ is a random dictionary atom perturbed entrywise by
    mean-$0$ Gaussian noise, and $X^*$ is its noiseless version.

    %%%%%%%%%%%%%%%%%
    \paragraph{A successful dictionary learner can be used as a successful multiple hypothesis tester from noisy samples}

    We first show that if we have $s$ samples $X := (X_1, ..., X_s)$ from 
    the hard distribution, for some unknown $D_*$, then an $(\epsilon, \delta)$-PAC dictionary learner $\Alg$ can be used to 
    estimate $D_*$ to arbitrary precision in the sense of Definition~\ref{def:metric-on-dictionaries}.  This is the content of Proposition~\ref{prop:learning-to-estimation} below.  We show this by
    showing that the \emph{excess risk} $\ExcessRisk(h) := R(h) - \inf_{h_* \in \Hyp} R(h_*)$ of a hypothesis $h$ is lower bounded by
    the squared distance between the hypothesis and the ground truth dictionary.

    %%%%%%%%%%%%%%%%%%%%%%%%%%%%%

    \begin{proposition}[A PAC dictionary learner estimates atoms from noisy samples]
        \label{prop:learning-to-estimation}
        Suppose that $\Alg$ is an $(\epsilon, \delta)$-PAC learning rule 
        with $s$ samples $X$, with rank and sparsity constraints $k, \kappa$, respectively.  Then with probability at least $1-\delta$ with respect to the hard distribution with ground-truth dictionary $D_*$, $\Alg(X) \in \Packing$ satisfies
        $ %\begin{align}
            \ddict^2(\Alg(X), D_*) \leq \Theta(PQ \epsilon).
        $ %\end{align}
    \end{proposition}

    \paragraph{A sample complexity lower bound for the multiple hypothesis testing problem}

    Our next goal is to prove an error probability lower bound for the multiple hypothesis testing problem associated with $\Packing$.
    \begin{proposition}[Testing error probability lower bound]
        \label{prop:estimation-lower-bound}
        Let $F$ be a tester for dictionaries in $\Packing$ from a
        dataset $X := \{X_s\}_{s=1}^S$ sampled from $\Dist$ with ground truth dictionary $D_*$.  We have that with probability
        at least
        \begin{align}
            1 - \frac{S}{\Omega(\sigma^2 (NP + MQ) \cdot \log(C/\gamma))},
        \end{align}
        for some positive constant $C$, $F(X) \neq D_*$.
        In particular, to drive the probability of error to $0$,
        we require $S = \Omega(\sigma^2\cdot (NP + MQ))$.
    \end{proposition}

    \paragraph{Putting everything together for the learning lower bound}
    We can now combine Propositions~\ref{prop:learning-to-estimation}
    and Proposition~\ref{prop:estimation-lower-bound} to prove Theorem~\ref{thm:sample-complexity-lower-bound}, the lower bound for dictionary learning in the rank-constrained setting.

    Let $\Alg$ be a learning rule as in the statement of the theorem.  We will
    use it to construct a tester for the multiple hypothesis testing problem associated with the packing $\Packing$ with minimum distance $\sqrt{PQ}\gamma$.  In particular, given a dataset $X$ generated by a ground-truth dictionary $D_* \in \Packing$,
    $\Alg$ outputs a dictionary $D \in \Packing$.  The test then outputs
    a dictionary $\hat{D} \in \Packing$ in the set $\argmin_{d \in \Packing} \ddict(d, D)$.  From Proposition~\ref{prop:learning-to-estimation}, we have
    that with probability at least $1-\delta$, 
    $ %\begin{align}
        \ddict(D, D_*) 
        \leq C\cdot \sqrt{PQ\epsilon}.
    $ %\end{align}
    This implies by the triangle inequality that
    \begin{align}
        \ddict(\hat{D}, D_*)
        &\leq \ddict(\hat{D}, D) + \ddict(D, D_*) \\
        &\leq \sqrt{PQ}\gamma/2 + C \sqrt{PQ\epsilon}.
    \end{align}
    Provided that $\sqrt{\epsilon} < \gamma/2$, we have that
    $ %\begin{align}
        \ddict(\hat{D}, D_*) < \sqrt{PQ}\gamma,
    $ %\end{align}
    which implies by the minimum distance property of the packing that
    $\hat{D} = D_*$, so that the test constructed from $\Alg$ is correct with
    probability at least $1-\delta$.

    From Proposition~\ref{prop:estimation-lower-bound}, we then have that
    $ %\begin{align}
        \delta 
        \geq 1 - \frac{S\log(PQ)}{\Omega((NP+MQ)\log(1/\epsilon)}.
    $ %\end{align}
    This completes the proof of Theorem~\ref{thm:sample-complexity-lower-bound}.
%\end{proof}

\section{Conclusion}
In this paper we introduced \ourmeth, a 2D dictionary learning model with low rank sparse coding for 2D data samples. We established  theoretical sample complexity bounds (both upper and lower) for our proposed problem that elucidate the effect on sample complexity of rank and sparsity constraints on coding matrices in the dictionary learning problem. We also proposed an optimization approach, called \ourmeth, which learns both the low-rank coding matrices and the dictionaries by alternating optimization.  We showed that the objective function values of the alternating optimization converge and that a solution to the $L_1$-regularized problem solves the original dictionary learning problem with $L_1$ constraints.  We demonstrated the quality of \ourmeth on five real-world datasets in comparison to analytical dictionary baselines as well as dictionary learning methods. \ourmeth outperformed all state-of-the-art baselines in data reconstruction and missing value imputation tasks. Compared to the best dictionary learning baselines, \ourmeth obtained up to $10\times$ model size reduction for the same representation quality in real-world datasets. It also outperformed baselines for missing value imputation. The atoms learned by our model represented intuitive road traffic patterns.
%As a future direction, we plan to extend our dictionary learning approach to multi-way data (i.e., tensors) samples making the core idea applicable to a wider range of problem settings.

%\section{Acknowledgments}

%{\appendix[Proof of the Zonklar Equations]
%Use $\backslash${\tt{appendix}} if you have a single appendix:
%Do not use $\backslash${\tt{section}} anymore after $\backslash${\tt{appendix}}, only $\backslash${\tt{section*}}.
%If you have multiple appendixes use $\backslash${\tt{appendices}} then use $\backslash${\tt{section}} to start each appendix.
%You must declare a $\backslash${\tt{section}} before using any $\backslash${\tt{subsection}} or using $\backslash${\tt{label}} ($\backslash${\tt{appendices}} by itself
% starts a section numbered zero.)}
%


% \section{References}
% You can use a bibliography generated by BibTeX as a .bbl file.
%  BibTeX documentation can be easily obtained at:
%  http://mirror.ctan.org/biblio/bibtex/contrib/doc/
%  The IEEEtran BibTeX style support page is:
%  http://www.michaelshell.org/tex/ieeetran/bibtex/
 

\bibliographystyle{IEEEbib}
\bibliography{reference}

\newpage
\section*{Appendix}

In this Appendix we provide supplemental content including omitted proofs, optimization steps for \ourmeth and baselines, as well as \ourmeth in the presence of missing values, additional evaluation results, and hyperparameter tuning protocol to aid reproducibility.

\section{Omitted proofs of lemmas used in the proof of Theorem~\ref{thm:dictionary-learning-generalization-bound}}
\label{sec:proofs-omitted-upper-bound}

\subsection{Proof of Lemma~\ref{lemma:lipschitz-bound}}
\label{sec:proof-lemma-lipschitz-bound}
%\begin{proof}
    We expand out the definition of $G$ and the $L_{\infty}$ norm:
    \begin{align}
        &\|G(L, R) - G(L', R')\|_{\infty} \\
        &= \max_{X} | \min_{Y,W} \|X - LYWR^{T}\|_F^2 - \min_{Y',W'} \|X - L'Y'W'R'^T\|_F^2 |
        \label{expr:starting-lipschitz}
    \end{align}
    Now, note that $\min_{Y,W} \|X - L'Y'W'R'^T\|_F^2 \leq \|X - L'YWR'^T\|_{F}^2$, where $(Y, W)$ minimizes the
    $(L, R)$ term in the above expression.  Thus, if we can upper bound
    \begin{align}
        |\|X - LYWR^{T}\|_F^2 - \|X - L'YWR'^T\|_F^2|
    \end{align}
    for arbitrary $Y, W$, this gives us an upper bound on (\ref{expr:starting-lipschitz}).  We next have
    \begin{align}
        &|\|X - LYWR^{T}\|_F^2 - \|X - L'YWR'^T\|_F^2| \\
        &\leq 2 |\|X - LYWR^{T}\|_F - \|X - L'YWR'^T\|_F| \\
        &\leq 2\| LYWR^T - L'YWR'^T\|_{F} \\
        &\leq 2 \|(L, R) - (L', R')\|_{op,1,2} \cdot  \|YW\|_{1} \label{expr:using-submultiplicativity} \\
        &\leq 2 \cdot \|(L, R) - (L', R')\|_{op,1,2}\cdot  \|Y\|_{1} \|W\|_1 \\
        &\leq 2 \cdot \kappa^2  \|(L, R) - (L', R')\|_{op,1,2}.
    \end{align}
    This implies the claimed Lipschitz bound.

\section{Omitted proofs of lemmas used in the proof of  Theorem~\ref{thm:sample-complexity-lower-bound}}
\label{sec:proofs-omitted-lower-bound}

\subsection{Proof of Lemma~\ref{lemma:packing-existence}}
\label{sec:proof-lemma-packing-existence}
%    \begin{proof}
        \textbf{Step 1: Reduce the problem of packing of the quotient space to packing of the ground space.}
        We first claim the following relationship between the packing number of $\Dicts$ and that of the space $\DictsOrd$ of \emph{ordered} dictionaries (of which $\Dicts$ is a quotient space with respect to equivalence under permutation of atoms):
        \begin{align}
            \PackingNum(\Dicts, \sqrt{PQ}\gamma) 
            \geq \PackingNum(\DictsOrd, \sqrt{PQ}\gamma)/(PQ)!. 
            \label{expr:reduction-to-dictsord}
        \end{align}

        The proof is as follows:
            in any packing of $\DictsOrd$, there are at most $(PQ)!$ members of each dictionary equivalence class.    So taking a maximal packing of $\DictsOrd$ and deleting redundant dictionaries leaves a packing of $\Dicts$ with cardinality at least the packing number of $\DictsOrd$ divided by $(PQ)!$.

        Thus, it is sufficient for us to lower bound the packing number
        of $\DictsOrd$.  In what follows, it will also be useful to consider
        the set $\DictsOrdUnconstrained$ of ordered dictionaries without the
        mutual coherence constraint:
        \begin{align}
            \DictsOrdUnconstrained 
            := \{ L \tensor R ~|~ \|L\|_F = \|R\|_F = 1\}.
            \label{expr:dicts-ord-unconstrained}
        \end{align}
        We also denote by $\DictsOrdUnconstrainedL$ and $\DictsOrdUnconstrainedR$ the spaces of left and right ordered dictionaries unconstrained by mutual coherence, respectively.

        %%%%%%
        \textbf{Step 2: Lower bound the packing number of the Kronecker product space by (approximately) the product of the packing numbers of the left and right dictionary spaces.}
            %\am{FINISH ME: Idea is to take maximal packings of both the left and right spaces, then take Kronecker products of all members.  But we need to check the separation property to complete the proof.}
            We next show that if there exist Frobenius norm packings $\Packing_L, \Packing_R$ of the left and right dictionary spaces with distance $\sqrt{P}\gamma$ and $\sqrt{Q}\gamma$,
            respectively, then there exists a packing $\Packing$ of $\DictsOrd$
            with cardinality $|\Packing_L|\cdot |\Packing_R|$ and minimum distance $\sqrt{PQ}\gamma \cdot (1 - o(1))$ as $\gamma\to 0$.

            %\am{Give the proof.  It's in the Remarkable, tspdraft.}
            The proof is as follows:
                we consider the following packing:
                \begin{align}
                    \Packing := \{ L \tensor R ~|~ L \in \Packing_L, R \in \Packing_R \}.
                \end{align}

                %%%%%
                %\am{Analyze the cardinality of $\Packing$: }
                The cardinality of $\Packing$ is
                $ %\begin{align}
                    |\Packing| = |\Packing_L| \cdot |\Packing_R|,
                $ %\end{align}
                by the unique factorization property of the tensor product.

                %%%%%
                %\am{Verify the distance property of $\Packing$: }
                We next need to verify the distance property of $\Packing$.
                To this end, let $L_1 \tensor R_1$ and $L_2 \tensor R_2 \in \Packing$.
                Then we have, by direct calculation,
                \begin{align}
                    &\| L_1 \tensor R_1 - L_2 \tensor R_2\|_F^2 \\
                    &= \|L_1\|_F^2 \cdot \|R_1\|_F^2 - 2 \Tr(L_1^T L_2) \Tr(R_1^T R_2).
                    \label{expr:packing-dist-1}
                \end{align}
                Now, for each $b \in \{1, 2\}$, 
                $ %\begin{align}
                    \|L_b\|_F^2 = P, \|R_b\|_F^2 = Q,
                $ %\end{align}
                by the column normalization condition on the dictionaries.
                We thus have
                \begin{align}
                    (\ref{expr:packing-dist-1})
                    = 2PQ - 2 \Tr(L_1^T L_2) \Tr(R_1^T R_2).
                    \label{expr:packing-dist-2}
                \end{align}

                We next show that the two trace expressions in (\ref{expr:packing-dist-1}) have simpler expressions in terms of the distance between the marginal dictionaries.
                In particular, we claim that
                \begin{align}
                    \Tr(L_1^T L_2) = P - \frac{1}{2} \|L_1 - L_2\|_F^2, \\
                    \Tr(R_1^T R_2) = Q - \frac{1}{2} \|R_1 - R_2\|_F^2.
                    \label{expr:trace-to-distance}
                \end{align}
                We show this only for $\Tr(L_1^T L_2)$, as the proof for $R_1^TR_2$
                is entirely analogous.
                We have
                \begin{align}
                    &\Tr(L_1^T L_2) \\
                    &= \Tr((L_1 - L_2+ L_2)^T (L_2 - L_1 + L_1)) \\
                    &= \Tr((L_1 - L_2)^T (L_2 - L_1))  \\
                    &~~~+ \Tr(L_1^TL_1) + \Tr(L_2^TL_2) - \Tr(L_1^TL_2) \\
                    &= - \|L_1 - L_2\|_F^2 + 2P - \Tr(L_1^TL_2).
                \end{align}
                Rearranging and dividing both sides by $2$ completes the derivation of (\ref{expr:trace-to-distance}). 

                We thus have
                \begin{align}
                    &(\ref{expr:packing-dist-2})
                    = 2PQ - 2 \cdot (P - \frac{1}{2} \|L_1 - L_2\|_F^2)(Q - \frac{1}{2} \|R_1 - R_2\|_F^2) \\
                    &= P \|R_1 - R_2\|_F^2 + Q \|L_1 - L_2\|_F^2 \\
                        &~~~
                        - \frac{1}{2}\|L_1 - L_2\|_F^2 \|R_1 - R_2\|_F^2.
                    \label{expr:packing-dist-3}    
                \end{align}

                Next, we appeal to the distance property of the left and right packings.
                $ %\begin{align}
                    (\ref{expr:packing-dist-3})
                    \geq PQ \gamma^2 - \frac{1}{2}PQ\gamma^4
                    = PQ \gamma^2 (1 - \frac{1}{2} \gamma^2).
                $ %\end{align}
                This completes the proof of the claim reducing the problem of constructing $\Packing$ from left and right dictionary packings.  In particular, we have shown, for any $\hat{\gamma}$ small enough and $> 0$,
                \begin{align}
                    &\PackingNum(\DictsOrd, \sqrt{PQ}\hat{\gamma}\sqrt{1 - \hat{\gamma}/2})  \\
                    &\geq \PackingNum(\DictsOrdL, \sqrt{P}\hat{\gamma})\PackingNum(\DictsOrdR, \sqrt{Q}\hat{\gamma}).
                    %\geq (C_1/\hat{\gamma})^{(N-1)P + (M-1)Q},
                    \label{expr:reduction-to-marginal-packings}
                \end{align}

        \textbf{Step 3: Lower bound the packing number of the left dictionary space.}
            Next, we claim that the $\sqrt{P}\gamma$-packing number of the space $\DictsOrdL$ of left
            dictionaries %\emph{without the mutual coherence constraint} 
            is at least 
            %\am{TODO: Check out the denominator and make sure it's correct.} 
            $\left( \frac{C}{\gamma} \right)^{\Omega(NP)}$, and the $\sqrt{Q}\gamma$-packing number of the space of right dictionaries is at least $\left(\frac{C}{\gamma} \right)^{\Omega(MQ)}$.  
            %We focus on showing this for
            %the left dictionary space, since the proof for the right space is entirely analogous.
            The proof for both cases is the same, so we focus only on the
            one for the left dictionaries.

            Letting $B(r)$ denote the Frobenius ball with radius $r$ in $\R^{N\times P}$, we have that
            \begin{align}
                &\PackingNum(\DictsOrdL, \sqrt{P}\gamma) 
                % Standard volumetric ratio inequality. 
                \geq \frac{\Vol(\DictsOrdL)}{\Vol(B(\sqrt{P}\gamma)} \\
                % Multiplying and dividing by $\Vol(\DictsOrdUnconstrainedL)$
                &= \frac{\Vol(\DictsOrdL)}{\Vol(\DictsOrdUnconstrainedL)}\cdot
                    \frac{\Vol(\DictsOrdUnconstrainedL)}{B(\sqrt{P}\gamma)}.
            \end{align}
            By a standard concentration of measure argument, the ratio
            $\frac{\Vol(\DictsOrdL)}{\Vol(\DictsOrdUnconstrainedL)}$
            is $1-o(1)$, and the second ratio is $(C/\gamma)^{(N-1)P}$
            for some positive constant $C$, by known asymptotics for
            volumes of balls.  The same derivation implies analogous asymptotics for the packing numbers of $\DictsOrdR$.  This implies
            \begin{align}
                \PackingNum(\DictsOrdL, \sqrt{P}\gamma)
                &= (1 - o(1))\cdot (C/\gamma)^{(N-1)P} \\
                \PackingNum(\DictsOrdR, \sqrt{Q}\gamma)
                &= (1 - o(1))\cdot (C/\gamma)^{(N-1)Q}.
                \label{expr:marginal-ordered-packing-nums}
            \end{align}

        \textbf{Step 5: Put everything together.}
            We put everything together as follows:
            from (\ref{expr:reduction-to-marginal-packings}) and 
            (\ref{expr:marginal-ordered-packing-nums}), we have
            \begin{align}
                \PackingNum(\DictsOrd, \sqrt{PQ}\hat{\gamma}\sqrt{1-\hat{\gamma}^2/2})
                \geq (C_1/\hat{\gamma})^{(N-1)P + (M-1)Q},
            \end{align}
                 for some constant $C_1 > 0$.
                Setting $\gamma := \hat{\gamma}\sqrt{1-\hat{\gamma}/2}$
                and noting that $\gamma \geq const \cdot \hat{\gamma}$ provided that $\hat{\gamma}$ is small enough,  this implies
                that
                \begin{align}
                    \PackingNum(\DictsOrd, \sqrt{PQ}\gamma) 
                    \geq (C/\gamma)^{(N-1)P + (M-1)Q},
                \end{align}
                for some $C>0$.

                Finally, applying (\ref{expr:reduction-to-dictsord})
                completes the proof of Lemma~\ref{lemma:packing-existence}.
\subsection{Proof of Proposition~\ref{prop:learning-to-estimation}}
\label{sec:proof-prop-learning-to-estimation}

    To prove Proposition~\ref{prop:learning-to-estimation}, we need some 
    lemmas. 
    The first lemma gives a characterization of the excess risk of a 
    hypothesis in terms of the expected squared Frobenius distance between a random 
    ground truth dictionary atom and the set of matrices in the range of
    the hypothesis.  To state it, we define the \emph{range} of a dictionary.
    \begin{align}
        \range((L, R)) := \{ X ~|~ X = LZR^{T}, Z \in \CoeffMat(k, \kappa) \}.
    \end{align}
    This is simply the set of matrices $X$ that can be represented exactly by
    the dictionary $(L, R)$.
    We also define $\FullRange((L, R))$ to be the range of the given dictionary, except that we drop the rank constraint.  We note that
    $\range((L, R)) \subset \FullRange((L, R))$,  and $\FullRange((L, R))$ is a polytope whose vertices correspond to the
    atoms of the dictionary.
    
    \begin{lemma}[Excess risk lower bound by polytope projection]
        \label{lemma:excess-risk-formula}
        Let $D \in \Packing$.  Let $X^*$ be a uniformly random atom chosen from
        the ground truth dictionary $D_*$.
        Then
        \begin{align}
            &\ExcessRisk(D)  \\
            &\geq \E_X[\min_{Y \in \range(D)} \|X^* - Y\|_F^2]
            %+ O\left( \sqrt{\frac{\sigma^2 NM}{PQ}} \ddict(D_*, D) \right).
            %+ O\left( \frac{\sigma \log(PQ) \ddict(D_*, D)}{\sqrt{PQ}} \right).
            - O( \sigma \sqrt{\log(PQ)}) \\
            &= \E_X[\min_{Y \in \range(D)} \|X^* - Y\|_F^2] - O\left(\frac{1}{\sqrt{\log(PQ)}}\right).
        \end{align}
        %\am{TODO: The first term is okay, but the second is too loose because of the presence of $\sqrt{NM}$.  This comes from using Cauchy-Schwarz in the proof.  With more work, we can reduce $\sqrt{NM}$ to $\log(PQ)$.  Must do this.}
        %\am{We only really need a decent lower bound that looks like this, so that we can plug it into Lemma~\ref{lemma:excess-risk}  See the German notebook, p 53.}
        %\am{We need to assume that all dictionaries are in the packing.}
    \end{lemma}
    \begin{proof}
        %\am{TODO: Define the inner product of matrices that we're using.  It's just the dot product of their vectorizations.}
        Below, we use the following inner product: $\innerproduct{A}{B} = \Tr(A^TB)$.
    
        %\am{Step 1: Compute an explicit expression for $R(D)$ and for
        %$R(D_*)$.  This immediately gives an expression for $\Delta R(D)$.
        %The expression is $\E_{X}[ \min_{Y \in \range(D)} [ \|X^*-Y\|_F^2 + 2 \innerproduct{\hat{\epsilon}_X}{X^*-Y}  ]]$.}
        \noindent{\textbf{Step 1: Compute an explicit expression for $R(D)$: }}
        To compute an explicit expression for $\Delta R(D)$, we first compute one for $R(D)$:
        \begin{align}
            &R(D) 
            % By definition
            := \E_X[ \min_{Y \in \range(D)} \|X - Y\|_F^2 ] \\
            % By definition of X^*
            &= \E_X[ \min_{Y \in \range(D)} \|X^*+\noiseterm - Y\|_F^2] \\
            % Expression as an inner product
            &= \E_X[ \min_{Y \in \range(D)} \innerproduct{X^*-Y}{X^*-Y}  \\
            &~~~~~~~~~~~~~~+ 2\innerproduct{\noiseterm}{X^*-Y} + \|\noiseterm\|_F^2 ].
            \label{expr:RD-explicit}
        \end{align}
        Here, the first equality is by definition of the risk, the second is by definition of $X^*$, and the third is by expression of the squared Frobenius norm as an inner product and then use of linearity of the inner product in its arguments.

        To write down $\ExcessRisk(D)$, we will denote by $Y^* \in \range(D_*)$ the analogous quantity in $R(D_*)$ to $Y$ in the above expression.  We then have
        \begin{align}
            \ExcessRisk(D) 
            &\geq \E_{X^*}[ \min_{Y \in \range(D)} \|X^* - Y\|_F^2    \\
                &+ 2 \innerproduct{\noiseterm}{X^* - Y} - 2\innerproduct{\noiseterm}{X^* - Y^*}] \\
            &\geq  \E_{X^*}[ \min_{Y \in \range(D)} \|X^* - Y\|_F^2]  \\
                & - 2 \sup_{\substack{Y \in \range(D) \\ Y^* \in range(D_*)}}| \E_{\noiseterm}[\innerproduct{\noiseterm}{Y - Y^*}]|.
                \label{expr:reduction-to-overfitting}
        \end{align}
        The first inequality is using the inequality $\min_{z}[ A(z) + B(z)] \geq \min_{z} A(z) + \min_{z} B(z)$.  The second inequality is by the inequality $A(z) + B(z) \geq A(z) - \sup_{z} B(z)$.
        We now define 
        \begin{align}
            \Overfitting(D)
            := 2 \sup_{\substack{Y \in \range(D) \\ Y^* \in range(D_*)}}| \E_{\noiseterm}[\innerproduct{\noiseterm}{Y - Y^*}]|.
        \end{align}

        %%%%%%%%
        \noindent{\textbf{Step 2: Upper bounding $\Overfitting(D)$ via Gaussian widths: }}
        We claim that $\Overfitting(D) = O(\sigma \sqrt{\log(PQ)})$
        for large enough $P, Q$.
        To show this, we proceed as follows:
        \begin{align*}
            \Overfitting(D)
            % By symmetry.
            &= 2 \sup_{\substack{Y \in \range(D) \\ Y^* \in range(D_*)}}| \E_{\noiseterm}[\innerproduct{\noiseterm}{Y + Y^*}]| \\
            % By triangle inequality.
            &\leq 2\sup_{\substack{Y \in \range(D) \\ Y^* \in range(D_*)}}
                [|\E_{\noiseterm}[\innerproduct{\noiseterm}{Y}|
                 + |\E_{\noiseterm}[\innerproduct{\noiseterm}{Y^*}|] \\
            % By sup || <= |sup|.
            &\leq 2 ( 
                | \sup_Y \E_{\noiseterm}[\innerproduct{\noiseterm}{Y}]  |
                + | \sup_{Y^*} \E_{\noiseterm}[\innerproduct{\noiseterm}{Y^*}]  |
            ) \\
            % By sup_z \E[A(z)] \leq \E[\sup_z A(z)].
            &\leq 2 (
                \gwidth(\range(D)) + \gwidth(\range(D_*))
            ).
        \end{align*}
        Here, the first equality is by closure of $\range(D_*)$ under negation.  The first inequality is the triangle inequality.  The second inequality is using the fact that $\sup_x |A(x)| \leq |\sup_{x} A(x)|$.  The last inequality is using the fact that
        $\sup_x \E[A(x)] \leq \E[\sup_x A(x)]$ and the definition of Gaussian width.

        %%%%%%
        \noindent{\textbf{Step 3: Upper bounding Gaussian width: } }
        To complete the proof, we upper bound the Gaussian width
        of $\range(D)$.  The same upper bound holds for $\gwidth(\range(D_*))$.  By definition of $\range(D)$, $Y$
        takes the form $LZR^T$, for a matrix $Z \in \CoeffMat(k=1, \kappa=1)$.
        This implies that $\rank(Y) = 1$ and $\|Y\|_1 \leq 1$.  Thus,
        we can write $Y$ as an outer product: $Y = \ell^T r$, for
        vectors $\ell, r$ with $L_1$ norm bounded by $1$ (this because the dictionary atoms are themselves operator norm-constrained, by assumption).

        Thus, we have
        \begin{align}
            &\gwidth(\range(D)) 
            = \E_{\noiseterm}[\sup_{Y\in\range(D)}\innerproduct{\noiseterm}{Y}] \\
            &\leq \E_{\noiseterm}[ \sup_{\|\ell\|_1, \|r\|_1 \leq 1} \ell^T\cdot \noiseterm \cdot r] 
            = \E_{\noiseterm}[\|\noiseterm\|_{\infty}] \\
            &= O(\sigma \sqrt{\log(PQ)}).
        \end{align}
        Here, the first equality is the definition of Gaussian width.  The first inequality is by expanding the set over which we take the supremum to all pairs of $L_1$-constrained vectors.  The second equality is by the dual norm characterization of the $L_{\infty}$ norm of a matrix.  The final equality is by a known bound on the
        expected maximum of a set of $PQ$ $\Normal(0, \sigma^2)$-distributed random variables.

        The same derivation holds for $\gwidth(D_*)$.  Thus, we have shown
        that $\Overfitting(D) = O(\sigma \sqrt{\log(PQ)})$.  Using our choice of $\sigma$ and (\ref{expr:reduction-to-overfitting}) completes the proof.

    \end{proof}

    %%%%%%%%%%%
    The second lemma gives the promised bound on $\ddict^2$ in terms of the 
    excess risk.
    \begin{lemma}[Excess risk is lower bounded by distance to ground truth]
        \label{lemma:excess-risk}
        We have, for any dictionary $D \in \Packing$ and for the ground truth dictionary $D_*$,
        $ %\begin{align}
            \ddict^2(D, D_*) 
            \leq O(PQ \ExcessRisk(D)).
        $ %\end{align}
    \end{lemma}
    \begin{proof}
        %\am{We get this by Lemma~\ref{lemma:excess-risk-formula} to reduce to the noiseless projection formula with a Gaussian width error term.  The noiseless projection formula gets further lower bounded by something like $\ddict^2(D, D_*)/(PQ)$.}

        We start with Lemma~\ref{lemma:excess-risk-formula}. 
        Summing over all dictionary atoms $X^*$ of $D_*$, we get
        \begin{align}
            &PQ \ExcessRisk(D) \\
            % By multiplying both sides of the previous lemma by PQ
            % and using the definition of the expectation with respect to uniformly random $X^*$.
            &\geq \sum_{X^*} \min_{Y\in\range(D)} \|X^*- Y\|_F^2 
            %+ O(\sqrt{\sigma^2 NM PQ} \ddict(D_*, D)).
            - O(\sigma PQ\sqrt{\log(PQ)}).
        \end{align}
        Here, we have used the fact that the expectation defining $\ExcessRisk(D)$ is with respect to a uniformly random $X^*$
        among dictionary atoms of $D_*$.

        %%%%%
        We focus now on lower bounding $\min_{Y\in\range(D)} \|X^*- Y\|_F^2$.  We expand this using properties of inner products:
        \begin{align}
            &\|X^* - Y\|_F^2 
            % Using the inner product definition of the squared Frobenius norm.
            = \innerproduct{X^*}{X^*} - 2\innerproduct{X^*}{Y} + \innerproduct{Y}{Y} \\
            % Since X^* is a dictionary atom.
            &= 1 - 2\innerproduct{X^*}{Y} + \innerproduct{Y}{Y}.
            \label{expr:intermediate-inner-product}
        \end{align}
        The second equality is by the dictionary atom normalization. 

        To proceed, since $Y \in \range(D)$, we can write $Y$ as
        $Y = Dz$, for some $z \in \CoeffMat(k, \kappa)$.   

        %%%
        %\am{Upper bound $\innerproduct{X^*}{Y}$.}
        To lower bound (\ref{expr:intermediate-inner-product}), 
        we first upper bound $\innerproduct{X^*}{Y}$:
        \begin{align}
            \innerproduct{X^*}{Y}
            = \innerproduct{z}{D^T x^*} 
            &\leq \|z\|_1 \cdot \|D^T \cdot X^*\|_{\infty}   \\
            &\leq \|D^T \cdot X^*\|_{\infty}.
        \end{align}
        where the first equality is using the definition of $Y$ and
        properties of the inner product.  The first inequality is H\"older's
        inequality.  The second inequality is by the $L_1$ constraint on $z$.

        % From the Remarkable notes.
        Next, let $d$ be the atom of $D$ that minimizes $\|X^* - d\|_F^2$.
        Using properties of the inner product, we can show that
        \begin{align}
            \|D^T \cdot X^*\|_{\infty}
            \leq 1 - \frac{1}{2}\|X^* - d\|_F^2.
        \end{align}

        %%%%%%
        %\am{Lower bounding $\|Y\|_F^2$ by $(1+\mu)\|z\|_2^2 - \mu\|z\|_1^2$.}
        Next, we lower bound $\|Y\|_F^2$.  To do this, we write it as
        $\|Y\|_F^2 = z^TD^TDz$.  Since the atoms of $D$ are approximately orthonormal, we define $M := D^TD - I$ and write $D^TD = I + M$, where $I$ is the identity matrix.  Intuitively, $M$ should be approximately $0$ in all entries.  Then
        \begin{align}
            z^TD^TDz
            = z^T(I + M)z
            = \|z\|_2^2 + z^TMz
            \geq \|z\|_2^2 - |z^TMz|,
        \end{align}
        by simple algebra.  Now, let $\mu$ denote the mutual coherence
        of the dictionary $D$.  Then
        %\am{TODO: Justify this.  Easy.}
        $ %\begin{align}
            |z^TMz|
            \leq \mu \|z\|_1^2 - \mu \|z\|_2^2.
        $ %\end{align}
        This implies
        \begin{align}
            \|Y\|_F^2
            \geq \|z\|_2^2 + \mu \|z\|_2^2 - \mu \|z\|_1^2
            = (1+\mu) \|z\|_2^2 - \mu \|z\|_1^2.
        \end{align}

        %%%%%%%
        %\am{Putting everything together.}
        Putting everything together, we get
        \begin{align}
            &\|X^* - Y\|_F^2 \\
            &\geq 1 - 2\left(1 - \frac{1}{2}\|X^* - d\|_F^2 \right)
                + (1+\mu) \|z\|_2^2 - \mu \|z\|_1^2.
        \end{align}
        Minimizing over all $z$ subject to the constraint defining it,
        we finally get the desired result.  We thus have shown that
        \begin{align}
            PQ\ExcessRisk(D)
            \geq \ddict^2(D, D_*) - O(\sigma PQ \sqrt{\log(PQ)}),
        \end{align}
        and using our choice of $\sigma$, we find that the second term
        in the above expression is $o(PQ)$, while $\ddict^2(D, D_*) = \Omega(PQ)$ by the fact that both $D, D_* \in \Packing$.  Thus, we
        can ignore the second term at the expense of introducing a positive constant on the left-hand side of the inequality.  This completes the proof of Lemma~\ref{lemma:excess-risk}.
    \end{proof}

    Lemma~\ref{lemma:excess-risk} has an easy consequence: the ground truth dictionary $D_*$ is the unique global minimizer of the risk.  Moreover, empirical risk minimization recovers it with sufficient data. 
    %\am{Implications: Ground truth is the unique global minimizer of the risk.  So the model is identifiable, and ERM finds it.}

    %%%%%%%%%%%%%%%%%%%%%%%%%%%%%%%
    We finally have enough to prove Proposition~\ref{prop:learning-to-estimation} reducing the multiple hypothesis testing problem to PAC learning.
    \begin{proof}[Proof of Proposition~\ref{prop:learning-to-estimation}]
        Let $D := \Alg(X)$.  Since $\Alg$ is $(\epsilon, \delta)$-PAC,
        we have that with probability at least $1-\delta$, 
        $\ExcessRisk(D) \leq \epsilon$.  Then from Lemma~\ref{lemma:excess-risk}, we get
        \begin{align}
            \ddict^2(D, D_*) 
            \leq O(PQ \ExcessRisk(D)) 
            \leq O(PQ \epsilon).
        \end{align}
        This completes the proof.
    \end{proof}

%    %%%%%%%%%%%%%%%%%%%%%%%%%%%%%
%    \paragraph{A sample complexity lower bound for the multiple hypothesis testing problem}
%
%    Our next goal is to prove an error probability lower bound for the multiple hypothesis testing problem associated with $\Packing$.
%    \begin{proposition}[Testing error probability lower bound]
%        \label{prop:estimation-lower-bound}
%        Let $F$ be a tester for dictionaries in $\Packing$ from a
%        dataset $X := \{X_s\}_{s=1}^S$ sampled from $\Dist$ with ground truth dictionary $D_*$.  We have that with probability
%        at least
%        \begin{align}
%            1 - \frac{S}{\Omega(\sigma^2 (NP + MQ) \cdot \log(C/\gamma))},
%        \end{align}
%        for some positive constant $C$, $F(X) \neq D_*$.
%        In particular, to drive the probability of error to $0$,
%        we require $S = \Omega(\sigma^2\cdot (NP + MQ))$.
%    \end{proposition}

\subsection{Proof of Proposition~\ref{prop:estimation-lower-bound}}
\label{sec:proof-prop-estimation-lower-bound}
    %%%%%%%%%%%%%%
    To prove Proposition~\ref{prop:estimation-lower-bound}, we will need
    a standard corollary of Fano's inequality~\cite{Cover2006}.
    \begin{lemma}[Fano's inequality \cite{Cover2006}]
        \label{lemma:fano}
        Consider a finite set $\Omega$, a family of probability distribution $\{\Dist_\omega\}_{\omega\in \Omega}$, a random variable $J\sim \Uniform(\Omega)$,
        and, conditioned on $J$, $Z \sim \Dist_J$.  Let
        $F(Z) \in \Omega$ be an arbitrary estimator of the index $J$.  Then
        \begin{align}
            p_{err} \geq 1 - \frac{\sup_{j,k} D_{KL}(\Dist_j ~\|~ \Dist_k) + \log(2)}{\log(|\Omega|)},
        \end{align}
        where
        $ %\begin{align}
            p_{err} := \Pr[ F(Z) \neq J ]
        $ %\end{align}
        is the probability of error for the estimator $F$.
    \end{lemma}

    %%%%%%%%%%%%%%%%%%%%%%%%
    We will use $\Omega := \Packing$.  
    To apply Fano's inequality, we need  
    an upper bound on the KL divergence between the observation distributions induced by pairs of dictionaries in $\Packing$.
    This is the content of the next lemma.
    %We next show a KL divergence upper bound between observed matrix distributions induced by pairs of dictionaries in the packing.
    \begin{lemma}[KL divergence upper bound in terms of dictionary distance]
        \label{lemma:kl-divergence-upper-bound}
        For any pair of dictionaries $D_1, D_2 \in \Packing$, 
        let $X, \hat{X}$ denote two observation matrices generated according
        to $\Dist$ with the respective dictionaries.
        Then
        $ %\begin{align}
            D_{KL}(X ~\|~ \hat{X})
            \leq \frac{2}{\sigma^2}.
        $ %\end{align}
    \end{lemma}
    \begin{proof}
        By definition of $\Dist$, we have that $X$ and $\hat{X}$ are both Gaussian matrices with mean equal to a dictionary atom from their respective dictionaries and diagonal covariance matrix with all nonzero entries equal to $\sigma^2$ (so we write the covariance matrix as $\sigma^2 I$).  Call the expected values $\mu$ and $\hat{\mu}$.  This implies, via the standard formula for the KL divergence between Gaussian vectors,
        \begin{align}
            &D_{KL}(X ~\|~ \hat{X}) \\
            &= \frac{1}{2}[
                \log\frac{\det(\sigma^2 I)}{\det(\sigma^2I)}
                - NM  \\
                &~~~+ \Tr(\frac{1}{\sigma^2}\cdot \sigma^2 I)
                + (\hat\mu - \mu )^T \sigma^{-2} I (\hat{\mu} - \mu)
            ] \\
            &= \frac{1}{2}[
                -NM + NM + \frac{1}{\sigma^2}\|\hat{\mu} - \mu\|_F^2 
            ] \\
            &\leq \frac{1}{2\sigma^2} \cdot 4 = \frac{2}{\sigma^2}.
        \end{align}
        Here, the only nontrivial step is the inequality, which is justified by the triangle inequality applied to the Frobenius norm,
        followed by appealing to the fact that all of the dictionary atoms are normalized in Frobenius norm, completing the proof.
    \end{proof}

    We can now present the proof of Proposition~\ref{prop:estimation-lower-bound}, the lower bound for the multiple hypothesis testing problem.
    \begin{proof}[Proof of Proposition~\ref{prop:estimation-lower-bound}]
        We use Lemma~\ref{lemma:fano} with $\Omega$ given by the set of
        hard distributions indexed by elements of $\Packing$.  We use Lemma~\ref{lemma:kl-divergence-upper-bound} to upper bound the numerator and Lemma~\ref{lemma:packing-existence} to lower bound the denominator.  This completes the proof.
    \end{proof}

\section{Proof of Theorem~\ref{THM:REGULARIZATION-SOLVES-CONSTRAINED}}  %~\ref{thm:regularization-solves-constrained}}  %  {The regularized optimization problem solves the constrained one}
\label{sec:equivalence}
Here we describe in detail the relationship between the reduction of the constrained optimization problem for sparse coding to the regularized one.

We start with a lemma.
\begin{lemma}[A solution to an $L_1$ regularized problem is a solution for the constrained one]
    \label{lemma:regularized-solves-constrained}

    Suppose that $f:\R^{d_1} \times \R^{d_2} \times \cdots \times \R^{d_k} \to \R$. 
    For $\gamma \in [0, \infty)^k$, define
    \begin{align}
        F(\gamma) := \argmin_{(x_1, ..., x_k) \in \R^{d_1} \times \cdots \times \R^{d_k}} f(x) + \sum_{j=1}^k \gamma_j \|x_j\|_1
    \end{align}
    and
    \begin{align}
        G(\kappa) := \argmin_{x_1^k  \in \R^{d_1} \times \cdots \times \R^{d_k} ~|~ \forall j\|x_j\|_1 \leq \kappa} f(x_1^k).
        \label{expr:constrained-general}
    \end{align}

    Then for every $\gamma \geq \bf{0}$, there exists a $\kappa$ such that
    $F(\gamma) \subseteq G(\kappa)$.

    %Then there exists a bijection $\phi:[0, \infty)\to [0, \infty)$ such that
    %\begin{align}
    %    F(\phi(\kappa))
    %    = G(\kappa).
    %\end{align}
\end{lemma}
\begin{proof}
    %We show this in two parts: first, we show that any $x_0 \in F(\gamma)$ is also in $G(\kappa)$ for some $\kappa$.  
    %We then show that for any $x_1 \in G(\kappa)$, there exists $\gamma$ for which $x_1 \in F(\gamma)$.

    %%%%%%%%%%%%%%%
    %\textbf{A solution to the regularized problem is a solution to the constrained one: }
    Let $x^{(0)} \in F(\gamma)$ for some $\gamma$.  Set $\kappa_j := \|x_j^{(0)}\|_1$ for all $j$.  Let $x^{(1)} \in G(\kappa)$.
    By optimality of $x^{(0)}$ for the regularized problem, we have
    \begin{align}
        f(x^{(0)}) + \sum_{j=1}^k \gamma_j \|x_j^{(0)}\|_1
        \leq f(x^{(1)}) +  \sum_{j=1}^k \gamma_j \|x_j^{(1)}\|_1. 
        \label{expr:intermediate1}
    \end{align}
    By optimality of $x^{(1)}$ for the constrained problem, we have $\|x_j^{(1)}\|_1 \leq \kappa_j = \|x_j^{(0)}\|_1$ for all $j$.
    This implies by (\ref{expr:intermediate1}) that
    \begin{align}
        f(x^{(0)}) + \sum_{j=1}^k \gamma_j \|x^{(0)}_j\|_1
        \leq f(x^{(1)}) + \sum_{j=1}^k \gamma_j \|x^{(0)}_j\|_1,
    \end{align}
    which implies
    $ %\begin{align}
        f(x^{(0)}) \leq f(x^{(1)}).
    $ %\end{align}
    Since $x^{(0)}$ is feasible for the constrained problem, this implies that $x^{(0)}$ is a solution for it,
    meaning that $x^{(0)} \in G(\kappa)$, as desired.
    
    We have thus shown that for each $\gamma$, there exists $\kappa$ such that
    $F(\gamma) \subseteq G(\kappa)$.  Note that we did not use convexity of $f$ anywhere.
    %%%%%%%%%%%%%%%%
    %\paragraph{A solution to the constrained problem is a solution to the regularized one: }
    %\am{FINISH ME}
    %
\end{proof}

%\begin{remark}
    %\am{See page 94 of my quick notes on Remarkable.}
    Now, it is easy to check the following monotonicity property of $F(\gamma)$:
    if $\gamma_j \leq \hat{\gamma}_j$ for each $j$, and if we define
    $\kappa := \min\{ k ~|~ F(\gamma) \subseteq G(k)\}$
    and $\hat{\kappa} := \min\{ k ~|~ F(\hat{\gamma}) \subseteq G(k)\}$, then
    $ %\begin{align}
        \kappa \geq \hat{\kappa}
    $ %\end{align}
    and
    $ %\begin{align}
        F(\hat{\gamma}) \subseteq F(\gamma).
    $ %\end{align}
    Moreover, setting $\gamma = \bf{0}$ eliminates the constraint, meaning that
    $F(\bf{0}) \subseteq G(\infty)$.  
    This allows for a bisection search to select each regularization parameter to enforce desired
    constraints.  It also has the following consequence: for every $\kappa > 0$, there exists
    a $\gamma$ for which $F(\gamma) \subseteq G(\kappa)$. 
%\end{remark}

We can apply Lemma~\ref{lemma:regularized-solves-constrained} to our problem formulation as follows to show that it can be solved via the regularization approach.
The constrained problem takes the form
\begin{align}
    &\argmin_{L,R} \sum_{j=1}^S \ell((L, R), X_j) \\
    &= \argmin_{L, R} \sum_{j=1}^S \min_{Y_j, W_j ~:~ \|Y_j\|_1, \|W_j\|_1 \leq \kappa} \| X_j - L Y_j W_j R^T \|_{F}^2.
    \label{expr:constrained}
\end{align}
Using Lemma~\ref{lemma:regularized-solves-constrained}, we can solve the inner minimization by solving the regularized version: defining
\begin{align}
    &Y_{*,j}, W_{*,j}  \\
    &:= \argmin_{Y_j, W_j} \|X_j - LYWR^T\|_{F}^2 + \lambda_1 \|Y_j\|_1 + \lambda_2 \|W_j\|_1,
\end{align}
for appropriately chosen $\lambda_1(L, R, X_j)$ and $\lambda_2(L, R, X_j)$,
we get that
\begin{align}
    (\ref{expr:constrained})
    = \argmin_{L, R} \sum_{j=1}^S \| X_j - LY_{*,j}W_{*,j}R^T \|_{F}^2.
\end{align}

This shows that solving the sparse coding problem via the regularized objective function provides an exact solution to the original problem, which completes the proof.

\section{Proof of Theorem~\ref{THM:CONVERGENCE-AODL}}  %{Convergence of the alternating minimization procedure}
\label{sec:convergence-of-alternating-minimization}

%\am{FINISH ME: We can get an upper bound on the norm of the gradient map, which is a measure of optimality of a point.  But how do we translate this to a bound on the difference between the current point and an optimum?}
%
Here we prove that the sequence of objective function values in our alternating minimization scheme converges.  This follows from the fact that any monotone decreasing sequence of real numbers that is bounded below converges to its infimum.  To apply this to our case, we simply establish that the sequence of iterates of the dictionaries $(L^{(k)}, R^{(k)})$ and the sparse coding matrices $(Y^{(k)}, W^{(k)}$ yields a monotone decreasing sequence of objective function values.  The objective function is bounded below by $0$, and so this establishes convergence.

For convenience, we introduce notation for the objective function:
we denote the objective of the inner minimization by $J_{coding,s}(X, L, R, Y, W)$, and then the entire objective function becomes
\begin{align}
    J(X, L, R, Y, W)
    := \sum_{s=1}^S J_{coding,s}(X, L, R, Y, W).
\end{align}
For fixed $L, R, X$, replacing $Y_s, W_s$ by \\$\argmin_{Y_s, W_s}J_{coding,s}(X, L, R, Y, W)$ decreases the value of the objective, and the same holds when we replace $L, R$ by
$\argmin_{L, R} J(X, L, R, Y, W)$.
Thus, the value of the objective function is non-increasing in each iteration of the algorithm.  This implies convergence of the sequence of objective function values, which completes the proof.

%\am{We can show by upper bounding the gradient norm that this converges to a critical point of the function.  We know a rate of convergence of the gradient to the zero vector.  See page 6 of the alternating minimization analysis by Amir Beck.}

\section{\ourmeth algorithm and derivation details}
\label{appendix:alg}
\noindent{\bf Initialization of $L$ and $R$.}
We experiment with two different approaches to initialize the dictionaries in Step 3 of Alg. \ref{alg:aodl}: normally random and on tensor decomposition. In the former approach we sample each element interdependently from a normal distribution $\mathcal{N}(0, 1)$, and then normalize atoms to unit length. In the latter, we stack all samples in a 3-way tensor and employ Tucker decomposition, which decomposes a tensor into a product of 3 factor matrices and a core tensor with user-defined inner dimensions. We use the mode-1 factor with selected second dimension $P$ to initialize $L$, and mode-2 factor with selected second dimension $Q$ to initialize $R$. Note that the Tucker initialization of the dictionaries works only when we are aiming to learn a complete or under-complete dictionaries since Tucker requires $P \leq N, Q \leq M$. 
We used random initialization for all experiments (and all baselines that learn dictionaries) with synthetic data; and tensor decomposition initialization for all real-world experiments (and all baselines that learn dictionaries). These decisions were based on empirically faster convergence (fewer iterations) for all competing methods with the corresponding initialization schemes.

% \todo{say what we end up using in experiments, do you do the same initializations for baselines.}

\noindent{\bf Stage I: Sparse coding.} Given fixed dictionaries, we estimate the sparse coding coefficients one sample at a time. Since for any $X_s$, the problem is convex in $Y_s$ and $W_s$ individually,
% jointly convex \am{It's not jointly convex, right?}, 
we can employ the Alternating Direction Method of Multipliers (ADMM)~\cite{Boyd:2011:DOS:2185815.2185816} to solve for $Y_s, W_s$ similar to~\cite{TGSD}. We first introduce intermediate variables $U_s = Y_s, V_s = W_s$, and rewrite the objective for sample $s$ as:

\begin{equation}
    \begin{aligned}
         % f(Y_s, W_s, U_s, V_s) = 
         \underset{Y_s, W_s, U_s, V_s} {\mathrm{argmin}} \hspace{0.05cm}  & || X_s - L Y_s W_s R^T ||_F^2 + \lambda_1 ||U_s||_1 + \lambda_2 ||V_s||_1 \\
         s.t. \hspace{0.05cm} &Y_s = U_s, W_s = V_s
    \end{aligned} 
    \label{eq:obj_l1_admm}
\end{equation}

We form the corresponding Lagrangian function:

\begin{equation}
    \begin{aligned}
         &\mathcal{L}(Y_s, W_s, U_s, V_s)  \\
         = & || X_s - L Y_s W_s R^T ||_F^2 + \lambda_1 ||U_s||_1 + \lambda_2 ||V_s||_1  \\
         &+ \frac{\rho_1}{2} || U_s - Y_s + \frac{\Gamma_1}{\rho_1} ||_F^2 + \frac{\rho_2}{2} || V_s - W_s + \frac{\Gamma_2}{\rho_2} ||_F^2.
    \end{aligned} 
    \label{eq:obj_lag}
\end{equation}

We alternate between direct updates of $Y_s, W_s, U_s$ and $V_s$ obtained by setting gradients w.r.t. each variable to zero. %\todo{potentially move to appendix and leave final updates}
%{\bf Update $\bf Y_s, W_s$:}
To update $Y_s$, we have the following optimization problem:
\begin{equation}
    \begin{aligned}
         \underset{Ys} {\mathrm{argmin}} || X_s - L Y_s W_s R^T ||_F^2 + \frac{\rho_1}{2} || U_s - Y_s + \frac{\Gamma_1}{\rho_1} ||_F^2
    \end{aligned} 
    \label{eq:update_Ys}
\end{equation}

Setting the gradient with respect to $Y_s$ to zero, we obtain:
\begin{equation}
    \begin{aligned}
         2 L^T L Y_s B B^T + \rho_1 Y_s = 2 L^T X_s B^T + \rho_1 U_s + \Gamma_1,
    \end{aligned} 
    \label{eq:update_Ys_grad}
\end{equation}
where $B = W_s R^T$.
To simplify we use the following eigendecompositions: $L^T L = Q_1 \Lambda_1 Q_1^T, B B^T = Q_2 \Lambda_2 Q_2^T$, and set $\Pi_1 = 2 L^T X_s B^T + \rho_1 U_s + \Gamma_1$. We can then rewrite the above equation as:
\begin{equation}
    \begin{aligned}
         2 L^T L Y_s B B^T + \rho_1 Y_s =& \Pi_1 \\
         2 Q_1 \Lambda_1 Q_1^T Y_s Q_2 \Lambda_2 Q_2^T + \rho_1 Y_s =& \Pi_1 \\
         2 Q_1^T Q_1 \Lambda_1 Q_1^T Y_s Q_2 \Lambda_2 Q_2^T Q_2 + \rho_1 Q_1^T Y_s Q_2 =& Q_1^T \Pi_1 Q_2 \\
         2 \Lambda_1 Q_1^T Y_s Q_2 \Lambda_2 + \rho_1 Q_1^T Y_s Q_2 =& Q_1^T \Pi_1 Q_2\\
    \end{aligned} 
\end{equation}
Allowing $E_1 = Q_1^T Y_s Q_2$, we obtain $2 \Lambda_1 E_1 \Lambda_2 + \rho_1 E_1 = Q_1^T \Pi_1 Q_2$, and an element-wise solution for $E_1$ as follows:
\begin{equation}
    \begin{aligned}
         {[E_1]}_{i,j} = \frac{{[Q_1^T \Pi_1 Q_2]}_{i,j}}{2 {[\Lambda_1]}_{i,i} {[\Lambda_2]}_{j,j} + \rho_1},
    \end{aligned} 
\end{equation}
and $Y_s$ can then be recovered as $Y_s = Q_1 E_1 Q_2^T$. We follow a similar procedure to derive an analogous update for $W_s$.%, we can extract the term that contains $W_s$ from eq (\ref{eq:obj_lag}), take the gradient of it with respect to $W_s$ and set it to zero. Then, we can get $W_s$ in a similarly way as solving $Y_s$.

%{\bf Update $\bf U_s, V_s$:}

The optimization sub-problems for $U_s$ and $V_s$ are:
\begin{equation}
    \begin{aligned}
         \underset{U_s} {\mathrm{argmin}} \hspace{0.05cm}& \lambda_1 ||U_s||_1 + \frac{\rho_1}{2} || U_s - Y_s + \frac{\Gamma_1}{\rho_1} ||_F^2 \\
         \underset{V_s} {\mathrm{argmin}} \hspace{0.05cm}& \lambda_2 ||V_s||_1 + \frac{\rho_2}{2} || V_s - W_s + \frac{\Gamma_2}{\rho_2} ||_F^2,
    \end{aligned} 
\end{equation}
with existing closed-form solutions due to~\cite{Lin2013TheAL}:
\begin{equation}
    \begin{aligned}
         {[U_s]}_{i,j} =& \text{sign}({[H_1]}_{i,j}) \times \max{(|{[H_1]}_{i,j}| - \frac{\lambda_1}{\rho_1}, 0)}\\
         {[V_s]}_{i,j} =& \text{sign}({[H_2]}_{i,j}) \times \max{(|{[H_2]}_{i,j}| - \frac{\lambda_2}{\rho_2}, 0)},
    \end{aligned} 
\end{equation}
where $H_1 = Y_s - \frac{\Gamma_1}{\rho_1}, H_2 = W_s - \frac{\Gamma_2}{\rho_2}$.

The overall LRSC algorithm is listed in Alg.~\ref{alg:lrcs}. We first initialize all variables by sampling from a normal distribution $\mathcal{N}(0, 1)$ (Step 3) and then iterate over the derived 0-gradient updates for each variable in turn: update $Ys$ (line 5-9); update $Ws$ (line 10-14); update $U_s, V_s$ (line 15-17); and update $\Gamma_1, \Gamma_1$ in the end. The eigendecomposition steps (line 6 and line 11) are the most expensive steps (cubic in their input) since eigendecomposition compared to matrix multiplications requires iterations (depending on the solver). As a result, eigendecomposition is dominating the running time of each iteration of LRSC with a complexity of $O(P^3 + Q^3 + k^3).$

\begin{algorithm} [h]
\footnotesize
    \caption{LRSC}
        \begin{algorithmic}[1]
        \State {\bf Input:} A single samples $X_s$, dictionaries $L, R$, encoding rank $k$ and sparsity params. $\lambda_1, \lambda_2$
        \State {\bf Output:} Encodings $Y_s, W_s$
        \State Initialize $Y_s, W_s, U_s, V_s, \Gamma_1, \Gamma_2$ randomly
        \Repeat
            \State $B = W_s R^T$
            \State $Q_1 \Lambda_1 Q_1^T = eig(L^TL); ~~Q_2 \Lambda_2 Q_2^T = eig(BB^T)$
            \State $\Pi_1 = 2 L^T X_s B^T + \rho_1 U_s + \Gamma_1$
            \State ${[E_1]}_{i,j} = \frac{{[Q_1^T \Pi_1 Q_2]}_{i,j}}{2 {[\Lambda_1]}_{i,i} {[\Lambda_2]}_{j,j} + \rho_1}$
            \State $Y_s = Q_1 E_1 Q_2^T$
            \State $A = L Y_s$
            \State $Q_3 \Lambda_3 Q_3^T = eig(A^TA); ~~Q_4 \Lambda_4 Q_4^T = eig(R^TR)$
            \State $\Pi_2 = 2 A^T X_s R + \rho_2 V_s + \Gamma_2$
            \State ${[E_2]}_{i,j} = \frac{{[Q_3^T \Pi_3 Q_4]}_{i,j}}{2 {[\Lambda_4]}_{i,i} {[\Lambda_3]}_{j,j} + \rho_2}$
            \State $W_s = Q_3 E_2 Q_4^T$
            \State $H_1 = Y_s - \frac{\Gamma_1}{\rho_1}, H_2 = W_s - \frac{\Gamma_2}{\rho_2}$
            \State ${[U_s]}_{i,j} = \text{sign}({[H_1]}_{i,j}) \times \max{(|{[H_1]}_{i,j}| - \frac{\lambda_1}{\rho_1}, 0)}$
            \State ${[V_s]}_{i,j} = \text{sign}({[H_2]}_{i,j}) \times \max{(|{[H_2]}_{i,j}| - \frac{\lambda_2}{\rho_2}, 0)}$
            \State $\Gamma_1 = \Gamma_1 + \rho_1 (U_s - Y_s)$
            \State $\Gamma_2 = \Gamma_2 + \rho_2 (V_s - W_s)$
        \Until{Convergence or fixed max iterations}
        \end{algorithmic}
    \label{alg:lrcs}
\end{algorithm}

\noindent{\bf Stage II: Dictionary updates.} We employ gradient projection for dictionary updates given fixed $Y_s, W_s, \forall s \leq S$ with objective  
% \am{TODO: The notation on the left-hand side is incorrect, since $L, R$ are the variables being optimized over.} \boya{I have removed the $f()$ notation on the left side as in TGSD and other papers. Let me know if you think adding $f(\X)$ is better aligned with the theorem. }
\begin{equation}
    \begin{aligned}
        \underset{L, R} {\mathrm{argmin}} \sum_{s=1}^{S} (&|| X_s - L Y_s W_s R^T ||_F^2,
    \end{aligned} 
    \label{eq:obj_LR}
\end{equation}
and by setting gradients w.r.t. $L$ to zero, we obtain:
\begin{equation}
    \begin{aligned}
         \sum_{s=1}^{S} -2 (X_s - L Y_s W_s R^T) R W_s^T Y_s^T = 0,
    \end{aligned} 
\end{equation}
with a closed-form solution for $L$:
\begin{equation}
    \begin{aligned}
         L = (\sum X_s R W_s^T Y_s^T)(\sum Y_s W_s R^T R W_s^T Y_s^T)^{-1}.
    \end{aligned} 
\end{equation}
$R$'s update is derived in a similar manner and is also listed in Alg.~\ref{alg:aodl}
%\begin{equation}
%    \begin{aligned}
%         R = (\sum X_s^T L Y_s W_s)(\sum W_s^T Y_s^T L^T L Y_s W_s)^{-1}
%    \end{aligned} 
%\end{equation}
Atoms of both dictionaries are normalized by \emph{normalize($\cdot$)} in Alg.~\ref{alg:aodl} so the magnitude in representing samples is fully represented in the coding matrices as opposed to the dictionaries. % to unit length denoted by dictionaries that all its atoms have unit Euclidean norm:
%\begin{equation}
%    \begin{aligned}
%        L =& \text{normalize}((\sum X_s R W_s^T Y_s^T)(\sum Y_s W_s R^T R W_s^T Y_s^T)^{-1})\\
%        R =& \text{normalize}((\sum X_s^T L Y_s W_s)(\sum W_s^T Y_s^T L^T L Y_s W_s)^{-1})
%    \end{aligned} 
%\end{equation}

The dictionary updates involve matrix multiplications and matrix inversions, which run in cubic time in the input size. However, matrix inversion is in general much slower than regular multiplication. Thus, the matrix inversion term is dominating the running time with a complexity of $O(T + P^3 + Q^3)$, where $T$ is the product of the maximum $3$ values among $\{N, M, P, Q, k\}$. $T$ represents the run time of matrix multiplication inside the inversion term.  $P^3$ is the inversion in the solution of $L$, and $Q^3$ is the inversion in the solution of $R$.

\section{CMOD-ADMM derivation}
\label{appendix:cmod-admm}

While the original paper introducing the CMOD method employs 2D-OMP for the sparse coding step, our experimental analysis (Tbl.~\ref{table:datasets}) demonstrated that the reliance on OMP limits the method's scalability as NNZ grows. For this reason we derive and employ a version of CMOD with ADMM sparse coding in the sparse coding subproblem (with fixed dictionaries $L$ and $R$) is as follows:
\begin{equation}
    \begin{aligned}
         % f(\X) = 
         \underset{Z} {\mathrm{argmin}} \sum_{s=1}^{S} (&|| X_s - L Z_s R^T ||_F^2 + \lambda_1 ||Z_s||_1
    \end{aligned} 
    \label{eq:obj_missing_appendix}
\end{equation} 

To obtain an ADMM solution we introduce a proxy variable $U_s=Z_s$ and obtain the following Lagrangian for sample $s$:

\begin{equation}
    \begin{aligned}
         &\mathcal{L}(Z_s, U_s) 
         =  || X_s - L Z_s R^T ||_F^2 + \lambda_1 ||U_s||_1  \\
         &+ \frac{\rho}{2} || U_s - Z_s + \frac{\Gamma}{\rho} ||_F^2 
    \end{aligned} 
    \label{eq:obj_lag_Z}
\end{equation}

The above equation is similar to Eq.~\ref{eq:update_Ys}. By simply replacing $B$ with $R$, $Y_s$ with $Z_s$, we can obtain closed-form updates for $Z_s$ and $U_S$ using the same steps. 

\section{\ourmeth (and CMOD) with missing values}
\label{appendix:alg-missing}

\noindent{\bf \ourmeth with missing values.} The missing values objective for our problems is:

\begin{equation}
    \begin{aligned}
         % f(\X) = 
         \underset{L, R, Y, W} {\mathrm{argmin}} \sum_{s=1}^{S} (&|| \Omega_s \odot (X_s - L Y_s W_s R^T) ||_F^2 + \lambda_1 ||Y_s||_1 + \lambda_2 ||W_s||_1),
    \end{aligned} 
    \label{eq:obj_missing}
\end{equation}
where $\Omega_s$ is a sample-specific missing value $1-0$ mask and $\odot$ denotes the element-wise product.
To optimize the missing values objective from Eq.~\ref{eq:obj_missing} we introduce additional proxy variables $D_s = X_s, U_s = Y_s, V_s = W_s$, arriving at the following objective:%, and solving the same way as in ~\cite{TGSD}:
\begin{equation}
    \begin{aligned}
         % f(\X) = 
         \underset{L, R, Y, W} {\mathrm{argmin}} \sum_{s=1}^{S} (&|| D_s - L Y_s W_s R^T ||_F^2 + \lambda_1 ||U_s||_1 \\
         &+ \lambda_2 ||V_s||_1 + \lambda_3 || \Omega_s \odot (D_s - X_s)||_F^2) \\
         s.t. ~~D_s = X_s, &Y_s = U_s, W_s = V_s,
    \end{aligned} 
    \label{eq:obj_missing_appendix}
\end{equation}

We form the Lagrangian function for sample $s$:

\begin{equation}
    \begin{aligned}
         &\mathcal{L}(D_s, Y_s, W_s, U_s, V_s)  \\
         = & || D_s - L Y_s W_s R^T ||_F^2 + \lambda_1 ||U_s||_1 \\
         &+ \lambda_2 ||V_s||_1  + \lambda_3 || \Omega_s \odot (D_s - X_s)||_F^2 \\
         & + \frac{\rho_1}{2} || U_s - Y_s + \frac{\Gamma_1}{\rho_1} ||_F^2 + \frac{\rho_2}{2} || V_s - W_s + \frac{\Gamma_2}{\rho_2} ||_F^2.
    \end{aligned} 
    \label{eq:obj_lag_missing}
\end{equation}

To update $D_s$, we have the following optimization problem:
\begin{equation}
    \begin{aligned}
         % f(\X) = 
         \underset{D_s} {\mathrm{argmin}} || D_s - L Y_s W_s R^T ||_F^2 + \lambda_3 || \Omega_s \odot (D_s - X_s)||_F^2
    \end{aligned} 
    \label{eq:obj_D}
\end{equation}

Taking the gradient and setting the equation to zero, we get

\begin{equation}
    \begin{aligned}
         D_s = (L Y_s W_s R^T + \lambda_3 \Omega_s \odot X_s ) \oslash (I + \lambda_3 \Omega_s), 
    \end{aligned} 
    \label{eq:solve_D}
\end{equation}
where $\oslash$ is element-wise division.

The update for other variables $Y_s, W_s, U_s, V_s$ are exactly the same as in Alg. \ref{alg:lrcs}, we will omit here.

To update the dictionaries, we have the following objective function which differs slightly from the one in the main paper:
\begin{equation}
    \begin{aligned}
         % f(L, R) = 
         \underset{L, R} {\mathrm{argmin}} \sum_{s=1}^{S} (&|| D_s - L Y_s W_s R^T ||_F^2.
    \end{aligned} 
    \label{eq:obj_LR_missing}
\end{equation}
By setting $\partial f / \partial L=0$, we get a closed-form solution for $L$:
\begin{equation}
    \begin{aligned}
         L = (\sum D_s R W_s^T Y_s^T)(\sum Y_s W_s R^T R W_s^T Y_s^T)^{-1},
    \end{aligned} 
\end{equation}
and $R$ can be solved for similarly. Both dictionaries are normalized to unit atom length.

% An overall algorithm with the detailed steps provided below in Alg.~\ref{alg:aodl_missing}.

% \begin{algorithm} [h]
% \footnotesize
%     \caption{\ourmeth with missing values}
%         \begin{algorithmic}[1]
%         \State {\bf Input:} Samples $X_s, s \in [1 \cdots S] $, Missing value indicator $\Omega_s, s \in [1 \cdots S] $, dictionary sizes $P$ and $Q$, encoding rank $k$ and sparsity params. $\lambda_1, \lambda_2, \lambda_3$
%         \State {\bf Output:} Dictionaries $L \in\R^{N\times P}$ and $R\in\R^{M\times Q}$ and encodings $(Y_s,W_s),\forall s \leq S$
%         \State Initialize $L,  R$ with unit-norm atoms
%         \Repeat
%             %\State // \emph{Stage I: Sparse coding}          
%             \For{$s = [1 \cdots S$]}
%                 \State $[D_s, Y_s, W_s] = \text{LRSC-with-missing}(\Omega_s, X_s, L, R, \lambda_1, \lambda_2, \lambda_3, k)$
%             \EndFor
%             %\State // \emph{Stage II: Dictionary updates}
%             \State {$L = \text{normalize}((\sum D_s R W_s^T Y_s^T)(\sum Y_s W_s R^T R W_s^T Y_s^T)^{-1})$}
%             \State {$R = \text{normalize}((\sum D_s^T L Y_s W_s)(\sum W_s^T Y_s^T L^T L Y_s W_s)^{-1})$}
%         \Until{Convergence or fixed max iterations}
%         \end{algorithmic}
%     \label{alg:aodl_missing}
% \end{algorithm}

\noindent{\bf CMOD with missing values.} Note that the sparse coding stage of CMOD-ADMM is similar to our objective with the key difference of a single sample-wise encoding matrix $Z_s$. To derive a missing-value-aware version of CMOD-ADMM we employ the same ADMM approach as the one outlined above for \ourmeth, with proxy variables for $Z_s$ and $X_s$ only: 
\begin{equation}
    \begin{aligned}
         % f(\X) = 
         \underset{Z} {\mathrm{argmin}} \sum_{s=1}^{S} (&|| D_s - L Z_s R^T ||_F^2  \\
         &+ \lambda_1 ||U_s||_1 + \lambda_2 || \Omega_s \odot (D_s - X_s)||_F^2) \\
         s.t. ~~D_s = X_s, &Z_s = U_s.
    \end{aligned} 
    \label{eq:obj_missing_Z_appendix}
\end{equation} 
Updates for the above objective when the dictionaries are fixed are obtained in the same manner as those for \ourmeth. For the dictionary update stage, again, we just replace $X_s$ with $D_s$.

\section{Additional experiments} 

\begin{figure*}[t]
    \centering 
    \subfigure [SNR vs NNZ]
    {
        \includegraphics[width=0.23\linewidth]{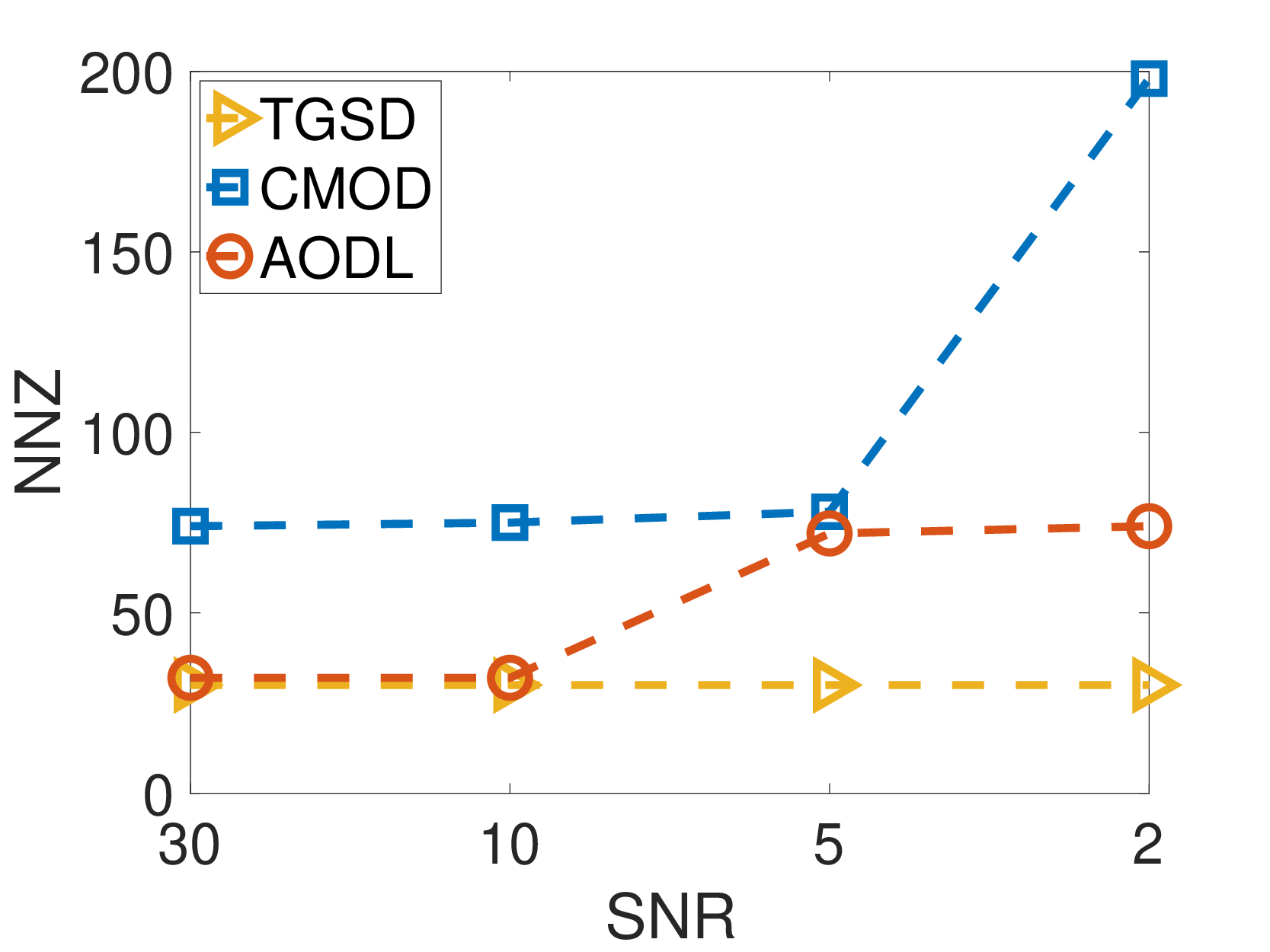}
        \label{fig:snr_nnz_f1}
    }
    \subfigure [NNZ vs RMSE]
    {
        \includegraphics[width=0.23\linewidth]{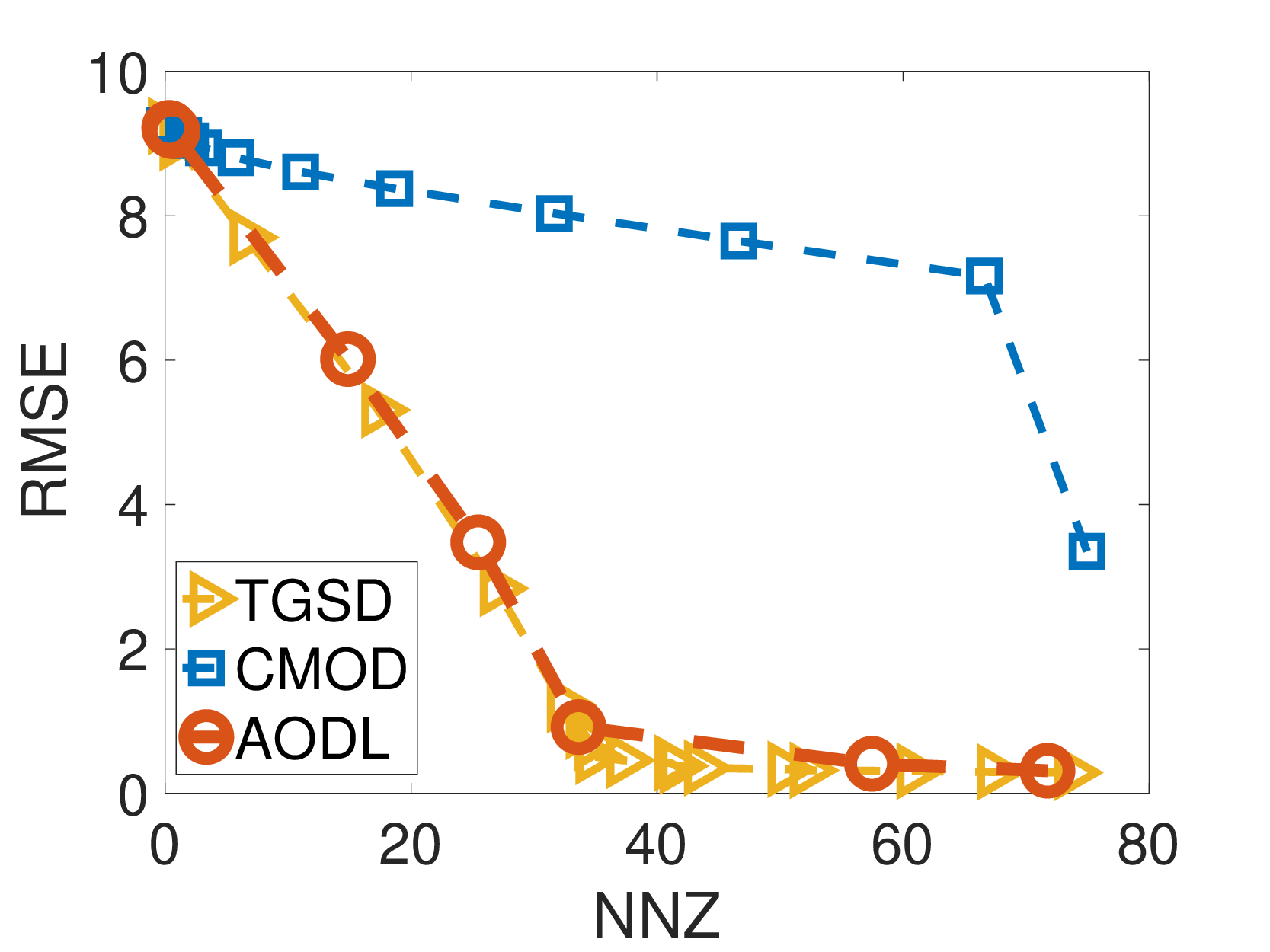}
        \label{fig:snr30_rmse}
    }
    \subfigure [NNZ vs RMSE (noisy)]
    {
        \includegraphics[width=0.23\linewidth]{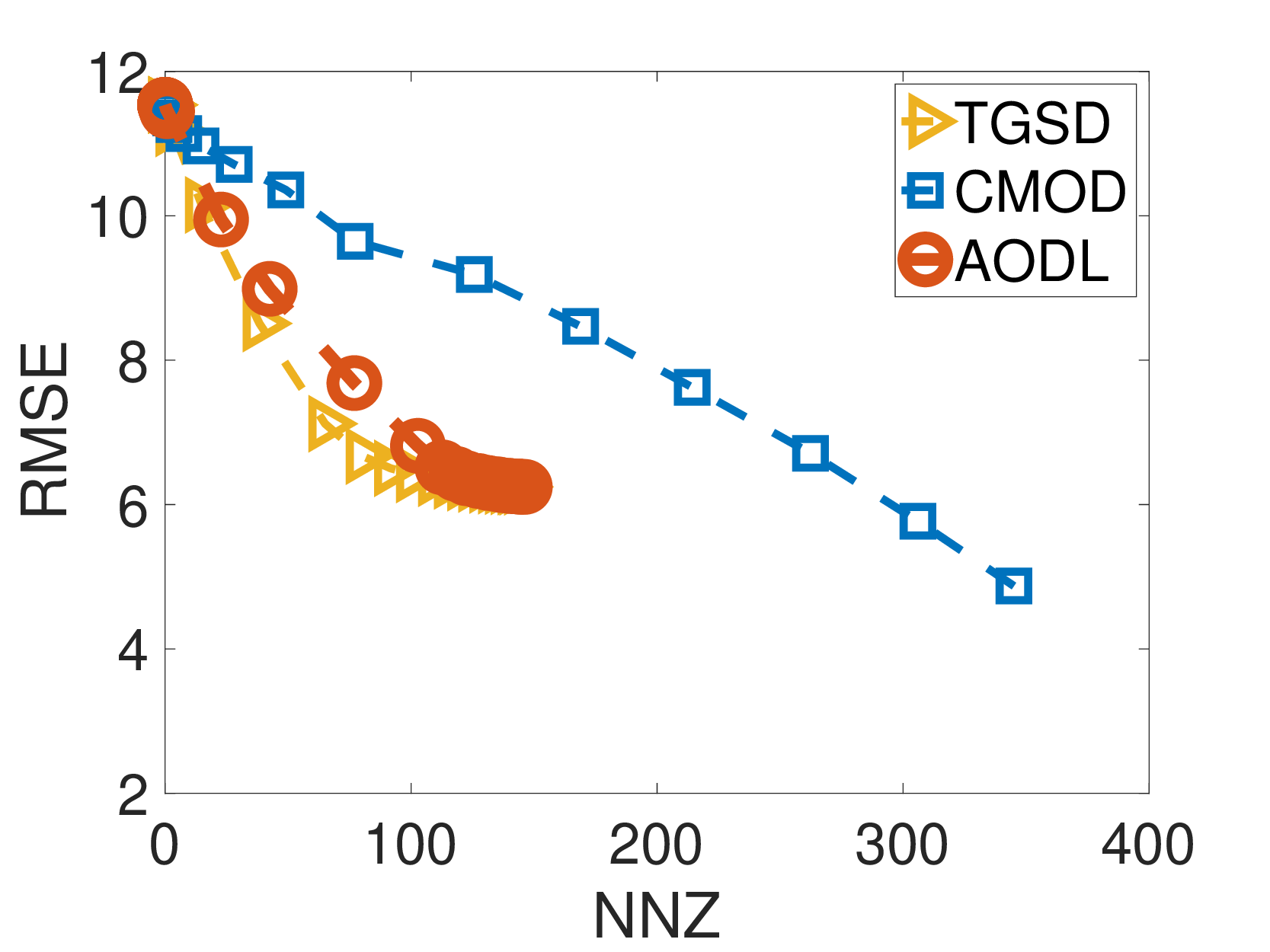}
        \label{fig:snr2_rmse_noisy}
    }
    \subfigure [NNZ vs RMSE (clean)]
    {
        \includegraphics[width=0.23\linewidth]{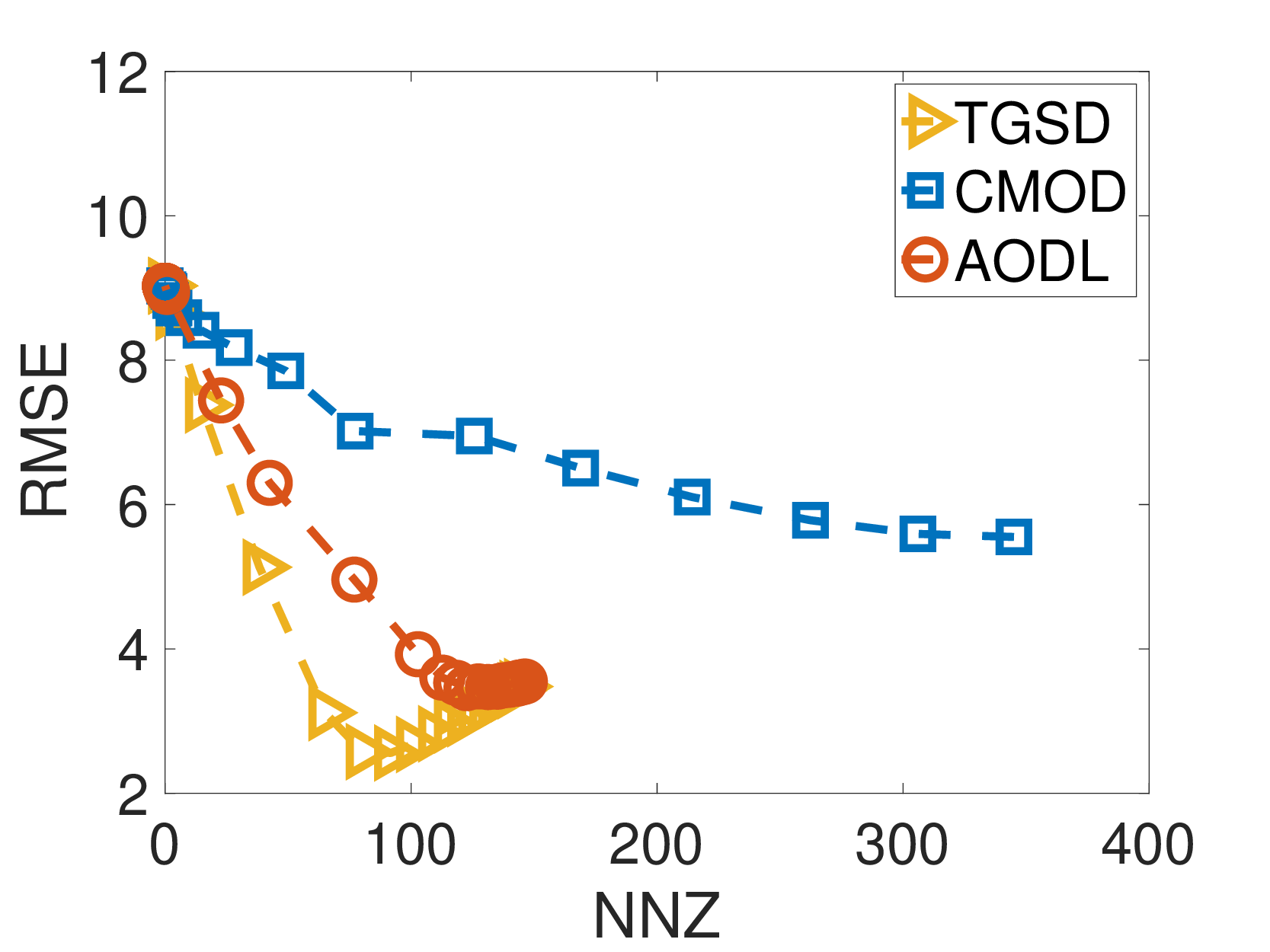}
        \label{fig:snr2_rmse_clean}
    }
    % \subfigure [NNZ vs RMSE (SNR = 30)]
    % {
    %     \includegraphics[width=0.23\linewidth]{fig/res/snr30-rmse-denser.eps}
    %     \label{fig:snr2_rmse_denser}
    % }
    \caption{\footnotesize Evaluation on noisy synthetic data. \subref{fig:snr_nnz_f1}: SNR vs NNZ for settings in which CMOD and AODL recover the GT dictionaries. \subref{fig:snr30_rmse}: NNZ vs RMSE for $SNR = 30$; \subref{fig:snr2_rmse_noisy} NNZ vs RMSE for SNR = 2 while representing the noisy data (clean + noise) and \subref{fig:snr2_rmse_clean} NNZ vs RMSE for $SNR =2$ w.r.t. the clean data only. 
    % \subref{fig:snr2_rmse_denser} GT NNZ = 50, NNZ vs RMSE when SNR = 30
    \vsa\vsa}
    \label{fig:syn_dict_size}
\end{figure*}

%%%%%%%%%%%%%%%%%%%%%%%%%%%%%%%%%%%%%%%%%%%%%%%%%%%%%%%%%%%%%%%%%%%%%%%%%%%%%%%
\noindent{\bf Synthetic data: dictionary recovery and data reconstruction in the presence of noise.}  Next we evaluate the ability of \ourmeth (and baselines) to encode data in the presence of noise and also their ability to recover dictionaries similar to the ground truth (GT) dictionaries used to generate the data. In all synthetic tests, TGSD has the advantage of encoding with the GT dictionaries and as a result could be viewed as a GT baseline that CMOD and \ourmeth can approach, but not necessarily beat.

In Fig.~\ref{fig:snr_nnz_f1} we vary the noise level quantified as SNR (signal to noise ratio) in the range $[30, 10, 5, 2]$(dB). Each sample of the ground truth signal is produced via the low-rank encoding model with a total of $30$ NNZ coefficients and by employing random GT dictionaries. We report the NNZ of each method when for hyperparameter settings with which both CMOD and \ourmeth can perfectly recover the GT dictionaries (best matching pairs of GT and learned atoms have a cosine similarity exceeding $0.99$). While both CMOD and  \ourmeth can both recover the GT dictionaries, \ourmeth produces this result with less NNZs in its encoding (with largest advantage gap in the noisiest setting SNR = 2). Both \ourmeth and CMOD use more coefficients than TGSD to represent the data (the latter uses GT dictionaries), however, thanks to \ourmeth's low rank model it requires fewer coefficients. We further visualize the alignment of learned and GT dictionaries from this experiment in Fig. \ref{fig:syn_dict_alignment}. The NNZ are fixed (up to 80), we plot the alignment of the learned dictionaries with the GT dictionaries. It is clear that \ourmeth can almost perfectly recover the GT dictionaries. Meanwhile, the dictionaries learned by CMOD is much noisier. 
For more complex scenarios, GT dictionaries may not be identifiable.
%there may be multiple sets of dictionaries that can produce low representation error and expecting a perfect reconstruction of GT dictionaries might not be feasible (or even desirable). 

In Fig.~\ref{fig:snr30_rmse}, we plot the reconstruction error (RMSE) of all methods at different sparsity levels for SNR=30db. \ourmeth is closely aligned with TGSD which uses the GT dictionaries while CMOD requires more coefficients to achieve the same RMSE levels. %In Fig. \ref{fig:snr30_rmse}, it is interesting to see that the RMSE of CMOD is dropping vertically around NNZ = 75. The last point achieves the best RMSE, however, the point before that are missing some crucial coefficients that are important in reconstruction the data.
Figs.~\ref{fig:snr2_rmse_noisy} and \ref{fig:snr2_rmse_clean} present the RMSE v.s. NNZ trade-off in a much noisier setting (SNR = 2db). While the methods are executed on noisy data (sparse coding from GT dictionary + noise) we seek to quantify the quality of fit to the ``clean'' component of the samples as well as the noisy samples. Thus, we report the RMSE computed with respect to noisy data (i.e. clean + noise) in Fig.~\ref{fig:snr2_rmse_noisy}; and with respect to only the clean component of the data in Fig.~\ref{fig:snr2_rmse_clean}. CMOD requires significantly more coefficients than \ourmeth to achieve the same RMSE levels. This is likely due to the low-rank encoding model in \ourmeth acting as a noise filter. Fig.~\ref{fig:snr2_rmse_clean} also suggests that CMOD likely uses coefficients to represent noise since its quality in clean data is worse than that in noisy data at high NNZs. \ourmeth is able to capture the clean component in the data much better, and its curve is closer to that of TGSD which employs the GT dictionaries.

\begin{figure*}[t]
    \centering 
    \subfigure [L alignment (AODL)]
    {
        \includegraphics[width=0.22\linewidth]{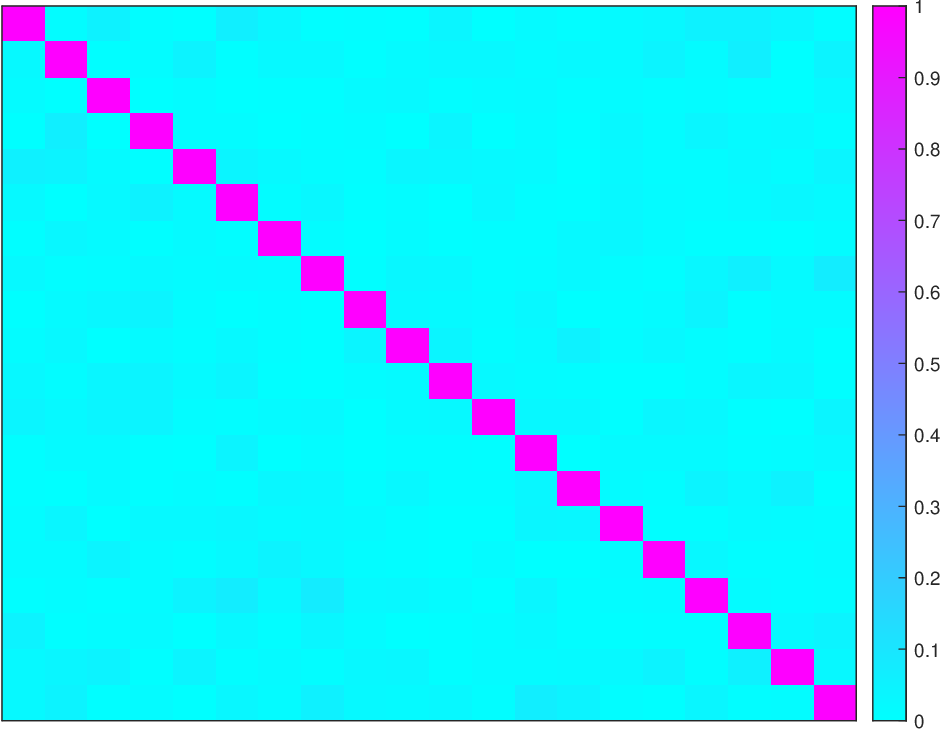}
        \label{fig:l_aodl}
    }
    \subfigure [L alignment (CMOD)]
    {
        \includegraphics[width=0.22\linewidth]{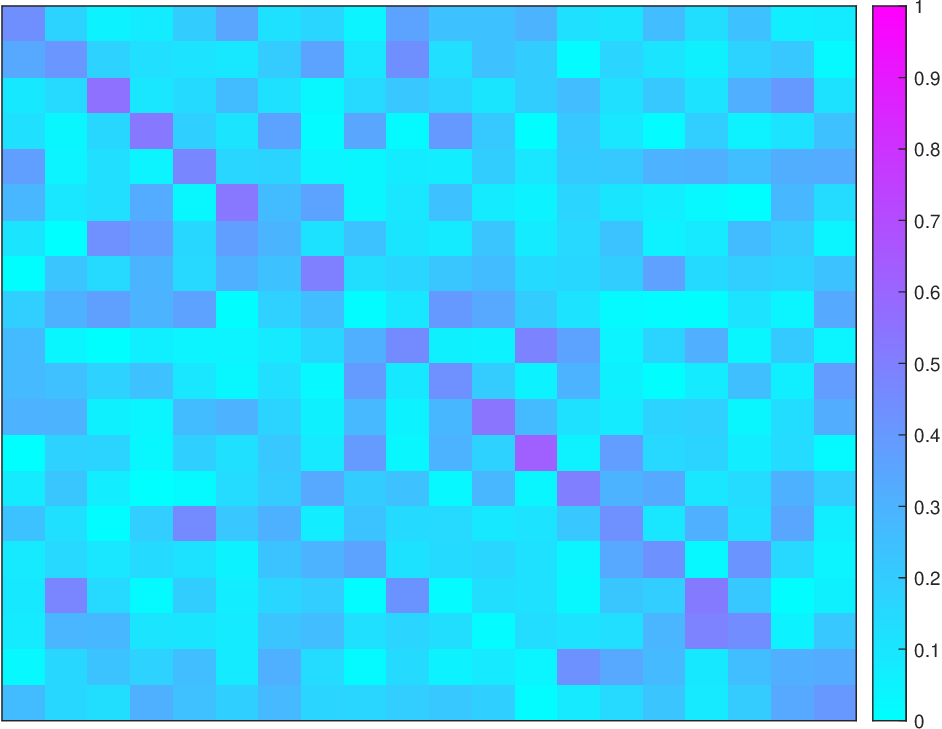}
        \label{fig:l_cmod}
    }
    \subfigure [R alignment (AODL)]
    {
        \includegraphics[width=0.22\linewidth]{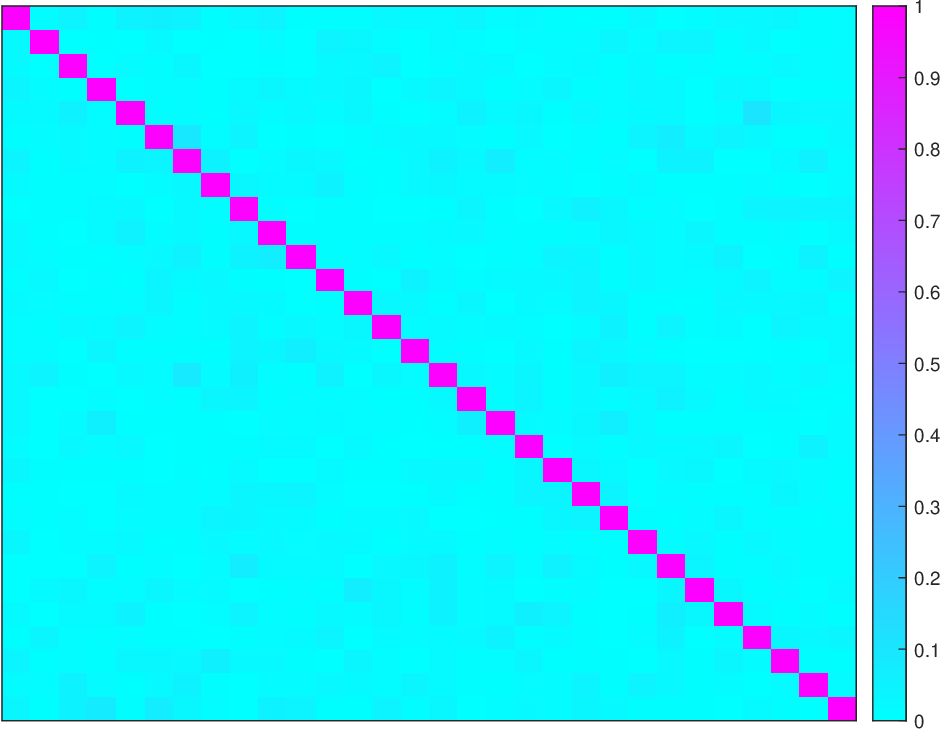}
        \label{fig:r_aodl}
    }
    \subfigure [R alignment (CMOD)] 
    {
        \includegraphics[width=0.22\linewidth]{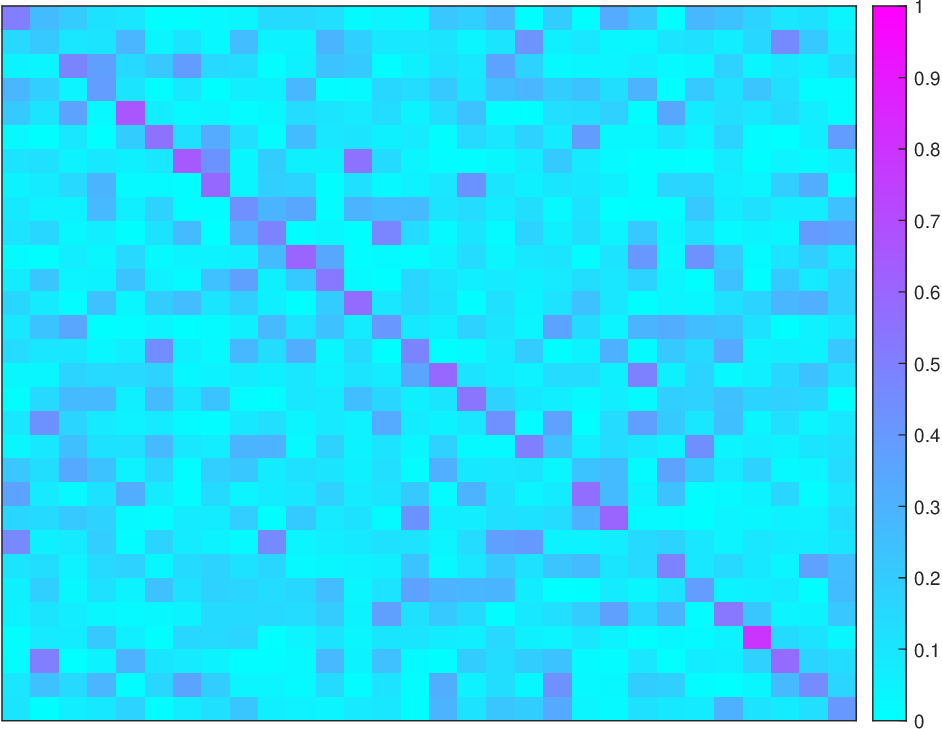}
        \label{fig:r_cmod}
    }
    \caption{ \footnotesize
    Alignment of the learned dictionary atoms (measured as inner products) with the ground truth dictionary atoms (when SNR = 2, NNZ for both methods are up to 80) in synthetic data. Identity matrix (1s on diagonal and 0s off-diagonal) corresponds to perfect atom recovery.
    }
    \label{fig:syn_dict_alignment}
\end{figure*}

\begin{table*}[t]
\tiny
\setlength\tabcolsep{5 pt}
\centering
 \begin{tabular}{|c|c|p{0.13\linewidth}|p{0.1\linewidth}|p{0.1\linewidth}|p{0.06\linewidth}|p{0.06\linewidth}|p{0.06\linewidth}|p{0.06\linewidth}|p{0.06\linewidth}|} 
 \hline
 Method & Parameters &  Range & Synthetic & Theoretical test & Road & Twitch & Wiki & MIT & Air\\
 \hline
 TGSD  & $\lambda_1, \lambda_2, k$ & $[10^{-3}, \cdots, 10^{3}]$, $[10^{-3}, \cdots, 10^{3}]$, $[3, 15, 30, 45]$ & $\lambda_1, \lambda_2$ vary, $k = 3$ & NA & Vary & Vary & Vary & Vary & Vary\\
 \hline
SeDiL  & $q, \mu, \lambda, \beta, \text{iter}$ & $[1,2,10]$, $[10, 10^2, 10^3]$, $[10, \cdots, 10^5]$, $[0.3, 0.5, 0.8]$, $[500,1k, 5k, 10k]$ & $1, 10^2, 10^2$, $0.8, 500$ & NA & $2$, $10^2$, $10^5$, $0.5$, $5k$ & $2$, $10^2$, $10^2$, $0.8$, $5k$ & $2$, $10^2$, $10^2$, $0.8$, $5k$ & $2$, $10^2$, $10^4$, $0.8$, $5k$ & $10$, $10^2$, $10^3$, $0.8$, $1k$\\
 \hline
  CMOD-OMP & $T_0$ & $[35, 1k, 1.5k, 3k]$ & 35 & NA & 3k & 3k & 1.5k & 1.5k & 1k\\
  \hline
 CMOD & $\lambda_1$ & $[10^{-3},  \cdots, 10^{3}]$ & Vary & NA & Vary & Vary & Vary & Vary & Vary\\
 \hline
 \ourmeth  & $\lambda_1, \lambda_2, k$ & $[10^{-3}, \cdots, 10^{3}]$, $[10^{-3}, \cdots, 10^{3}]$, $[3, 5, 15, 30, 45]$ & $\lambda_1, \lambda_2$ vary, $k = 3$ & $\lambda_1, \lambda_2$ vary, $k = 5$ & Vary & Vary & Vary & Vary & Vary\\
 \hline
\end{tabular}
\caption{\footnotesize Parameters for competing methods where $\lambda_1, \lambda_2$ are sparsity parameters for ADMM sparse solver; $T_0$ is the targeting number of coefficients for 2D-OMP;  $k$ is the rank parameter of TGSD. Some methods are not included in the theoretical test the corresponding cells are marked as NA for Not Applicable. Ranges for tested values are listed in the Range column. 
}
\label{table:method_params}
\end{table*}

\noindent{\bf \ourmeth vs Random Dictionaries} To demonstrate the learned dictionary does help in producing better reconstruction error since low rank model by nature provide better NNZ. Instead of comparing \ourmeth with the ground truth dictionaries, we compare it with random generated dictionaries. In this test, the basic settings are the same with synthetic test settings. We generated two random dictionaries that serve as the left and the right dictionaries, and we use them directly with our low rank sparse coding model, which is named RAND. We can clearly see in Fig \ref{fig:ab_test}, with the same level of NNZ, RAND's reconstruction error is much higher than \ourmeth.

\begin{figure}[t]
    
   \centering
   \subfigure[\ourmeth v.s. Random dictionaries]
   {
    \includegraphics[width=0.45\linewidth]{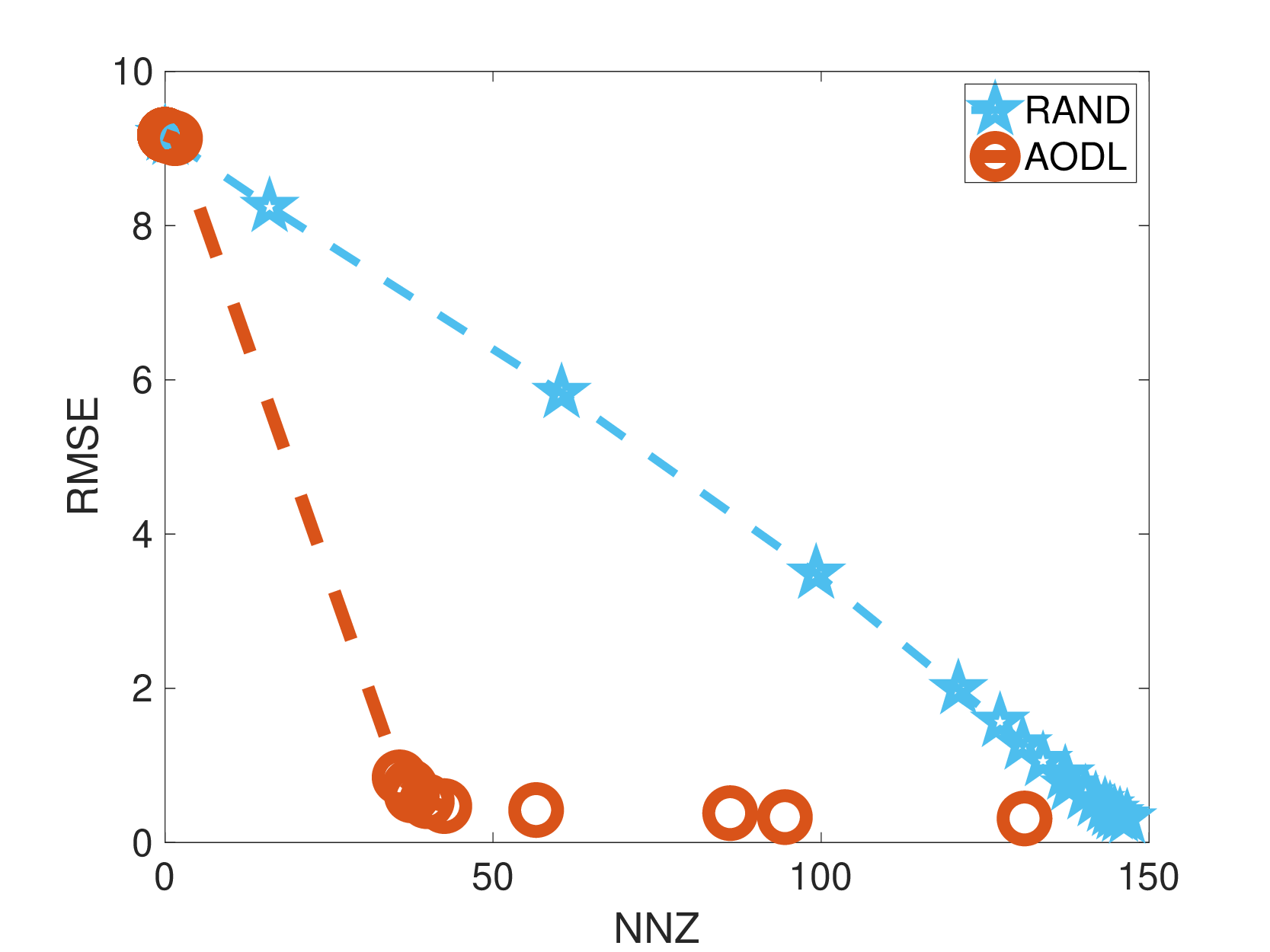}
    \label{fig:aodl_rand}
     }   
     \subfigure[Slicing Twitch data over time]
   {
    \includegraphics[width=0.45\linewidth]{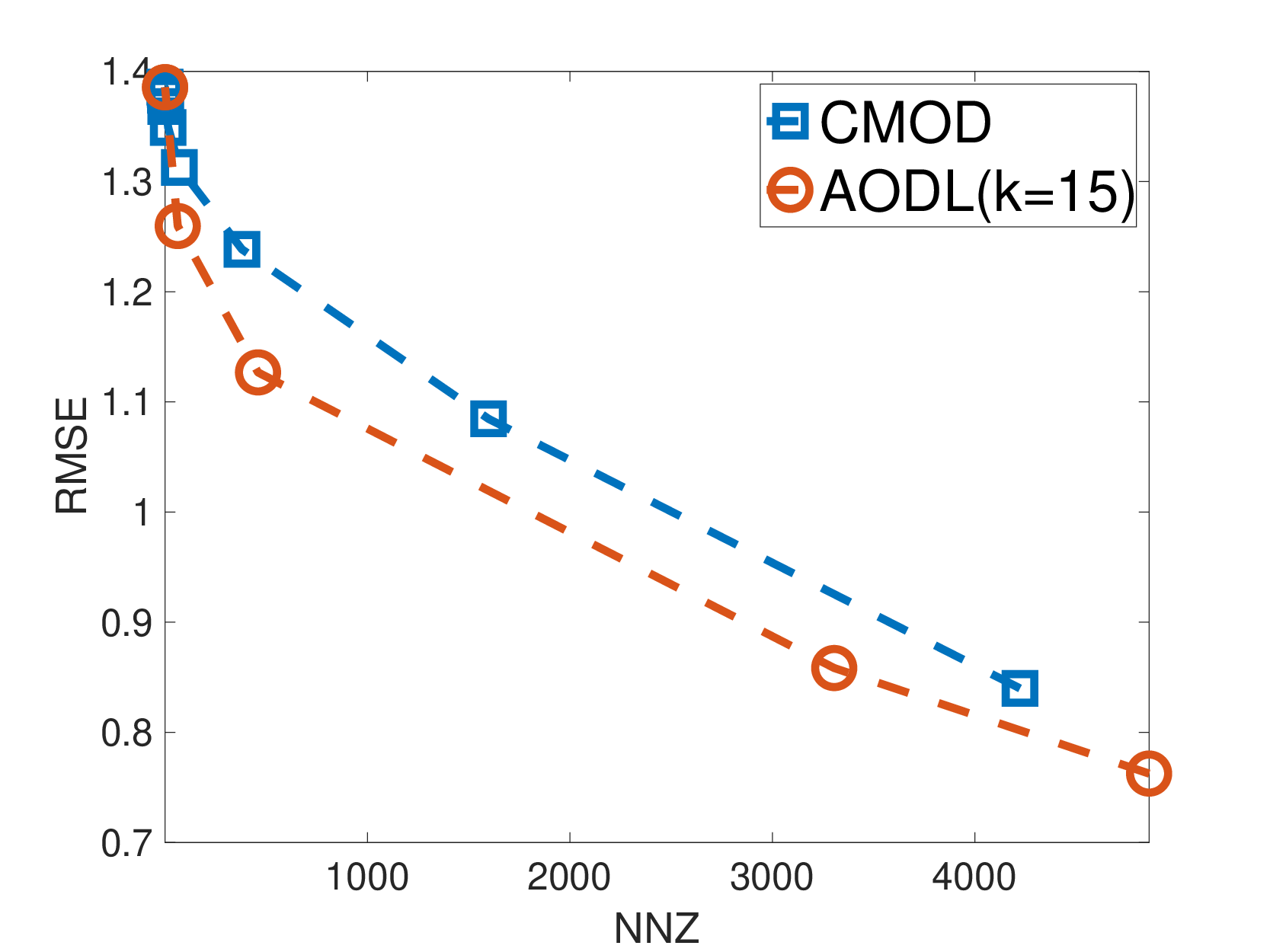}
    \label{fig:twitch_slice_short}
     }  
    \caption{\subref{fig:aodl_rand} \footnotesize Comparison of \ourmeth and TGSD when using random dictionaries. Since low rank model is providing better NNZ, we can see that the learned dictionaries is important in representing the data. \subref{fig:twitch_slice_short} In all tests, we split the data on the larger dimension. This figure shows limitations when slicing the data on the shorter dimension since the size of the encoding matrices is decided by $k, P, Q$ (more details in the explanation below).  }
    \label{fig:ab_test}
\end{figure}

\noindent{\bf Slicing the data on the smaller dimension} As mentioned in the real-world test, we slice the data into multiple samples on the longer dimension to reduce the calculation cost and also prevent learning large dictionaries. However, our goal is just dividing the data into multiple samples to fit our model and the baselines, we are free to slice on any dimension. In the previous test, we slice Twitch data who has $9000$ nodes and $512$ timesteps into $30$ samples on the node direction. Here, we slice it on the time direction into $16$ samples, each sample has a size of $\R^{9000 \times 32}$. From Fig \ref{fig:twitch_slice_short}, we can see our method \ourmeth still has some advantage regarding NNZ vs RMSE, however, the difference is vary limited comparing with CMOD. This is because the dictionary $R \in \R^{32 \times 32}$ has a very small size (atom number), and it is close to the rank parameter $k$ we choose. We know that the number of coefficients of \ourmeth is NNZ(Y) + NNZ(W), while for CMOD, the number of coefficients is NNZ(Z). The size of Y is $P \times k$; and size of W is $k \times Q$. So, the total size of \ourmeth would be $P k + k Q = k (P+Q)$. While the size of $Z$ in CMOD is just $P \times Q$. The advantage of our method \ourmeth will be large when $k (P+Q) << PQ $. Since we always prefer smaller $k$, if either $P$ or $Q$ is small as well, our advantage will be vanished. As a result, our model will always prefer the two learned dictionaries to have relatively large size, and a small rank $k$.

\section{Hyperparameter tuning and selection}
\label{appendix:hyper}

The parameter settings for all competing techniques unless otherwise specified are as follows. For low rank models like TGSD and \ourmeth, we set $k=3$ in synthetic test, and $k$ takes values in $[15, 30, 45]$ in real-world reconstruction test and missing value imputations tests. In the case study, the dictionaries are calculated at $k = 15$. In addition to grid search for $k$, as mentioned in \cite{TGSD}, one could estimate the rank $k$ using SVD.
The $\lambda$s are sparsity parameter for ADMM method. In CMOD, we only have $\lambda_1$ since only one encoding matrix is calculated, and in TGSD, we have $\lambda_1$ and $\lambda_2$. We grid search these parameter in range $[10^{-3}, 10^{-2}, \cdots, 10^{3}]$ to produce curve with different NNZ and RMSE values. In CMOD-OMP, $T_0$ is the only parameter, which indicates the target number of coefficients in the encoding matrix. OMP models works much slower as $T_0$ increases, as a result, we choose some relatively small targets in tests and only report them in Tbl. \ref{table:datasets}.

SeDiL is another baseline model that requires intensive parameter search. $q$ means the weight of mixed sparsity measure, which indicates how the sparsity term is being adjusted using power of $q$. $\mu$ is the multiplier of the sparse matrix. $\lambda$ is the Lagrange multiplier of the sparsity term. $\beta$ is the step size when updating the sparse matrix. $iter$ is the maximum iteration of the model. We can see that all the above parameters are affecting the performance of the sparsity of the model. We picked the values that can help us to reach the target NNZ, which is defined in Tbl. \ref{table:datasets}, for a fair comparison of all models.

All hyperparameter ranges are listed in Tbl. \ref{table:method_params}.

% Todo (if no luck)
% 1. Rerun CMOD on wiki and road, low NNZ's rmes might be improved.
% 2. Add different rank k for TGSD
% 3. Estimate rank k using PCA or SVD.
% \input{tex/11_appendix_6pages}
% \input{tex/12_diff_tsp_vs_icassp}

% \newpage

 

% %\begin{IEEEbiography}[{\includegraphics[width=1in,height=1.25in,clip,keepaspectratio]{fig1}}]{Michael Shell}
% %Use $\backslash${\tt{begin\{IEEEbiography\}}} and then for the 1st argument use $\backslash${\tt{includegraphics}} to declare and link the author photo.
% %Use the author name as the 3rd argument followed by the biography text.
% %\end{IEEEbiography}


% \bf{If you will not include a photo:}%\vspace{-33pt}

\vfill

\end{document}